\documentclass{article}

\usepackage{microtype}
\usepackage{graphicx}
\usepackage{subfigure}
\usepackage{booktabs} 

\usepackage{hyperref}

\usepackage[accepted]{arxiv_tem}

\usepackage{amsmath}
\usepackage{wrapfig}
\usepackage{amssymb}
\usepackage{mathtools}
\usepackage{amsthm}
\usepackage{xcolor}
\usepackage{algorithm}
\usepackage{algorithmic}
\usepackage{amsmath,amssymb}
\usepackage{caption}

\usepackage[capitalize,noabbrev]{cleveref}

\newcommand{\highlight}[1]{\textcolor{blue}{#1}} 

\theoremstyle{plain}
\newtheorem{theorem}{Theorem}[section]

\newtheorem{lemma}{Lemma}

\theoremstyle{definition}

\newtheorem{assumption}{Assumption}
\theoremstyle{remark}

\usepackage[textsize=tiny]{todonotes}

\icmltitlerunning{Trajectory Consistency for One-Step Generation on Euler Mean Flows}

\begin{document}

\twocolumn[
\icmltitle{Trajectory Consistency for One-Step Generation on Euler Mean Flows}



\begin{icmlauthorlist}
\icmlauthor{Zhiqi Li}{yyy}
\icmlauthor{Yuchen Sun}{yyy}
\icmlauthor{Duowen Chen}{yyy}
\icmlauthor{Jinjin He}{yyy}
\icmlauthor{Bo Zhu}{yyy}
\end{icmlauthorlist}

\icmlaffiliation{yyy}{College of Computing, Georgia Tech, Location, Country}

\icmlcorrespondingauthor{Zhiqi Li}{zli3167@gatech.edu}

\icmlkeywords{}

\vskip 0.3in
]

\newcommand{\bo}[1]{\textcolor{blue}{[Bo: #1]}}



\printAffiliationsAndNotice{}  

\begin{abstract}
We propose \emph{Euler Mean Flows (EMF)}, a flow-based generative framework for one-step and few-step generation that enforces long-range trajectory consistency with minimal sampling cost. 
The key idea of EMF is to replace the trajectory consistency constraint, which is difficult to supervise and optimize over long time scales, with a principled linear surrogate that enables direct data supervision for long-horizon flow-map compositions.  We derive this approximation from the semigroup formulation of flow-based models and show that, under mild regularity assumptions, it faithfully approximates the original consistency objective while being substantially easier to optimize.  This formulation leads to a unified, JVP-free training framework that supports both $u$-prediction and $x_1$-prediction variants, avoiding explicit Jacobian computations and significantly reducing memory and computational overhead.  Experiments on image synthesis, particle-based geometry generation, and functional generation demonstrate improved optimization stability and sample quality under fixed sampling budgets, together with approximately $50\%$ reductions in training time and memory consumption compared to existing one-step methods for image generation.

\end{abstract}

\section{Introduction}



Recent advances in generative modeling, particularly diffusion models and flow matching methods, have achieved remarkable success in image generation \cite{lipman2023flow, song2021denoising}, video synthesis \cite{ho2022video, ho2022imagen}, and 3D geometry modeling \cite{luo2021diffusion, vahdat2022lion, zhang2025geometry}.
From a continuous-time perspective, these methods can be unified by the continuity equation, which learns a time-dependent velocity field of probability flow to transform simple noise distributions into complex data distributions \cite{lipman2024flow}. Under this formulation, the generation process corresponds to a continuous trajectory evolving from noise space to data space, and model training aims to characterize the dynamics of this flow-map trajectory at different time points.


While such trajectory-based models provide strong expressive power, sampling from the learned dynamics typically requires a large number of time steps, resulting in substantial inference cost. To improve efficiency, a growing body of recent work focuses on one- and few-step generation, aiming to approximate long sampling trajectories with only a small number of steps \cite{song2023consistency, frans2025shortcut, guo2025splitmeanflow, geng2025mean}, thereby reducing inference time while maintaining competitive generation quality. A central challenge in one-step and few-step generation lies in learning trajectory consistency \cite{frans2025shortcut, guo2025splitmeanflow}, meaning that predictions at different points along the trajectory should agree with each other.

Mathematically, trajectory consistency can be characterized by the semigroup property of flow maps: for all $t \le s \le r$, the flow maps $\phi_{t\to r}$ satisfy $\phi_{t\to r} = \phi_{s\to r}\circ \phi_{t\to s}$.  Here, the flow map $\phi_{t\to r} : \mathcal{X} \to \mathcal{X}$ is defined as the mapping that transports a state in the space $\mathcal{X}$ from time $t$ to time $r$ along the underlying dynamics, and satisfies $\phi_{t\to r}(x_t)=x_r$ for any trajectory $(x_t)_{t\in[0,1]}$.  
This semi-group property ensures coherent long-range flow maps across different time scales~\cite{webb1985semi}. However, learning such flow maps with trajectory consistency with supervision from data is nontrivial, because in traditional flow-based models (e.g., \cite{lipman2023flow, liu2023flow}) there is no explicit reference flow map $\phi_{t\to r}$ derived from the data distribution. As a result, trajectory consistency constraints cannot be directly supervised during model training. Moreover, inaccurate formulations of trajectory consistency may disrupt the underlying flow-map structure, leading to unstable training or degraded generation quality \cite{boffi2025flow}.


Existing approaches for addressing this issue can be categorized into two classes.
The first category methods progressively extend short-range transitions to longer intervals by composing locally learned dynamics \cite{frans2025shortcut,guo2025splitmeanflow}. Although conceptually simple, such methods suffer from error accumulation for trajectories, as long-range behavior is inferred indirectly from short-range estimates without explicit global supervision.  The second class, represented by MeanFlow and related methods \cite{geng2025mean, zhang2025alphaflow}, derives training objectives directly from continuity equations. By introducing consistency constraints at the level of flow maps, these methods provide principled supervision for long-range dynamics. However, they rely on explicit gradient computation with several practical limitations: (1) Explicit gradient computation incurs substantial memory and computational overhead that limits efficient network architectures and training procedures (e.g., FlashAttention \cite{dao2022flash}). (2) Incorporating explicit gradients into the loss may lead to numerical instability, especially under mixed-precision training, as observed in our image and SDF generation experiments. 
(3) Gradient-based objectives are poorly compatible with sparse computation primitives, limiting their applicability to domains such as functional generation and point cloud modeling.


In this work, we propose a new approach for trajectory-consistent one-step generation by revisiting the semigroup structure of flow maps. Our key idea is to apply a local linearization to the trajectory consistency equation and enable direct supervision from the data distribution for long-range flow maps. This linear approximation transforms the original long-range consistency constraint into a learnable surrogate objective without calculating derivatives. We proved that, under reasonable conditions (\autoref{assupmption} and \autoref{thm:validity}), this surrogate loss faithfully approximates the original consistency objective and enables accurate learning of the instantaneous velocity along long-range flow maps. Based on this analysis, we further develop a gradient-free training framework that significantly reduces memory and computational cost and leads to more stable optimization. 
%
%
Motivated by the manifold assumption 
as advocated in~\citep{li2025back}, we formulate a unified framework for one-step and few-step generation that supports both $u$-prediction and $x_1$-prediction, with the latter emphasizing direct supervision on the terminal state of the flow. Our linearized formulation is inspired by Euler time integration~\cite{hairer1993solving} in numerical mathematics; accordingly, we refer to our approach as \emph{Euler Mean Flows} (EMF).
%

Our main contributions are summarized as follows:
\begin{itemize}
    \item We propose \emph{Euler Mean Flows (EMF)}, a trajectory-consistent framework for one-step and few-step generation based on a linearized semigroup formulation.
    \item We introduce a surrogate loss obtained by local linearization of the semigroup consistency objective, with theoretical guarantees under mild assumptions.
    \item We develop a unified, JVP-free training scheme that avoids explicit derivative computations and supports both $u$-prediction and $x_1$-prediction variants.
\end{itemize}

\section{Related Work}
\paragraph{Diffusion and Flow Matching.}
Diffusion models \cite{ho2020denoising, song2019generative, song2021score} have achieved remarkable success in data generation by progressively denoising random initial samples to produce high-quality data. This generative process is commonly formulated as the solution of stochastic differential equations (SDEs). In contrast, Flow Matching methods \cite{liu2023flow, lipman2023flow, albergo2023building} learn the velocity fields that define continuous flow trajectories between probability distributions. 

\paragraph{Few-step Diffusion/Flow Models.} Consistency models~\cite{song2023consistency, song2023improved, geng2025consistency, lu2025simplifying} were proposed as independently trainable one-step generators in parallel to model distillation \cite{salimans2022progressive, meng2023distillation, geng2023onestep}. Motivated by consistency models, recent works have introduced self-consistency principles into related generative frameworks~\cite{yang2024consistency,frans2025shortcut,zhou2025inductive}. Mean Flow \cite{geng2025mean} models the time-averaged velocity by differentiating the Mean Flow identity. $\alpha$-Flow \cite{zhang2025alphaflow} improves the training process by disentangling the conflicting components in the Mean Flow objective. SplitMeanFlow \cite{guo2025splitmeanflow} leverages interval-splitting consistency to eliminate the need for JVP computations in Mean Flow models. While both SplitMeanFlow and our method are JVP-free, SplitMeanFlow is limited to a distillation-based setting, whereas our approach enables fully independent training.
\begin{figure*}[t]
    \centering
    \includegraphics[width=\linewidth]{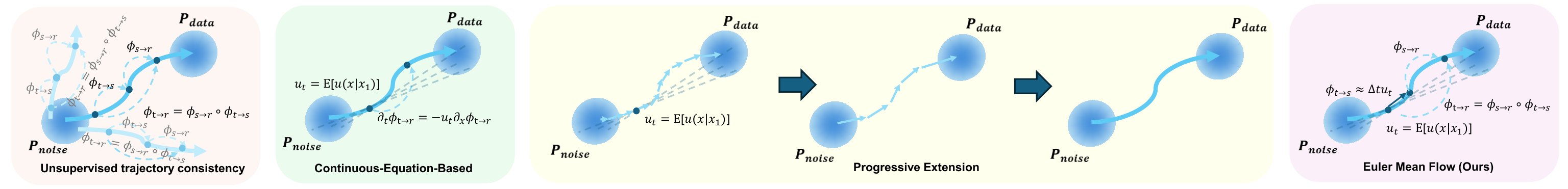}
    \caption{Illustration of trajectory consistency and the Euler Mean Flow (EMF) method.
    Left: Multiple flow maps can satisfy trajectory consistency, but only the solid path correctly transports noise to the data distribution, highlighting the necessity of data supervision.
    Middle: Two existing approaches for learning long-range trajectories, including continuous-equation-based methods and progressive extension.
    Right: Our EMF reformulates the trajectory consistency equation via a local linear approximation and introduces direct data supervision for long-range dynamics through the resulting linearized segment.}
    \vspace{-5mm}
    \label{fig:overview}
\end{figure*}
\section{Background}
Let $\mathcal{D} = \{ x^i \in \mathcal{X}\}_{i=1}^n$ be a dataset drawn from an unknown data distribution $p_{\text{data}}$ on space $\mathcal{X}$. Flow Matching aims to learn $p_1 = p_{\text{data}}$ by learning a continuous-time velocity field $u(x,t)$, $t \in [0,1]$, that transports a base distribution $p_0$, typically Gaussion distribution $\mathcal{N}(0,\sigma^2)$, to $p_1$ along a continuous path of distributions $(p_t)_{t \in [0,1]}$.  The evolution of the distribution path is governed by the continuity equation
\begin{equation}
    \begin{aligned}
    \frac{\partial}{\partial t} p_t(x) + \nabla \cdot \big( p_t(x)u_t(x) \big) = 0
    \end{aligned}
\end{equation}
Given learned $u_t(x)$, sampling $x_0 \sim p_0$, samples from $p_1$ can be obtained by integrating the ODE
\begin{equation}\label{eq:ODE_sampling}
    \begin{aligned}
        \frac{d x_t}{d t} = u(x_t, t), \quad x_0\sim p_0         
    \end{aligned}
\end{equation}
The associated flow $\phi_t : \mathcal{X} \to \mathcal{X}$ is defined by $\phi_t(x_0) = x_t$ for any $x_0, x_t$ satisfying the ODE and it satisfies
\begin{equation}\label{eq:evolv_flow_map}
    \begin{aligned}
\frac{\partial}{\partial t} \phi_t = u_t \circ \phi_t,  \quad\phi_0 = \mathrm{Id}_{\mathcal{X}} .        
    \end{aligned}
\end{equation}
We further define the flow map $\phi_{t \to r} = \phi_r \circ \phi_t^{-1}$. The path $p_t$ can be written as a pushforward $p_t = (\phi_t)_\sharp p_0$.

Flow Matching seeks to learn the velocity field $u_t(x)$.  Given a parameterized model $u_t^\theta(x)$, samples are generated by numerically integrating \autoref{eq:ODE_sampling} from $t=0$ to $t=1$.  A natural training objective for training is $\mathcal{L}^{FM}(\theta) = \mathbb{E}_{t, x\sim p_t(x)}\|u_t^\theta(x) - u_t(x)\|^2$, which directly matches the model velocity to the reference velocity field.  However, this objective cannot be optimized in practice, since both $u_t(x)$ and the marginal distribution $p_t(x)$ are not directly observable from the dataset.  To incorporate supervision from data, Flow Matching introduces conditional velocities $u_t(x|x_1) = \frac{x_1-x}{1-t}$ and conditional flows $\phi_t(x| x_1)=t(x_1-x)+x$ for arbitrary $x_1 \in \mathcal{X}$. These conditional quantities induce a conditional distribution  $p_t(x | x_1)=(\phi(\cdot| x_1))_\sharp p_0$, and the marginal velocity field and distribution can then be recovered by marginalization $u_t(x) = \mathbb{E}_{x_1\sim p_t(x_1|x)}[u_t(x | x_1)]$ and $p_t(x) = \mathbb{E}_{x_1\sim p_t(x_1|x)}[p(x | x_1)]$respectively.  Based on these constructions, Flow Matching defines the conditional surrogate objective $\mathcal{L}_c^{FM}(\theta) = \mathbb{E}_{t,x_1\sim p_{data}, x\sim p_t(x|x_1)}\|u_t^\theta(x) - u_t(x|x_1)\|^2$, which admits supervision from data samples.  It has been shown that $\nabla_\theta \mathcal{L}_c^{FM}(\theta) = \nabla_\theta \mathcal{L}^{FM}(\theta)$
and therefore $\mathcal{L}_c^{\mathrm{FM}}$ serves as a valid surrogate for optimizing $\mathcal{L}^{\mathrm{FM}}$.

Flow Matching learns the instantaneous velocity field $u_t(x)$.
As a result, sample generation requires iterative numerical integration, making it inherently a multi-step process.
In contrast, one-step and few-step generative models aim to directly learn the flow maps $\phi_{t \to r}$, enabling efficient generation with a small number of transitions.

\section{One-Step Generation on Euler Mean Flows} \label{sec:tranject_consist}
According to \autoref{eq:evolv_flow_map}, a valid flow map $\phi_{t\to r}(x)$ must satisfy (1) trajectory consistency and (2) the boundary conditions $\phi_{t\to t}(x)=x$ and $\partial_t \phi_{t\to r}(x)\vert_{r=t}=u_t(x)$. While the boundary conditions can be easily supervised from data, enforcing trajectory consistency is considerably more challenging, which hinders the learning of accurate long-range dynamics.
In this section, we study how to introduce effective data-driven supervision for long-range trajectory consistency and present Euler Mean Flow with its theoretical justification and the $x_1$-prediction variant.

\subsection{Challenge of Trajectory Consistency} \label{sec:tranject_consist}
Consider a trajectory $(x_t)_{t\in[0,1]}$, with $x_t=\phi_t(x_0)$, that satisfies \autoref{eq:ODE_sampling}, where $\phi_t$ denotes the flow defined in \autoref{eq:evolv_flow_map}.  For any $t \le s \le r$ with $t,s,r \in [0,1]$, the following \textbf{trajectory consistency} holds:
\begin{equation}\label{eq:trajectory_consistency}
\phi_{t\to r}(x_t) = \phi_{s\to r}(x_s), x_s = \phi_{t\to s}(x_t)
\end{equation}
as illustrated in \autoref{fig:overview}. Taking the limit $s\to t$, this formulation admits a continuous formulation, 
\begin{equation}\label{eq:continous_trajectory_consistency}
    \begin{aligned}
        \partial_t \phi_{t\to r}(x) + \partial_x \phi_{t\to r}(x)(\partial_{s}\phi_{t\to s}(x))|_{s=t} = 0
    \end{aligned}
\end{equation}
Leveraging the trajectory consistency formulation, we can derive discrete trajectory consistency loss $\mathcal{L}^C(\theta)$ to train a long-range model $\phi^\theta_{t\to r}(x_t)$ that represents transitions across arbitrary temporal horizons $(t,r)$.  
\begin{equation}\label{eq:consistency_flow}
    \begin{aligned}
        \mathcal{L}^C(\theta)&=\mathbb{E}_{t,s,r, x_t = (1-t)x_0 + tx_1, x_1\sim p_{data}, x_0\sim p_0}\\
        &\frac{1}{w(t,r)}\|\phi^\theta_{t\to r}(x_t) - \phi^\theta_{s\to r}(\phi^\theta_{t\to s}(x))\|_2^2
    \end{aligned}
\end{equation}
where $\frac{1}{w(t,r)}$ denotes a tunable weight. For efficiency, parts of the formulation can be implemented with a stop-gradient operator (sg) without altering the underlying semantics.  

However, trajectory consistency alone is insufficient to uniquely determine the flow map. In particular, the consistency constraint $\mathcal{L}^C(\theta)$ admits infinitely many solutions and does not, by itself, introduce supervision from the data distribution. This ambiguity stems from two fundamental issues:  (1) Like velocity fields in Flow Matching, flow maps $\phi_{t\to r}(x)$ do not admit an analytic reference derived from the data distribution $p_{data}$ and dataset $\mathcal{D}$;  (2) Flow maps do not possess a conditional counterpart $\phi_{t\to r}(x|x_1)$ analogous to conditional velocities $u_{t}(x|x_1)$ which could calculate from dataset, as formalized below.
\begin{theorem}[Non-existence of conditional flow maps]\label{thm:non_exist}There exists no conditional flow maps $\phi_{t\to r}(x | x_{t_1})$ that simultaneously (i) is consistent with the conditional velocity $u(x|x_1)$ under \autoref{eq:evolv_flow_map}, and (ii) satisfies the consistency relation $\phi_{t\to r}(x)=\mathbb{E}_{x_{1}\sim p_t(x_1|x)}[\phi_{t\to r}(x|x_1)]$ with marginal flow maps. As a result, a self-consistent conditional cumulative field does not exist.  (See \autoref{thm:non_exist_proof} for a proof.)
\end{theorem}

Existing methods resolve this indeterminacy through two main strategies. \textbf{Progressive Extension} methods, such as Split-Mean Flow (SplitMF)~\cite{guo2025splitmeanflow} and ShortCut~\cite{frans2025shortcut}, learn the instantaneous velocity $\partial_s \phi_{t\to s}(x)\vert_{s=t}=u_t(x)$ and progressively extend it to longer horizons using the semigroup constraint in \autoref{eq:consistency_flow}. While effective in practice, these methods rely on indirect supervision accumulated from local dynamics, leading to weak long-range constraints and error accumulation. In contrast, \textbf{Continuous-Equation-Based Formulations}, exemplified by MeanFlow, derive long-range objectives from the continuous consistency equation in \autoref{eq:continous_trajectory_consistency} and provide more direct supervision of long-range flow maps, but require explicit gradient computation via Jacobian--vector products (JVPs), incurring high overhead and unstable optimization, particularly in sparse settings.




\subsection{Euler Mean Flow} \label{sec:u_pred}

To address these issues, we propose the Euler Mean Flows (EMF) framework.  Our key idea is to start from the semigroup  objective in~\autoref{eq:consistency_flow} and reformulate this objective via a local linear approximation, which enables direct supervision from data.  We also provide a rigorous theoretical justification for the validity of this approximation in \autoref{thm:validity} under reasonable \autoref{assupmption} on the flow maps.

\begin{theorem}[Local Linear Approximation]\label{thm:local_approximation}
Let $f : X \to Y$ be a smooth mapping between finite-dimensional spaces, and let $\mathbf{x}_0 \in X$.  When $\mathbf{x}$ is sufficiently close to $\mathbf{x}_0$, $f$ can be approximated by a linear function of the perturbation:
\begin{equation}
    \begin{aligned}
        f(\mathbf{x})\approx f(\mathbf{x}_0) + Df(\mathbf{x}_0)(\mathbf{x}-\mathbf{x}_0) + o(\|x-\mathbf{x}_0\|) .
    \end{aligned}
\end{equation}
which means in the small-perturbation limit, nonlinear effects enter only at higher order, and the local behavior of $f$ is governed by its linearization.
\end{theorem}

To reformulate the trajectory consistency objective, we follow MeanFlow~\cite{geng2025mean} and define the mean velocity field $u_{t \to r}(x) = \frac{\phi_{t \to r}(x) - x}{r - t}$.  Under this definition, the trajectory consistency relation can be rewritten as:
\begin{equation}\label{eq:traj_u_consistency}
    \begin{aligned}
        (r-t)u_{t\to r}(x_t)= (s-t)u_{t\to s}(x_t) + (r-s)u_{s\to r}(x_s)
    \end{aligned}
\end{equation}
Dividing both sides by $s-t$, we obtain
\begin{equation}\label{eq:traj_u_dev}
    \begin{aligned}
        u_{t\to s}(x_t) 
        &= (r-s)\frac{u_{t\to r}(x_t) - u_{s\to r}(x_s)}{(s-t)}+u_{t\to r}(x_t)\\
    \end{aligned}
\end{equation}
Unlike Shortcut and SplitMF, in our EMF we choose $s$ and $t$ to be close by setting $s = t + \Delta t$ with a small fixed step size $\Delta t$. We then apply \autoref{thm:local_approximation} to obtain a local approximation of the flow maps with respect to $s$: $\phi_{t\to s}(x)\approx \phi_{t\to t}(x) + \frac{\partial \phi_{t\to s}(x)}{\partial s}|_{s=t}(s-t)$.  Substituting this into the relation $\phi_{t\to s}(x) = (s-t) u_{t\to s}(x) + x$ between flows and average velocity yields $u_{t\to s}(x) \approx u_{t\to t}(x)$ when $s$ is sufficiently close to $t$. Based on this approximation, we obtain the following approximation with \autoref{eq:traj_u_dev}:
\begin{equation}\label{eq:EFM_equation}
\resizebox{\linewidth}{!}{$
\begin{aligned}
    \highlight{u_{t\to t}}(x_t)  & \approx (r-t-\Delta t)\frac{u_{t\to r}(x_t) - u_{t+\Delta t\to r}(x_{t+\Delta t})}{\Delta t}
    +u_{t\to r}(x_t)\\
   u_{t\to r}(x_t) &\approx \highlight{u_{t\to t}}(x_t)
   + (r-t-\Delta t)\frac{u_{t+\Delta t\to r}(x_{t+\Delta t}) - u_{t\to r}(x_t)}{\Delta t}
\end{aligned}
$}
\end{equation}
where $x_{t+\Delta t}$ is calculated as $x_{t+\Delta t} = \Delta tu_{t\to t+\Delta t}(x_t) + x_t \approx \Delta t \highlight{u_{t\to t}}(x_t) + x_t$.  In the above derivation, the highlighted velocity field $\highlight{u_{t\to t}}$ is obtained using the local linear approximation in \autoref{thm:local_approximation}.  Similar to MeanFlow, we replace $u_{t\to t}$ on the right-hand side of \autoref{eq:EFM_equation} with the conditional instantaneous velocity to obtain supervision from the dataset, which leads to the following loss function
\begin{equation}\label{eq:EFM_loss_equation}
\resizebox{\linewidth}{!}{$
\begin{aligned}
        &\mathcal{L}^{E}(\theta) = \mathbb{E}_{t,r,x_1 \sim p_1, x \sim p_t(x| x_1), x'=sg(\Delta t u^\theta_{t\to t}(x))+x}\\
        &[|u_{t\to r}^\theta(x)  - (u_t(x | x_1)+ (r-t-\Delta t)_+ \text{sg}(\frac{u^\theta_{t+\Delta t\to r}(x')-u^\theta_{t\to r}(x)}{\Delta t}))\|^2]  
\end{aligned}
$}
\end{equation}
Following MeanFlow, we sample a fraction of training pairs with $r = t$. With the positive clamp $(r - t - \Delta t)_+$, the proposed loss \autoref{eq:EFM_loss_equation} reduces to the Flow Matching objective $\|u^\theta_{t\to t}(x) - u_t(x|x_1)\|^2$ when $r = t$. This encourages $u^\theta_{t\to t}(x)$ to accurately learn the instantaneous velocity $u_t(x)$, which plays a crucial role in both the theoretical correctness and the stability of practical training. 

\begin{figure*}
    \centering
    \includegraphics[width=0.99\linewidth]{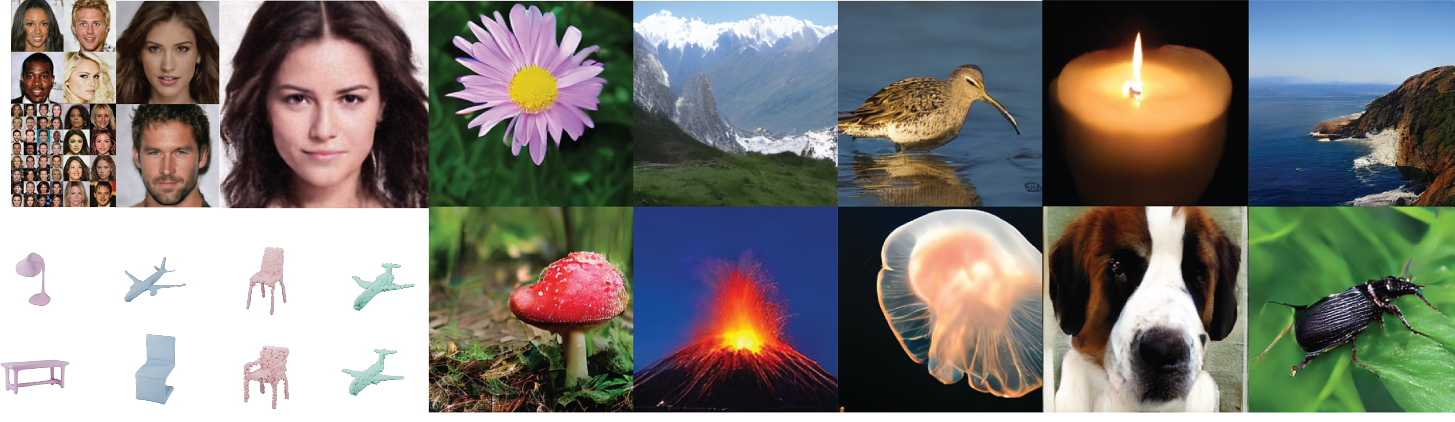}
    \caption{We present 1-step generation results of our EMF method for functional image generation (top left), SDF generation conditioned on 64 surface points (bottom left), unconditional point cloud generation on ShapeNet~\cite{chang2015shapenet} (bottom right), and ImageNet~\cite{deng2009imagenet} class-conditional generation (right).}
    \label{fig:all_results}
    \vspace{-5pt}
\end{figure*}

To theoretically justify the validity of this loss, we first introduce the following assumption on $u^\theta_{t\to r}$, which is empirically verified in \autoref{sec:assumtion_validation} ($M_g\sim 1e-3$, $M_x\sim 1e-4$ and $M_x\sim 1e1$).
\begin{assumption}[Assumption of $u^\theta_{t\to r}$]\label{assupmption}
We assume that $u^\theta_{t\to r}$ is differentiable with respect to its parameters and satisfies the following regularity conditions: (1) $M_g = \sqrt{\mathbb{E}_{t,r,x \sim  p_t(x) }[\frac{1}{m}\|\nabla_\theta u_{t\to r}^\theta(x)\|_2^2]} < +\infty$,(2) $M_x = \sqrt{\mathbb{E}_{t,r,x \sim  p_t(x) }[\frac{1}{m}\|\partial_x u_{t+\Delta t\to r}^\theta(x')\|_2^2]} < +\infty$, (3) $M_t = \sqrt{\mathbb{E}_{t,r,x \sim  p_t(x) }[\|\partial_s u_{t\to s}^\theta|_{s=t}\|^2]} < +\infty$
where $x' = x + \Delta tu_{t\to t}^\theta(x)$, $m$ is the model size and and $\|\cdot\|_2$ denotes the matrix $2$-norm.
\end{assumption}

Next, we show that,  up to an $O(\Delta t)$ error, the proposed loss serves as a valid surrogate for the trajectory consistency objective and leads to comparable optimization behavior. We begin with the following lemma.
\begin{lemma}\label{lemma_for_u}
With $M_g <+\infty$ holds in \autoref{assupmption}, our Euler Mean Flow loss $\mathcal{L}^{E}(\theta)$ and the approximated trajectory consistency loss $\mathcal{L}^{\tilde C}(\theta)$ satisfy
 \begin{equation}
     \resizebox{\linewidth}{!}{$\begin{aligned}
         D(\nabla  \mathcal{L}^{E}(\theta),\nabla  \mathcal{L}^{ \tilde C}(\theta)) \le M_g \sqrt{\mathbb{E}_{t,r, x \sim  p_t(x)} [\|u_{t\to t}^\theta(x)-u_t(x)\|^2]}
     \end{aligned}
     $}
 \end{equation}
where $D$ denotes the root mean squared error (RMSE). Consequently, during training, if $\|u^\theta_{t\to t}(x) - u_t(x)\|^2 \to 0$, then $\mathcal{L}^{E}(\theta)$ and $\mathcal{L}^{\tilde C}(\theta)$ share the same optimal target at $\theta$.  The term $u_t(x)$ denotes the reference velocity at $x$, defined as $u_t(x) = \mathbb{E}_{x_1 \sim p(x_1|x)}\big[u_t(x|x_1)\big]$,  which is intractable to compute analytically. (see \autoref{sec:proof_lemma_for_u} for proof.)

Here, the approximated trajectory consistency loss $\mathcal{L}^{\tilde C}(\theta)$ are defined as 
\begin{equation}\label{eq:define_approx_loss}
   \resizebox{\linewidth}{!}{$
    \begin{aligned}
  &\mathcal{L}^{\tilde C}(\theta) = \mathbb{E}_{t,r,x \sim  p_t(x), x'=sg(\Delta t u^\theta_{t\to t}(x))+x}\\
  &[|u_{t\to r}^\theta(x)  - (u^\theta_{t\to t}(x) + (r-t-\Delta t) \text{sg}(\frac{u^\theta_{t+\Delta t\to r}(x')-u^\theta_{t\to r}(x)}{\Delta t}))\|^2] 
    \end{aligned}
    $}
\end{equation}
It is straightforward to verify that the loss $\mathcal{L}^{\tilde C}(\theta)$ is the mean-velocity formulation of $\mathcal{L}^{C}(\theta)$ under the local linear approximation in \autoref{eq:EFM_equation}, expressed via $u_{t\to r}(x) = \frac{\phi_{t\to r}(x)-x}{r-t}$, and differs by a temporal scaling factor $1/\Delta t$.
\end{lemma}
The above lemma links the surrogate Euler Mean Flow loss $\mathcal{L}^{E}(\theta)$ to the approximated trajectory consistency loss $\mathcal{L}^{\tilde C}(\theta)$. Building on this result, we can further relate  $\mathcal{L}^{E}(\theta)$ to the original trajectory consistency objective $\mathcal{L}^{C}(\theta)$ thereby showing that $\mathcal{L}^{E}(\theta)$ serves as a valid surrogate for the trajectory consistency objective. 
\begin{theorem}[Surrogate Loss Validity]\label{thm:validity} With $M_g =  < +\infty$, $M_x < +\infty$, and $M_t < +\infty$ hold in \autoref{assupmption}, Our Euler Mean Flow loss $\mathcal{L}^{E}(\theta)$ and the trajectory consistency loss $\mathcal{L}^{C}(\theta)$ satisfy
\begin{equation}
   \resizebox{\linewidth}{!}{$
    \begin{aligned}
        &D(\nabla_\theta L^E(\theta),\nabla_\theta L^C(\theta)) \le  M_g \sqrt{\mathbb{E}_{t,r, x \sim  p_t(x)} [\|u_{t\to t}^\theta(x)-u_t(x)\|^2]}\\ 
        &+ (M_gM_t + M_x M_t)\Delta t + O(\Delta t^2)
    \end{aligned}$}
\end{equation}
see \autoref{sec:u_validity_proof} for proof.
\end{theorem}
\autoref{thm:validity} shows that, provided condition $\mathbb{E}_{t,x \sim  p_t(x)} [\|u_{t\to t}^\theta(x)-u_t(x)\|^2] \to 0$ holds during training, $\mathcal{L}^{E}(\theta)$ serves as a valid surrogate for $\mathcal{L}^{C}(\theta)$ up to $O(\Delta t)$. This condition can be promoted by local linear approximation and by mixing a fixed proportion of samples with $r=t$ in time sampling, as discussed below. 

\paragraph{Rationale for the Local Linear Approximation}
In \autoref{eq:EFM_equation}, we apply the local linear approximation in two places. First, we approximate $u_{t\to s}(x)$ in the summation by $u_{t\to t}(x)$, enabling conditioning as $u(x\mid x_1)$ and introducing direct data supervision for long-range trajectory consistency. This choice reduces the objective to standard Flow Matching when $r=t$, allowing $u^\theta_{t\to t}(x)$ to be optimized toward $u_t(x)$ and providing the boundary condition required by \autoref{thm:validity}. Second, in the update $x_{t+\Delta t}=x_t+\Delta t\,u_{t\to t+\Delta t}(x_t)$, we approximate $u_{t\to t+\Delta t}$ by $u_{t\to t}$. This approximation is motivated by efficiency, as $u_{t\to t}(x)$ is substantially easier to estimate under memory constraints, while using $u_{t\to t+\Delta t}$ offers no noticeable quality improvement (see \autoref{tab:memory_time_comparison2}).


\paragraph{Comparison with Previous Methods} To provide an intuitive comparison highlighting the key differences among related methods, we summarize them in \autoref{tab:method_comparison}.
\begin{table}[H]
    \centering
    \caption{Comparison of flow map--based one-step methods by training strategy, JVP usage, and prediction type.}
    \label{tab:method_comparison}
    \resizebox{0.99\linewidth}{!}{
    \begin{tabular}{cccccc}
        \toprule
        \textbf{Method} & \textbf{Scratch} & \textbf{Distill} & \textbf{JVP-free} & 
        \textbf{$u$-pred} & 
        \textbf{$x_1$-pred}\\
        \midrule
        MF \cite{geng2025mean, geng2025improved} & $\checkmark$ & $\checkmark$ & $\times$ & $\checkmark$ & $\times$\\ 
        $\alpha-$Flow \cite{zhang2025alphaflow} & $\checkmark$ & $\checkmark$ & $\times$ & $\checkmark$ & $\times$\\
        ShortCut \cite{frans2025shortcut} & $\checkmark$ & $\checkmark$ & $\times$ & $\checkmark$ & $\times$\\
        SplitMF \cite{guo2025splitmeanflow}& $\times$ & $\checkmark$ & $\checkmark$ & $\checkmark$ & $\times$\\
        Ours& $\checkmark$ & $\checkmark$ & $\checkmark$ & $\checkmark$ & $\checkmark$\\
        \bottomrule
    \end{tabular}
    }    
\end{table}

\subsection{$x_1$-prediction Euler Mean Flows}\label{sec:x1_pred}
Whether minimizing $\mathcal{L}^E(\theta)$ in \autoref{eq:EFM_loss_equation} correctly enforces trajectory consistency depends on condition $\mathbb{E}_{t,x \sim  p_t(x)} [\|u_{t\to t}^\theta(x)-u_t(x)\|^2] \to 0$, namely that $u^\theta_{t\to t}(x)$ accurately approximates the reference instantaneous velocity $u_t(x)$. However, in several applications, including pixel-space image generation \autoref{sec:pixel_experiment} and our SDF experiments \autoref{sec:SDF_experiment}, $u$-prediction fails to reliably learn $u_t(x)$, as also discussed in~\cite{li2025back} from a data-manifold perspective. As a result, the loss $\mathcal{L}^E(\theta)$ that relies on accurate velocity learning may become ineffective.

To overcome this limitation, inspired by~\cite{li2025back}, we adopt an $x_1$-prediction formulation and introduce the $x_1$-prediction Euler mean flow. Specifically, we define the $x_1$-prediction mean field
\begin{equation}\label{traj_consist_x1}
    \tilde{x}_{t\to r}(x) = (1-t)\frac{\phi_{t\to r}(x)-x}{r-t}+ x
\end{equation}
where $\tilde{x}_{t\to r}(x)$ satisfies $\tilde{x}_{t\to r}(x) = (1-t)\,u_{t\to r}(x) + x$, which mirrors the instantaneous $x_1$-prediction flow-matching field $\tilde{x}_{t}(x) = (1-t)\,u_{t}(x) + x$.  Under this formulation, the trajectory consistency relation can be rewritten as
\begin{equation}\label{eq:traj_x1_dev}
     \resizebox{\linewidth}{!}{$\begin{aligned}
       \tilde{x}_{t\to r}(x_t)&= \tilde{x}_{t\to s}(x_t) + (r-s)\frac{(1-t)}{(1-r)}\frac{\tilde{x}_{s\to r}(x_s)-\tilde{x}_{t\to r}(x_t)}{s-t}\\
    \end{aligned}$}
\end{equation}
Following the $u$-prediction case, we set $s=t+\Delta t$ and use a local approximation of the flow map, giving $\tilde{x}_{t\to s} \approx \tilde{x}_{t\to t}$ for small $\Delta t$. This leads to the approximation in \autoref{traj_consist_x1}
\begin{equation}
     \resizebox{\linewidth}{!}{$\begin{aligned}
        \tilde{x}_{t\to r}(x_t)&\approx \highlight{\tilde{x}_{t\to t}}(x_t) + (r-t-\Delta t)\frac{(1-t)}{(1-r)}\frac{\tilde{x}_{t+\Delta t\to r}(x_{t+\Delta t})-\tilde{x}_{t\to r}(x_t)}{\Delta t}\\
    \end{aligned}$}
\end{equation}
where $x_{t+\Delta t}$ is calculated as $x_{t+\Delta t} = \frac{\Delta t}{1-t}(\tilde{x}_{t\to s}(x_t) - x_t)+x_t\approx \frac{\Delta t}{1-t}(\highlight{\tilde{x}_{t\to t}}(x_t) - x_t)+x_t$.  The highlighted field $\highlight{\tilde{x}_{t\to t}}$ is obtained using the local linear approximation for $\tilde{x}_{t\to r}$.   Similar to $u-$prediction version, we replace $\tilde{x}_{t\to t}$ on the right-hand side of \autoref{eq:EFM_equation} with the conditional instantaneous $\tilde{x}$ field, namely $\tilde{x}(x|x_1) = x_1$, to obtain supervision from the dataset, which leads to the following loss function
 \begin{equation}\label{eq:EFM_loss_equation_x_1}
    \resizebox{\linewidth}{!}{$ \begin{aligned}
        &\mathcal{L}^{E'}(\theta) = \mathbb{E}_{t,r,x_1 \sim p_1, x \sim p_t(x| x_1), x'=sg(\Delta t \frac{\tilde{x}^\theta_{t\to t}(x)-x}{1-t})+x}[|\tilde{x}^{\theta}_{t\to r}(x)  - \\
        &(x_t^1(x|x_1)+ (r-t-\Delta t)_+ \frac{1-t}{1-r}\text{sg}(\frac{\tilde{x}^{\theta}_{t+\Delta t\to r}(x')-\tilde{x}^{\theta}_{t\to r}(x)}{\Delta t}))\|^2]  
    \end{aligned}$}
\end{equation}
As in the $u$-prediction setting, we sample a fraction of training pairs with $r = t$, such that \autoref{eq:EFM_loss_equation_x_1} reduces to the $x_1$-prediction flow-matching objective $\|\tilde{x}^{\theta}_{t\to t}(x) - \tilde{x}_t(x \mid x_1)\|^2$ when $r = t$. For the field $\tilde{x}^{\theta}$, we make the following assumption, under which a surrogate loss validity result analogous to that of the $u$-prediction EMF can be established.

\begin{assumption}[Assumption of $\tilde{x}^{\theta}_{t\to r}$]\label{assupmption_x1}
We assume that $\tilde{x}^\theta_{t\to r}$ is differentiable with respect to its parameters and satisfies the following regularity conditions: (1) $M'_g = \sqrt{\mathbb{E}_{t,r,x \sim  p_t(x) }[\frac{1}{m}\|\nabla_\theta \tilde{x}_{t\to r}^\theta(x)\|_2^2]} < +\infty$,  (2) $M'_x = \frac{1}{1-r}\sqrt{\mathbb{E}_{t,r,x \sim  p_t(x) }[\frac{1}{m}\|\partial_x \tilde{x}_{t+\Delta t\to r}^\theta(x')\|_2^2]} < +\infty$,  (3) $M'_t = \sqrt{\mathbb{E}_{t,r,x \sim  p_t(x) }[\|\partial_s \tilde{x}_{t\to s}^\theta|_{s=t}\|^2]} < +\infty$, where $x' = x + \Delta t\frac{\tilde{x}_{t\to t}^\theta(x)-x}{1-t}$, $m$ is the model size and and $\|\cdot\|_2$ denotes the matrix $2$-norm (spectral norm).
\end{assumption}

\begin{theorem}[Surrogate Loss Validity for $x_1$-Prediction]\label{thm:validity_x1} With $M'_g  < +\infty$, $M'_x < +\infty$, and $M'_t < +\infty$ hold in \autoref{assupmption_x1} and \autoref{lemma_for_x1}, our Euler Mean Flow loss $\mathcal{L}^{E}(\theta)$ and the trajectory consistency loss $\mathcal{L}^{C}(\theta)$ satisfy
\begin{equation}
\resizebox{\linewidth}{!}{$
    \begin{aligned}
         &MSE(\nabla_\theta L^E(\theta) , \nabla_\theta L^C(\theta))\le  M_g \sqrt{\mathbb{E}_{t,r, x \sim  p_t(x)} [\|\tilde{x}_{t\to t}^\theta(x)-\tilde{x}_t(x)\|^2]}\\
        &+ (M_gM_t + M_x M_t)\Delta t + O(\Delta t^2)
    \end{aligned}$}
\end{equation}
See \autoref{thm:x1_validity_proof} for proof.
\end{theorem}

\paragraph{Optimization of Time Weights}
When $r=t$, $\mathcal{L}^{E'}$ in \autoref{eq:EFM_loss_equation_x_1} reduces to the $x_1$-prediction flow-matching objective $\|\tilde x^\theta_{t\to t}(x) - \tilde{x}_t(x|x_1)\|^2$. As shown in \autoref{thm:validity_x1}, enforcing trajectory consistency further depends on how well $\tilde x^\theta_{t\to t}(x)$ approximates $\tilde{x}_t(x|x_1)$.  However, \cite{li2025back} demonstrate that loss $\|\tilde x^\theta_{t\to t}(x) - \tilde{x}_t(x|x_1)\|^2$ yields suboptimal fitting, and to mitigate this issue, \cite{li2025back} introduces a time weight $\frac{1}{(1-t)^2}$, leading to the weighted loss $\frac{1}{(1-t)^2}\|\tilde x^\theta_{t\to t}(x) - \tilde{x}_t(x|x_1)\|^2$ (referred to as the $x$-pred \& $u$-loss).  Following this strategy,  we adopt the same strategy and incorporate the time weight $\|\tilde x^\theta_{t\to t}(x) - \tilde{x}_t(x|x_1)\|^2$ into $\mathcal{L}^{E'}$ in \autoref{eq:EFM_loss_equation_x_1} to improve the learning of $\tilde x^\theta_{t\to t}(x)$. For numerical stability, we clamp the denominators $1-t$ and $1-r$ to a minimum value of 0.02.

\subsection{Algorithm}
Building on the above discussion, we derive the training and sampling procedures of Euler Mean Flows for both conditional and unconditional generation, as summarized in Algorithms~\ref{alg:training} and~\ref{alg:sampling}.  For conditional generation, following~\cite{geng2025mean}, we adopt classifier-free guidance (CFG) during training, with an effective guidance scale given by $w' = \frac{w}{1-k}$, where $w$ and $k$ denote the CFG coefficients.
Additional details on CFG, adaptive loss weighting, and time sampling strategies are provided in~\autoref{sec:alg_details}.

\begin{algorithm}[t]
\caption{Euler Mean Flow: Training\\
\footnotesize\emph{Highlighted steps are used for conditional generation. $C$ represents the class label, and $C_0$ the corresponding unconditional label. $w$ and $k$ are parameter for CFG}}
\label{alg:training}
\begin{algorithmic}[1]
\REQUIRE Dataset $\mathcal{D}$, parameters $\theta$, learning rate $\eta$, noise sampler $\mathcal{N}$, time sampler $\mathcal{T}$  
\REPEAT
    \STATE Sample $x_1,\highlight{C} \sim \mathcal{D}$, $x_0\sim \mathcal{N}$, $t, r\sim \mathcal{T}$
    \STATE $x_t \leftarrow  (1-t)x_t + (1-t)x_t$ 
    \IF{$u$-prediction}
        \STATE $\highlight{U^u \leftarrow  u_{t\to t}^\theta(x_t, C_0)}$, $U^c \leftarrow u_{t\to t}^\theta(x_t, \highlight{C})$
        \STATE $u_t(x|x_1) \!\leftarrow\! \highlight{(1\!-\!w\!-\!k)U^u + kU^c + w}(x_1-x_0)$
        \STATE $x_{t+\Delta t} \leftarrow  \Delta tU^c_t + x_t$
        \STATE $\mathcal{L} \leftarrow  \|u_{t\to r}^\theta(x_t,\highlight{C}) -sg(u_t(x|x_1)+(r-t-$\\$\qquad\Delta t)_+\frac{u_{t + \Delta t\to r}^\theta(x_{t+\Delta t},\highlight{C}) - u_{t \to r}^\theta(x_{t},\highlight{C})}{\Delta t})\|^2$\\
    \hspace{-3.7mm}\textbf{else if }{$x_1$-prediction} \textbf{then}
        \STATE $\highlight{X^u \leftarrow  \tilde{x}^{\theta}_{t\to t}(x_t, C_0)}$, $X^c \leftarrow \tilde{x}^{\theta}_{t\to t}(x_t, \highlight{C})$
        \STATE $\tilde{x}_t(x|x_1) \!\leftarrow\! \highlight{(1\!-\!w\!-\!k)X^u + kX^c + w}x_1$
        \STATE $x_{t+\Delta t} \leftarrow  \Delta t\frac{X^c-x_t}{1-x_t} + x_t$
        \STATE $\mathcal{L} \leftarrow  \|\tilde{x}^{\theta}_{t\to r}(x_t,\highlight{C}) -sg(\tilde{x}_t(x|x_1)+(r-t-$\\$\qquad\Delta t)_+\frac{1-t}{1-r}\frac{\tilde{x}^{\theta}_{t + \Delta t\to r}(x_{t+\Delta t}) -\tilde{x}^{\theta}_{t \to r}(x_{t})}{\Delta t})\|^2$
    \ENDIF
    \STATE $\theta \gets \theta - \eta\nabla_\theta\mathcal{L}$
\UNTIL{convergence}
\end{algorithmic}
\end{algorithm}

\section{JVP-Free Training}
\subsection{Training Speed and Memory Efficiency}

The comparison of memory and computational cost is reported in \autoref{tab:memory_time_comparison}. Here, we further analyze the memory and computational cost of our training algorithm in \autoref{alg:training}. For conditional generation, our training requires three stop-gradient forward passes $u^\theta_{t\to t}(x,C)$, $u^\theta_{t\to r}(x,C_0)$ and $u_{t+\Delta t\to r}^\theta(x_{t+\Delta t},C)$ and one optimized forward pass $u_{t\to r}^\theta(x,C)$, while MeanFlow~\cite{geng2025mean} requires two stop-gradient forward passes $u^\theta_{t\to t}(x,C)$, $u^\theta_{t\to r}(x,C_0)$, one JVP computation, and one optimized forward pass $u_{t\to r}^\theta(x,C)$. Although the latter two are jointly computed via \texttt{torch.jvp} in PyTorch, the JVP operation still introduces non-negligible overhead. Compared to MeanFlow, our method replaces one JVP computation with an additional stop-gradient forward pass, resulting in lower memory and runtime costs.  Moreover, by avoiding JVP operations, our method is compatible with FlashAttention, whereas MeanFlow does not support FlashAttention due to its reliance on JVP.  

\begin{wrapfigure}{r}{0.35\linewidth}
  \centering
  \vspace{-5pt}
  \includegraphics[width=\linewidth]{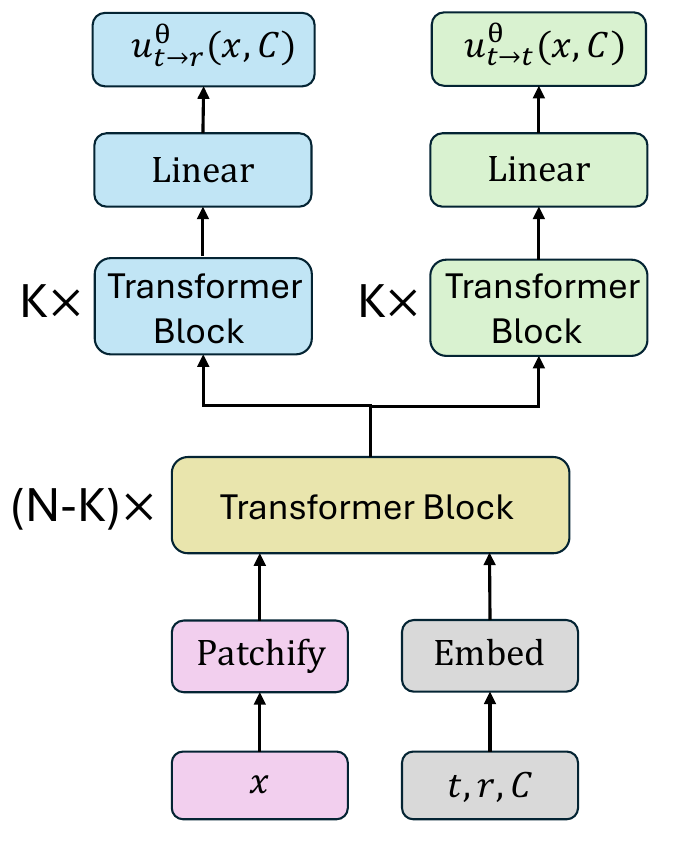}
  \vspace{-25pt}
  \caption{Auxiliary Branch for $u_{t\to t}^\theta(x)$ Prediction}
  \label{fig:aux_model}
\end{wrapfigure}

For unconditional generation, our method requires two stop-gradient forward passes $u^\theta_{t\to t}(x)$ and $u_{t+\Delta t\to r}^\theta(x_{t+\Delta t})$ and one optimized forward pass $u_{t\to r}^\theta(x)$, whereas MeanFlow only requires one JVP and one optimized forward pass $u_{t\to r}^\theta(x)$.  Although our approach remains more efficient, the efficiency gap becomes smaller.  To further reduce the cost, we adopt the strategy of~\cite{geng2025improved} by introducing a lightweight auxiliary branch to predict $u^\theta_{t\to t}(x)$, while the main branch predicts $u^\theta_{t\to r}(x)$. The auxiliary and main branches share forward computations, and an additional loss $\mathcal{L}^F$ is used to improve the approximation of $u^\theta_{t\to t}$. The final loss is given as $\mu_1\mathcal{L}^{EMF} + \mu_2 \mathcal{L}^F$, with hyperparameters $\mu_1$ and $\mu_2$, where we set $\mu_1=1$ and $\mu_2=1$ in practice. With this design, training only requires one stop-gradient forward pass and one optimized forward pass, leading to substantially reduced memory and computational cost.  

\subsection{Optimization Stability}
\begin{wrapfigure}{r}{0.5\linewidth}
  \centering
  \vspace{-5pt}
   \includegraphics[width=1.0\linewidth]{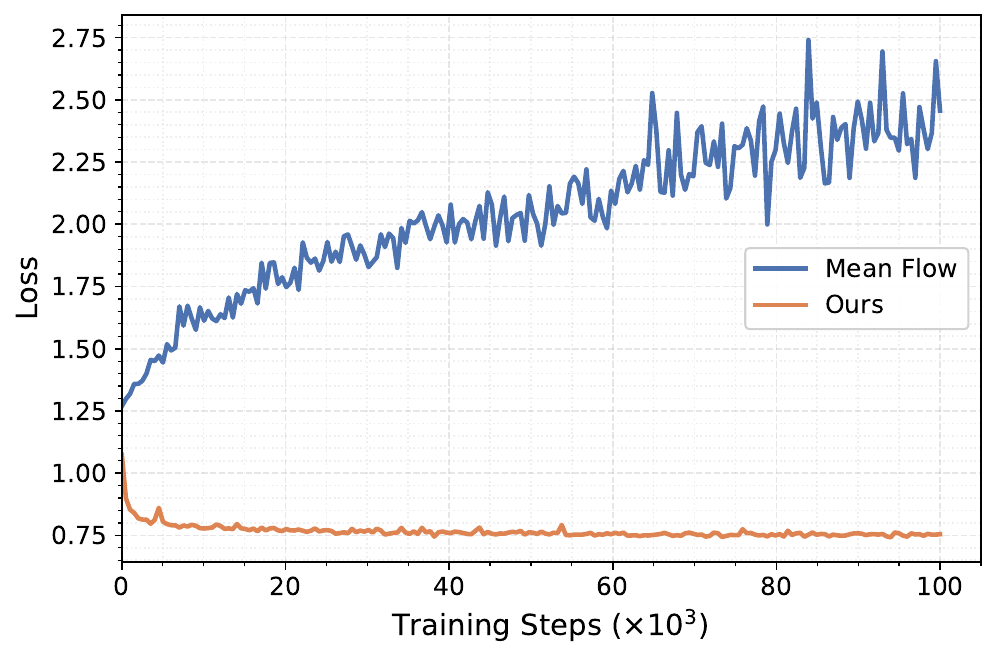}
  \vspace{-25pt}
    \caption{Training loss comparison between Euler Mean Flow and Mean Flow.
    }
    \label{fig:comparision_intuition}
\end{wrapfigure}
The original MeanFlow framework often exhibits anomalous loss escalation during training.
As shown in \autoref{fig:comparision_intuition}, the training loss of MeanFlow tends to increase abnormally as optimization progresses, even when adaptive loss weighting is applied for stabilization, resulting in high variance and unstable dynamics.
In contrast, our method achieves steadily decreasing loss with well-controlled variance, even without adaptive weighting.
Moreover, we observe that MeanFlow is prone to training collapse in image generation tasks, including both latent-space (\autoref{fig:training_dynamics_imageNet}) and pixel-space (\autoref{fig:dynamics_jit_compare}) settings, especially under mixed-precision training, whereas our approach remains robust.
As a result of its improved stability, our method consistently outperforms MeanFlow on both image generation \autoref{tab:training_objective_comparison} and SDF generation \autoref{tab:reconstruction_metrics} tasks.

\subsection{Broader Applications}
Many sparse computation libraries, such as PVCNN and TorchSparse, do not support JVP operations, limiting the applicability of MeanFlow in these domains.  In contrast, EMF is fully JVP-free and achieves strong performance on functional and point cloud generation tasks, while enabling efficient one-step and few-step generation in sparse settings (\autoref{sec:point_cloud}, \autoref{sec:functional_generation}).
\vspace{-5mm}
\section{Experiment}\label{sec:experiment}
\subsection{Validation}\label{sec:assumtion_validation}
Our theorems in \autoref{thm:validity} and \autoref{thm:validity_x1} rely on Assumptions \autoref{assupmption} and \autoref{assupmption_x1}, respectively. To validate these assumptions, we train a DiT-B/2 model on CelebA-HQ dataset and monitor the values of $M_g$ ($M_g'$), $M_x$ ($M_x'$), and $M_t$ ($M_t'$) throughout training. The training protocol, model architecture, and hyperparameters follow \autoref{sec:latent_image}.  To estimate the spectral norms in $M_g$($M_g'$) and $M_x$($M_x'$), for any matrix $M$, we randomly sample $n_1$ unit vectors $|v|=1$ and approximate $\|M]\|_2$ by $\max \|Mv\|$. For expectations of the form $\mathbb{E}_{t,r,x\sim p_t(x)}$, we instead evaluate $\mathbb{E}_{t,r,x\sim p_t(x\mid x_1),x_1\sim p{\text{data}}}$. We sample $n_2$ points from $t,r\sim\mathcal{T}$, draw $x_1\sim p_{\text{data}}$ and $x\sim p_t(x|x_1)$, and estimate the expectations via Monte Carlo averaging.  Results are reported in \autoref{assumption_valid} and \autoref{assumption_x1_valid}. Additional experimental details on memory and timing statistics are provided in \autoref{sec:additional_experiment}.
\vspace{-2.5mm}
\subsection{Applications}
\vspace{-1.25mm}
\subsubsection{Latent Space Image Generation}\label{sec:latent_image}
\vspace{-1.25mm}
We evaluate our method on latent-space image generation tasks using two datasets: ImageNet-1000~\cite{deng2009imagenet} and CelebA-HQ~\cite{liu2015faceattributes}, both resized to a resolution of $256 \times 256$. Following the latent-space generation paradigm, we adopt a DiT-B/2 backbone~\cite{peebles2023scalable} together with a standard pre-trained VAE from Stable Diffusion~\cite{rombach2022high}, which maps a $256 \times 256 \times 3$ image into a compact latent representation of size $32 \times 32 \times 4$.  For training efficiency, we employ mixed-precision training with FP16, in contrast to the FP32 training used in~\cite{geng2025mean}. Our method consistently outperforms existing approaches on both ImageNet-1000 and CelebA-HQ (see \autoref{tab:training_objective_comparison}). Moreover, as reflected in the training dynamics compared with MeanFlow \autoref{fig:comparision_intuition}, our method exhibits significantly improved optimization stability.

\vspace{-2.5mm}
\subsubsection{Pixel Space Image Generation}\label{sec:pixel_experiment}
\vspace{-2.5mm}
For pixel-space image generation, we adopt the JiT framework following \cite{li2025back}. JiT is a plain Vision Transformer that directly processes images as sequences of pixel patches, without relying on VAEs or other latent representations. To accommodate the high dimensionality of pixel-space generation, JiT employs relatively large patch sizes. We build our model upon JiT-B/16 and train it on the CelebA-HQ dataset at a resolution of $256\times 256$.  In the one-step generation setting, we observe behavior consistent with prior findings on JiT: the $u$-prediction variant of EMF produces images with significant noise and poor visual quality \autoref{fig:jit_u_celeb}. This further highlights the necessity of the $x_1$-prediction variant. A comprehensive comparison is provided in \autoref{tab:pixel_space_method_comparison_celeba}. Moreover, the training dynamics in \autoref{fig:training_dynamics_imageNet} show that our method achieves substantially improved stability compared to MeanFlow.

\vspace{-2.5mm}
\subsubsection{SDF Generation}\label{sec:SDF_experiment}
\vspace{-2.5mm}
Next, we evaluate our method on SDF generation. We adopt the Functional Diffusion framework~\cite{zhang2024functional}, 
in which the model is conditioned on a sparse set of observed surface points (64 points) and generates the complete SDF function from noise using an attention-based architecture. 
Experiments are conducted on the ShapeNet-CoreV2 dataset~\cite{chang2015shapenet} and evaluated using Chamfer Distance, F-score, and Boundary Loss, which measure surface accuracy and boundary fidelity (see \autoref{sec:sdf_generation} for details). As shown in \autoref{tab:training_objective_comparison}, our method significantly outperforms MeanFlow and achieves performance comparable to multi-step generation. We also apply the same framework to a 2D MNIST-based SDF generation task (\autoref{fig:minst}), where handwritten digits are converted into SDFs. In this case, the $u$-prediction variant of EMF sueffers from attention variance collapse during training, whereas only the $x_1$-prediction successfully generates high-quality shapes.

\vspace{-2.5mm}
\subsubsection{Point Cloud Generation}\label{sec:point_cloud}
\vspace{-2.5mm}
To demonstrate the applicability of our method to sparse and irregular domains, we apply EMF to point cloud generation. We adopt the Latent Point Diffusion Model (LION) architecture~\cite{vahdat2022lion}, which builds on a VAE that encodes each shape into a hierarchical latent representation comprising a global shape latent and a point-structured latent point cloud.  We use pre-trained encoders and decoders based on Point-Voxel CNNs (PVCNNs) and fine-tune both the global and point cloud latents using EMF on the airplane and chair categories. Training and model details are provided in \autoref{sec:point_cloud}.  For evaluation, we compare generated samples against reference sets using Coverage (COV) and 1-Nearest Neighbor Accuracy (1-NNA), computed with either Chamfer Distance or Earth Mover’s Distance, to assess sample diversity and distributional alignment. As shown in \autoref{fig:comparision_intuition}, our method achieves competitive performance among one-step generation approaches.

\subsubsection{Function-Based Image Generation}\label{sec:functional_generation}
We further evaluate our method on sparse domains via function-based image generation using an architecture built on Infty-Diff~\cite{bond2024infty}. Infty-Diff represents images as continuous functions defined over randomly sampled pixel coordinates and employs a hybrid sparse--dense architecture that combines sparse neural operators with a dense convolutional backbone for global feature extraction. Sparse features are interpolated to a coarse grid for dense processing and mapped back to the original coordinates, enabling efficient learning from partial observations. We conduct experiments on FFHQ~\cite{karras2019style} and CelebA-HQ at $256\times256$ resolution, randomly sampling $25\%$ of pixels during training, and exploit the resolution-invariant nature of functional representations to generate images at multiple resolutions (see \autoref{sec:functional} for details). As shown in \autoref{fig:celeb_functional}, our method achieves competitive performance in one-step functional image generation compared to existing approaches.

\vspace{-2.5mm}
\section{Conclusion}
\vspace{-2.5mm}
We proposed EMF as a trajectory-consistent framework for efficient one-step and few-step generation, enabling direct data supervision of long-range flow maps via a local linear approximation of the semigroup objective. EMF avoids explicit derivative computation through a unified, JVP-free training scheme with theoretical guarantees, and extending it to broader tasks, larger models, and more general theoretical settings is an important direction for future work.

\newpage
\section*{Impact Statement}
This paper presents work whose goal is to advance the field of machine learning. There are many potential societal consequences of our work, none of which we feel must be specifically highlighted here.

\bibliography{example_paper,3D_generation}

\begin{thebibliography}{67}
\providecommand{\natexlab}[1]{#1}
\providecommand{\url}[1]{\texttt{#1}}
\expandafter\ifx\csname urlstyle\endcsname\relax
  \providecommand{\doi}[1]{doi: #1}\else
  \providecommand{\doi}{doi: \begingroup \urlstyle{rm}\Url}\fi

\bibitem[Achlioptas et~al.(2018)Achlioptas, Diamanti, Mitliagkas, and Guibas]{achlioptas2018learning}
Achlioptas, P., Diamanti, O., Mitliagkas, I., and Guibas, L.
\newblock Learning representations and generative models for 3d point clouds.
\newblock In \emph{International conference on machine learning}, pp.\  40--49. PMLR, 2018.

\bibitem[Albergo \& Vanden-Eijnden(2023)Albergo and Vanden-Eijnden]{albergo2023building}
Albergo, M.~S. and Vanden-Eijnden, E.
\newblock Building normalizing flows with stochastic interpolants.
\newblock In \emph{International Conference on Learning Representations (ICLR)}, 2023.

\bibitem[Boffi et~al.(2025)Boffi, Albergo, and Vanden-Eijnden]{boffi2025flow}
Boffi, N.~M., Albergo, M.~S., and Vanden-Eijnden, E.
\newblock Flow map matching with stochastic interpolants: A mathematical framework for consistency models.
\newblock \emph{Transactions on Machine Learning Research (TMLR)}, 2025.

\bibitem[Bond-Taylor \& Willcocks(2024)Bond-Taylor and Willcocks]{bond2024infty}
Bond-Taylor, S. and Willcocks, C.~G.
\newblock $\infty$-diff: Infinite resolution diffusion with subsampled mollified states.
\newblock In \emph{International Conference on Learning Representations (ICLR)}, 2024.

\bibitem[Cai et~al.(2020)Cai, Yang, Averbuch-Elor, Hao, Belongie, Snavely, and Hariharan]{cai2020learning}
Cai, R., Yang, G., Averbuch-Elor, H., Hao, Z., Belongie, S., Snavely, N., and Hariharan, B.
\newblock Learning gradient fields for shape generation.
\newblock In \emph{European Conference on Computer Vision}, pp.\  364--381. Springer, 2020.

\bibitem[Chang et~al.(2015)Chang, Funkhouser, Guibas, Hanrahan, Huang, Li, Savarese, Savva, Song, Su, et~al.]{chang2015shapenet}
Chang, A.~X., Funkhouser, T., Guibas, L., Hanrahan, P., Huang, Q., Li, Z., Savarese, S., Savva, M., Song, S., Su, H., et~al.
\newblock Shapenet: An information-rich 3d model repository.
\newblock \emph{arXiv preprint arXiv:1512.03012}, 2015.

\bibitem[Dao et~al.(2022)Dao, Fu, Ermon, Rudra, and Ré]{dao2022flash}
Dao, T., Fu, D.~Y., Ermon, S., Rudra, A., and Ré, C.
\newblock Flashattention: Fast and memory-efficient exact attention with io-awareness.
\newblock \emph{Advances in neural information processing systems}, 2022.

\bibitem[Deng et~al.(2009)Deng, Dong, Socher, Li, Li, and Fei-Fei]{deng2009imagenet}
Deng, J., Dong, W., Socher, R., Li, L.-J., Li, K., and Fei-Fei, L.
\newblock Imagenet: A large-scale hierarchical image database.
\newblock In \emph{IEEE Conference on Computer Vision and Pattern Recognition (CVPR)}, 2009.

\bibitem[Du et~al.(2021)Du, Collins, Tenenbaum, and Sitzmann]{du2021learning}
Du, Y., Collins, K., Tenenbaum, J., and Sitzmann, V.
\newblock Learning signal-agnostic manifolds of neural fields.
\newblock 2021.

\bibitem[Dupont et~al.(2022{\natexlab{a}})Dupont, Kim, Eslami, Rezende, and Rosenbaum]{dupont2022data}
Dupont, E., Kim, H., Eslami, S., Rezende, D., and Rosenbaum, D.
\newblock From data to functa: Your data point is a function and you can treat it like one.
\newblock \emph{International Conference on Machine Learning (ICML)}, 2022{\natexlab{a}}.

\bibitem[Dupont et~al.(2022{\natexlab{b}})Dupont, Teh, and Doucet]{dupont2022generative}
Dupont, E., Teh, Y.~W., and Doucet, A.
\newblock Generative models as distributions of functions.
\newblock In \emph{International Conference on Artificial Intelligence and Statistics (AISTATS)}, 2022{\natexlab{b}}.

\bibitem[Frans et~al.(2025)Frans, Hafner, Levine, and Abbeel]{frans2025shortcut}
Frans, K., Hafner, D., Levine, S., and Abbeel, P.
\newblock One step diffusion via shortcut models.
\newblock In \emph{International Conference on Learning Representations (ICLR)}, 2025.

\bibitem[Geng et~al.(2023)Geng, Pokle, and Kolter]{geng2023onestep}
Geng, Z., Pokle, A., and Kolter, J.~Z.
\newblock One-step diffusion distillation via deep equilibrium models.
\newblock In \emph{Neural Information Processing Systems (NeurIPS)}, 2023.

\bibitem[Geng et~al.(2024)Geng, Pokle, Luo, Lin, and Kolter]{geng2024consistency}
Geng, Z., Pokle, A., Luo, W., Lin, J., and Kolter, J.~Z.
\newblock Consistency models made easy.
\newblock \emph{arXiv preprint arXiv:2406.14548}, 2024.

\bibitem[Geng et~al.(2025{\natexlab{a}})Geng, Deng, Bai, Kolter, and He]{geng2025mean}
Geng, Z., Deng, M., Bai, X., Kolter, J.~Z., and He, K.
\newblock Mean flows for one-step generative modeling.
\newblock In \emph{Neural Information Processing Systems (NeurIPS)}, 2025{\natexlab{a}}.

\bibitem[Geng et~al.(2025{\natexlab{b}})Geng, Lu, Wu, Shechtman, Kolter, and He]{geng2025improved}
Geng, Z., Lu, Y., Wu, Z., Shechtman, E., Kolter, J.~Z., and He, K.
\newblock Improved mean flows: On the challenges of fastforward generative models.
\newblock \emph{arXiv preprint arXiv:2512.02012}, 2025{\natexlab{b}}.

\bibitem[Geng et~al.(2025{\natexlab{c}})Geng, Pokle, Luo, Lin, and Kolter]{geng2025consistency}
Geng, Z., Pokle, A., Luo, W., Lin, J., and Kolter, J.~Z.
\newblock Consistency models made easy.
\newblock In \emph{International Conference on Learning Representations (ICLR)}, 2025{\natexlab{c}}.

\bibitem[Guo et~al.(2025)Guo, Wang, Yuan, Cao, Chen, Chen, Huo, Zhang, Wang, Liu, et~al.]{guo2025splitmeanflow}
Guo, Y., Wang, W., Yuan, Z., Cao, R., Chen, K., Chen, Z., Huo, Y., Zhang, Y., Wang, Y., Liu, S., et~al.
\newblock Splitmeanflow: Interval splitting consistency in few-step generative modeling.
\newblock \emph{arXiv preprint arXiv:2507.16884}, 2025.

\bibitem[Hairer et~al.(1993)Hairer, N{\o}rsett, and Wanner]{hairer1993solving}
Hairer, E., N{\o}rsett, S.~P., and Wanner, G.
\newblock \emph{Solving Ordinary Differential Equations I: Nonstiff Problems}.
\newblock Springer-Verlag Berlin Heidelberg, 1993.

\bibitem[Heusel et~al.(2017)Heusel, Ramsauer, Unterthiner, Nessler, and Hochreiter]{heusel2017gans}
Heusel, Z., Ramsauer, H., Unterthiner, T., Nessler, B., and Hochreiter, S.
\newblock Gans trained by a two time-scale update rule converge to a local nash equilibriumrium models.
\newblock In \emph{Neural Information Processing Systems (NeurIPS)}, 2017.

\bibitem[Ho et~al.(2020)Ho, Jain, and Abbeel]{ho2020denoising}
Ho, J., Jain, A., and Abbeel, P.
\newblock Denoising diffusion probabilistic models.
\newblock In \emph{Neural Information Processing Systems (NeurIPS)}, 2020.

\bibitem[Ho et~al.(2022{\natexlab{a}})Ho, Chan, Saharia, Whang, Gao, Gritsenko, Kingma, Poole, Norouzi, Fleet, et~al.]{ho2022imagen}
Ho, J., Chan, W., Saharia, C., Whang, J., Gao, R., Gritsenko, A., Kingma, D.~P., Poole, B., Norouzi, M., Fleet, D.~J., et~al.
\newblock Imagen video: High definition video generation with diffusion models.
\newblock \emph{arXiv preprint arXiv:2210.02303}, 2022{\natexlab{a}}.

\bibitem[Ho et~al.(2022{\natexlab{b}})Ho, Salimans, Gritsenko, Chan, Norouzi, and Fleet]{ho2022video}
Ho, J., Salimans, T., Gritsenko, A., Chan, W., Norouzi, M., and Fleet, D.~J.
\newblock Video diffusion models.
\newblock \emph{Advances in neural information processing systems}, 35:\penalty0 8633--8646, 2022{\natexlab{b}}.

\bibitem[Hui et~al.(2025)Hui, Liu, Zeng, Fu, and Vahdat]{hui2025not}
Hui, K.-H., Liu, C., Zeng, X., Fu, C.-W., and Vahdat, A.
\newblock Not-so-optimal transport flows for 3d point cloud generation.
\newblock \emph{arXiv preprint arXiv:2502.12456}, 2025.

\bibitem[Karras et~al.(2019)Karras, Laine, and Aila]{karras2019style}
Karras, T., Laine, S., and Aila, T.
\newblock A style-based generator architecture for generative adversarial networks.
\newblock In \emph{IEEE Conference on Computer Vision and Pattern Recognition (CVPR)}, 2019.

\bibitem[Kim et~al.(2020)Kim, Lee, Kang, Lee, and Kim]{kim2020softflow}
Kim, H., Lee, H., Kang, W.~H., Lee, J.~Y., and Kim, N.~S.
\newblock Softflow: Probabilistic framework for normalizing flow on manifolds.
\newblock \emph{Advances in Neural Information Processing Systems}, 33:\penalty0 16388--16397, 2020.

\bibitem[Kim et~al.(2021)Kim, Yoo, Lee, and Hong]{kim2021setvae}
Kim, J., Yoo, J., Lee, J., and Hong, S.
\newblock Setvae: Learning hierarchical composition for generative modeling of set-structured data.
\newblock In \emph{Proceedings of the IEEE/CVF Conference on Computer Vision and Pattern Recognition}, pp.\  15059--15068, 2021.

\bibitem[Kingma \& Welling(2022)Kingma and Welling]{kingma2022auto}
Kingma, D.~P. and Welling, M.
\newblock Auto-encoding variational bayes, 2022.
\newblock URL \url{https://arxiv.org/abs/1312.6114}.

\bibitem[Klokov et~al.(2020)Klokov, Boyer, and Verbeek]{klokov2020discrete}
Klokov, R., Boyer, E., and Verbeek, J.
\newblock Discrete point flow networks for efficient point cloud generation.
\newblock In \emph{European Conference on Computer Vision}, pp.\  694--710. Springer, 2020.

\bibitem[Kynk{\"a}{\"a}nniemi et~al.(2023)Kynk{\"a}{\"a}nniemi, Karras, Aittala, Aila, and Lehtinen]{kynkaanniemi2023role}
Kynk{\"a}{\"a}nniemi, T., Karras, T., Aittala, M., Aila, T., and Lehtinen, J.
\newblock The role of imagenet classes in fr$\backslash$'echet inception distance.
\newblock In \emph{International Conference on Learning Representations (ICLR)}, 2023.

\bibitem[Li \& He(2025)Li and He]{li2025back}
Li, T. and He, K.
\newblock Back to basics: Let denoising generative models denoise.
\newblock \emph{arXiv preprint arXiv:2511.13720}, 2025.

\bibitem[Li et~al.(2025)Li, Sun, Turk, and Zhu]{li2025functional}
Li, Z., Sun, Y., Turk, G., and Zhu, B.
\newblock Functional mean flow in hilbert space.
\newblock \emph{arXiv preprint arXiv:2511.12898}, 2025.

\bibitem[Lipman et~al.(2023)Lipman, Chen, Ben-Hamu, Nickel, and Le]{lipman2023flow}
Lipman, Y., Chen, R. T.~Q., Ben-Hamu, H., Nickel, M., and Le, M.
\newblock Flow matching for generative modeling.
\newblock In \emph{International Conference on Learning Representations (ICLR)}, 2023.

\bibitem[Lipman et~al.(2024)Lipman, Havasi, Holderrieth, Shaul, Le, Karrer, Chen, Lopez-Paz, Ben-Hamu, and Gat]{lipman2024flow}
Lipman, Y., Havasi, M., Holderrieth, P., Shaul, N., Le, M., Karrer, B., Chen, R.~T., Lopez-Paz, D., Ben-Hamu, H., and Gat, I.
\newblock Flow matching guide and code.
\newblock \emph{arXiv preprint arXiv:2412.06264}, 2024.

\bibitem[Liu et~al.(2023)Liu, Gong, and Liu]{liu2023flow}
Liu, X., Gong, C., and Liu, Q.
\newblock Flow straight and fast: Learning to generate and transfer data with rectified flow.
\newblock In \emph{International Conference on Learning Representations (ICLR)}, 2023.

\bibitem[Liu et~al.(2015)Liu, Luo, Wang, and Tang]{liu2015faceattributes}
Liu, Z., Luo, P., Wang, X., and Tang, X.
\newblock Deep learning face attributes in the wild.
\newblock In \emph{Proceedings of International Conference on Computer Vision (ICCV)}, 2015.

\bibitem[Liu et~al.(2019)Liu, Tang, Lin, and Han]{liu2019point}
Liu, Z., Tang, H., Lin, Y., and Han, S.
\newblock Point-voxel cnn for efficient 3d deep learning.
\newblock \emph{Advances in neural information processing systems}, 2019.

\bibitem[Lu \& Song(2025)Lu and Song]{lu2025simplifying}
Lu, C. and Song, Y.
\newblock Simplifying, stabilizing and scaling continuous-time consistency models.
\newblock In \emph{International Conference on Learning Representations (ICLR)}, 2025.

\bibitem[Luo \& Hu(2021)Luo and Hu]{luo2021diffusion}
Luo, S. and Hu, W.
\newblock Diffusion probabilistic models for 3d point cloud generation.
\newblock In \emph{Proceedings of the IEEE/CVF conference on computer vision and pattern recognition}, pp.\  2837--2845, 2021.

\bibitem[Meng et~al.(2023)Meng, Rombach, Gao, Kingma, Ermon, Ho, and Salimans]{meng2023distillation}
Meng, C., Rombach, R., Gao, R., Kingma, D., Ermon, S., Ho, J., and Salimans, T.
\newblock On distillation of guided diffusion models.
\newblock In \emph{IEEE Conference on Computer Vision and Pattern Recognition (CVPR)}, 2023.

\bibitem[Mo et~al.(2023)Mo, Xie, Chu, Hong, Niessner, and Li]{mo2023dit}
Mo, S., Xie, E., Chu, R., Hong, L., Niessner, M., and Li, Z.
\newblock Dit-3d: Exploring plain diffusion transformers for 3d shape generation.
\newblock \emph{Advances in neural information processing systems}, 36:\penalty0 67960--67971, 2023.

\bibitem[Molodyk et~al.(2025)Molodyk, Choi, Romero, Liu, and Chen]{molodyk2025mfm}
Molodyk, P., Choi, J., Romero, D.~W., Liu, M.-Y., and Chen, Y.
\newblock Mfm-point: Multi-scale flow matching for point cloud generation.
\newblock \emph{arXiv preprint arXiv:2511.20041}, 2025.

\bibitem[Peebles \& Xie(2023)Peebles and Xie]{peebles2023scalable}
Peebles, W. and Xie, S.
\newblock Scalable diffusion models with transformers.
\newblock In \emph{IEEE Conference on Computer Vision and Pattern Recognition (CVPR)}, 2023.

\bibitem[Rombach et~al.(2022{\natexlab{a}})Rombach, Blattmann, Lorenz, Esser, and Ommer]{rombach2021highresolution}
Rombach, R., Blattmann, A., Lorenz, D., Esser, P., and Ommer, B.
\newblock High-resolution image synthesis with latent diffusion models.
\newblock In \emph{IEEE Conference on Computer Vision and Pattern Recognition (CVPR)}, 2022{\natexlab{a}}.

\bibitem[Rombach et~al.(2022{\natexlab{b}})Rombach, Blattmann, Lorenz, Esser, and Ommer]{rombach2022high}
Rombach, R., Blattmann, A., Lorenz, D., Esser, P., and Ommer, B.
\newblock High-resolution image synthesis with latent diffusion models.
\newblock In \emph{Proceedings of the IEEE/CVF conference on computer vision and pattern recognition}, pp.\  10684--10695, 2022{\natexlab{b}}.

\bibitem[Salimans \& Ho(2022)Salimans and Ho]{salimans2022progressive}
Salimans, T. and Ho, J.
\newblock Progressive distillation for fast sampling of diffusion models.
\newblock In \emph{International Conference on Learning Representations (ICLR)}, 2022.

\bibitem[Song et~al.(2021{\natexlab{a}})Song, Meng, and Ermon]{song2021denoising}
Song, J., Meng, C., and Ermon, S.
\newblock Denoising diffusion implicit models.
\newblock In \emph{International Conference on Learning Representations (ICLR)}, 2021{\natexlab{a}}.

\bibitem[Song \& Dhariwal(2023)Song and Dhariwal]{song2023improved}
Song, Y. and Dhariwal, P.
\newblock Improved techniques for training consistency models.
\newblock In \emph{International Conference on Learning Representations (ICLR)}, 2023.

\bibitem[Song \& Ermon(2019)Song and Ermon]{song2019generative}
Song, Y. and Ermon, S.
\newblock Generative modeling by estimating gradients of the data distribution.
\newblock In \emph{Neural Information Processing Systems (NeurIPS)}, 2019.

\bibitem[Song et~al.(2021{\natexlab{b}})Song, Sohl-Dickstein, Kingma, Kumar, Ermon, and Poole]{song2021score}
Song, Y., Sohl-Dickstein, J., Kingma, D.~P., Kumar, A., Ermon, S., and Poole, B.
\newblock Score-based generative modeling through stochastic differential equations.
\newblock In \emph{International Conference on Learning Representations (ICLR)}, 2021{\natexlab{b}}.

\bibitem[Song et~al.(2023)Song, Dhariwal, Chen, and Sutskever]{song2023consistency}
Song, Y., Dhariwal, P., Chen, M., and Sutskever, I.
\newblock Consistency models.
\newblock In \emph{International Conference on Machine Learning (ICML)}, 2023.

\bibitem[Vahdat et~al.(2022)Vahdat, Williams, Gojcic, Litany, Fidler, Kreis, et~al.]{vahdat2022lion}
Vahdat, A., Williams, F., Gojcic, Z., Litany, O., Fidler, S., Kreis, K., et~al.
\newblock Lion: Latent point diffusion models for 3d shape generation.
\newblock \emph{Advances in Neural Information Processing Systems}, 35:\penalty0 10021--10039, 2022.

\bibitem[Vaswani et~al.(2017)Vaswani, Shazeer, Parmar, Uszkoreit, Jones, Gomez, Kaiser, and Polosukhin]{vaswani2017attention}
Vaswani, A., Shazeer, N., Parmar, N., Uszkoreit, J., Jones, L., Gomez, A.~N., Kaiser, {\L}., and Polosukhin, I.
\newblock Attention is all you need.
\newblock \emph{Advances in neural information processing systems}, 30, 2017.

\bibitem[Wang et~al.(2025)Wang, Lin, Liu, Xu, Dou, Long, Guo, Komura, Wang, and Li]{wang2025pdt}
Wang, J., Lin, C., Liu, Y., Xu, R., Dou, Z., Long, X., Guo, H., Komura, T., Wang, W., and Li, X.
\newblock Pdt: Point distribution transformation with diffusion models.
\newblock In \emph{Proceedings of the Special Interest Group on Computer Graphics and Interactive Techniques Conference Conference Papers}, pp.\  1--11, 2025.

\bibitem[Webb(1985)]{webb1985semi}
Webb, G.~F.
\newblock Semigroups of linear operators and applications to partial differential equations (a. pazy).
\newblock \emph{SIAM Review}, 1985.

\bibitem[Wu et~al.(2023)Wu, Wang, Gong, Liu, Xiong, Ranjan, Krishnamoorthi, Chandra, and Liu]{wu2023fast}
Wu, L., Wang, D., Gong, C., Liu, X., Xiong, Y., Ranjan, R., Krishnamoorthi, R., Chandra, V., and Liu, Q.
\newblock Fast point cloud generation with straight flows.
\newblock In \emph{Proceedings of the IEEE/CVF conference on computer vision and pattern recognition}, pp.\  9445--9454, 2023.

\bibitem[Yang et~al.(2019)Yang, Huang, Hao, Liu, Belongie, and Hariharan]{yang2019pointflow}
Yang, G., Huang, X., Hao, Z., Liu, M.-Y., Belongie, S., and Hariharan, B.
\newblock Pointflow: 3d point cloud generation with continuous normalizing flows.
\newblock In \emph{Proceedings of the IEEE/CVF international conference on computer vision}, pp.\  4541--4550, 2019.

\bibitem[Yang et~al.(2024)Yang, Zhang, Zhang, Liu, Xu, Zhang, Meng, Ermon, and Cui]{yang2024consistency}
Yang, L., Zhang, Z., Zhang, Z., Liu, X., Xu, M., Zhang, W., Meng, C., Ermon, S., and Cui, B.
\newblock Consistency flow matching: Defining straight flows with velocity consistency.
\newblock \emph{arXiv preprint arXiv:2407.02398}, 2024.

\bibitem[Zhang \& Wonka(2024)Zhang and Wonka]{zhang2024functional}
Zhang, B. and Wonka, P.
\newblock Functional diffusion.
\newblock In \emph{Proceedings of the IEEE/CVF Conference on Computer Vision and Pattern Recognition}, pp.\  4723--4732, 2024.

\bibitem[Zhang et~al.(2022)Zhang, Nie{\ss}ner, and Wonka]{zhang20223dilg}
Zhang, B., Nie{\ss}ner, M., and Wonka, P.
\newblock 3dilg: Irregular latent grids for 3d generative modeling.
\newblock \emph{Advances in Neural Information Processing Systems}, 35:\penalty0 21871--21885, 2022.

\bibitem[Zhang et~al.(2023)Zhang, Tang, Niessner, and Wonka]{zhang20233dshape2vecset}
Zhang, B., Tang, J., Niessner, M., and Wonka, P.
\newblock 3dshape2vecset: A 3d shape representation for neural fields and generative diffusion models.
\newblock \emph{ACM Transactions On Graphics (TOG)}, 42\penalty0 (4):\penalty0 1--16, 2023.

\bibitem[Zhang et~al.(2025{\natexlab{a}})Zhang, Ren, and Wonka]{zhang2025geometry}
Zhang, B., Ren, J., and Wonka, P.
\newblock Geometry distributions.
\newblock In \emph{Proceedings of the IEEE/CVF International Conference on Computer Vision}, pp.\  1495--1505, 2025{\natexlab{a}}.

\bibitem[Zhang et~al.(2025{\natexlab{b}})Zhang, Siarohin, Menapace, Vasilkovsky, Tulyakov, Qu, and Skorokhodov]{zhang2025alphaflow}
Zhang, H., Siarohin, A., Menapace, W., Vasilkovsky, M., Tulyakov, S., Qu, Q., and Skorokhodov, I.
\newblock Alphaflow: Understanding and improving meanflow models.
\newblock \emph{arXiv preprint arXiv:2510.20771}, 2025{\natexlab{b}}.

\bibitem[Zhou et~al.(2024)Zhou, Zhong, Hanji, Guo, Fogarty, Sztrajman, Gao, and Oztireli]{zhou2024frepolad}
Zhou, C., Zhong, F., Hanji, P., Guo, Z., Fogarty, K., Sztrajman, A., Gao, H., and Oztireli, C.
\newblock Frepolad: Frequency-rectified point latent diffusion for point cloud generation.
\newblock In \emph{European Conference on Computer Vision}, pp.\  434--453. Springer, 2024.

\bibitem[Zhou et~al.(2021)Zhou, Du, and Wu]{zhou20213d}
Zhou, L., Du, Y., and Wu, J.
\newblock 3d shape generation and completion through point-voxel diffusion.
\newblock In \emph{Proceedings of the IEEE/CVF international conference on computer vision}, pp.\  5826--5835, 2021.

\bibitem[Zhou et~al.(2025)Zhou, Ermon, and Song]{zhou2025inductive}
Zhou, L., Ermon, S., and Song, J.
\newblock Inductive moment matching.
\newblock In \emph{International Conference on Machine Learning (ICML)}, 2025.

\bibitem[Zhuang et~al.(2023)Zhuang, Abnar, Gu, Schwing, Susskind, and Bautista]{zhuang2023diffusion}
Zhuang, P., Abnar, S., Gu, J., Schwing, A., Susskind, J.~M., and Bautista, M.~A.
\newblock Diffusion probabilistic fields.
\newblock In \emph{International Conference on Learning Representations (ICLR)}, 2023.

\end{thebibliography}
\bibliographystyle{arxiv_tem}

\newpage
\appendix
\onecolumn
\section{Design Philosophy of the Euler MeanFlow Loss}\label{sec:philosophy}

Here we provide an overview of the design rationale of Euler MeanFlow, explaining the principles behind the Euler MeanFlow losses.  The loss design of Euler MeanFlow follows the same fundamental logic as Flow Matching: at their core, both aim to learn a direct target.   For example, Flow Matching optimizes $\mathcal{L}^{FM}(\theta)=\mathbb{E}_{t, x \sim p_t}\|u_t^\theta(x) - u_t(x)\|_2^2$,  where $u_t(x)$ is the reference velocity field. Since $u_t(x)$ is not directly accessible from data, Flow Matching introduces a conditional distribution $p_t(x \mid x_1)$ and a conditional velocity $u(x \mid x_1)$ on $x_1 \in \mathcal{D}$, yielding the conditional loss $\mathcal{L}_{c}^{FM}(\theta)=\mathbb{E}_{t, x \sim p_t(\cdot|x_1),x_1\sim p_1}\|u_t^\theta(x) - u_t(x \mid x_1)\|_2^2$, which is shown that $\nabla \mathcal{L}^{FM}_c(\theta) = \nabla \mathcal{L}^{FM}(\theta)$. Because $u(x \mid x_1) = \frac{x - x_1}{1 - t}$, with $x$ sampled from the tractable conditional distribution $p(x \mid x_1)$, is directly computable from dataset $\mathcal{D}$, the conditional loss $\mathcal{L}_{c}^{FM}(\theta)$ can be used to train a model targeting the original objective $\mathcal{L}^{FM}(\theta)$.  

Euler MeanFlow is built on the same principle. Its ideal learning objective is $\mathcal{L}^C(\theta) = \mathbb{E}_{t,s,r, x_t = (1-t)x_0 + tx_1, x_1\sim p_{data}, x_0\sim p_0}\frac{1}{w(t,r)}\|\phi^\theta_{t\to r}(x_t) - \phi^\theta_{s\to r}(\phi^\theta_{t\to s}(x))\|_2^2$ in \autoref{eq:consistency_flow}.  Similar to Flow Matching, this direct objective does not explicitly leverage information from the training data.   A straightforward solution is to impose supervision only at boundary conditions and propagate it outward, as in previous works~\cite{guo2025splitmeanflow,frans2025shortcut}. However, such boundary-based supervision remains sparse and indirect, which is insufficient to constrain long-range dynamics and often leads to unstable training and degraded performance.  Therefore, our goal is to design a training objective that provides dense, data-driven supervision for $u_{t \rightarrow r}^\theta(x)$ while avoiding reliance on boundary constraints.

The central difficulty is that, unlike instantaneous velocity fields, long-range velocity fields do not admit a natural conditional form (\autoref{thm:non_exist}), making it unclear how to incorporate dataset supervision.  To overcome this challenge, we propose a two-step strategy.
\begin{enumerate}
    \item First, we observe that $\mathcal{L}^{C}(\theta)$ involves three time segments $t\to s$, $s\to r$ and $t\to r$. We select one segment $t\to s$ to be sufficiently short and apply a local linear approximation (\autoref{thm:local_approximation}) on this interval. This transforms part of the long-range transport into an instantaneous velocity field, which admits a well-defined conditional counterpart $u_{t\to t}(x|x_1)$. As a result, we obtain an intermediate surrogate objective $\mathcal{L}^{\tilde{C}}(\theta)$ in \autoref{eq:define_approx_loss} that partially connects long-range dynamics with locally defined velocities.
    \item Second, since $\mathcal{L}^{\tilde{C}}(\theta)$ now involves instantaneous velocity fields, we can follow the Flow Matching framework and replace them with conditional instantaneous velocity fields. This step injects explicit dataset supervision into the objective and yields the final loss $\mathcal{L}^{EMF}(\theta)$.
\end{enumerate}
In \autoref{lemma_for_u} and \autoref{thm:validity}, we theoretically justify this construction by showing that $\nabla_\theta \mathcal{L}^{\tilde{C}}(\theta)\approx \nabla_\theta \mathcal{L}^{EMF}(\theta)$ and $\nabla_\theta \mathcal{L}^{\tilde{C}}(\theta)\approx \nabla_\theta \mathcal{L}^{C}(\theta)$.  These results indicate that optimizing $\mathcal{L}^{EMF}$ provides a faithful approximation to the ideal objective $\mathcal{L}^{C}$, while simultaneously incorporating explicit dataset supervision for learning long-range dynamics.

The $x_1$-prediction variant follows the same strategy. It is worth noting that, although two time variables $t$ and $r$ are involved, only the quantities at time $t$ are generated through sampling. Consequently, only variables at time $t$ can be naturally conditioned on observed data.


\section{Missing Proofs and Derivations}
\subsection{Proof of \autoref{thm:non_exist}}\label{thm:non_exist_proof}
\paragraph{\autoref{thm:non_exist}}(Non-existence of conditional flow maps) There exists no conditional flow maps $\phi_{t\to r}(x | x_{t_1})$ that simultaneously (i) is consistent with the conditional velocity $u(x|x_1)$ under \autoref{eq:evolv_flow_map}, and (ii) satisfies the consistency relation $\phi_{t\to r}(x)=\mathbb{E}_{x_{1}\sim p_t(x_1|x)}[\phi_{t\to r}(x|x_1)]$ with marginal flow maps. As a result, a self-consistent conditional cumulative field does not exist. 
\begin{proof}
First, we denote the mappings $\phi_{t\to r}(x)$ obtained from (1) and (2) as $\phi^{(1)}_{t\to r}(x)$ and $\phi^{(2)}_{t\to r}(x)$, respectively. Specifically, $\phi^{(1)}_{t\to r}(x)=\phi_r(\phi_t^{-1}(x))$, and $\phi^{(2)}_{t\to r}(x)=\mathbb{E}_{x_1\sim p_t(x_1|x)}[\phi_{t\to r}(x|x_1)]$. It suffices to show that $\phi^{(1)}_{t\to r}(x)\neq \phi^{(2)}_{t\to r}(x)$. To this end, it is sufficient to prove that $\frac{d}{dt}\phi^{(1)}_{t\to r}(x)\neq \frac{d}{dt}\phi^{(2)}_{t\to r}(x)$ at $t=0$.  

\begin{equation}
    \begin{aligned}
        \frac{d}{dr}\phi^{(2)}_{t\to r}(x)&= \frac{d}{dr}\int_{x_1} \phi_{t\to r}(x|x_1)p_t(x_1|x)dx_1\\
        &=\int_{x_1} \frac{d}{dr}\phi_{t\to r}(x|x_1)p_t(x_1|x)dx_1\\
        &=\int_{x_1} u_{r}(\phi_{t\to r}(x|x_1)|x_1)\frac{p_{data}(x_1)p_t(x|x_1)}{p_t(x)}dx_1\\
        &=\int_{x_1} u_{r}(\phi_{t\to r}(x|x_1)|x_1)p_{data}(x_1)dx_1\\    
        &=\int_{x_1}(x_1-x)p_{data}(x_1)dx_1\\    
        &= \mathbb{E}_{x_1\sim p_{data}(x_1)}[x_1] - x
    \end{aligned}
\end{equation}
Consequently, if $\phi^{(1)}_{t\to r}(x) = \phi^{(2)}_{t\to r}(x)$ we must have $\frac{d}{dr}\phi_r(x) = \mathbb{E}_{x_1\sim p_{data}(x_1)}[x_1] - x$, $\phi_r(x) =(\mathbb{E}_{x_1\sim p_{data}(x_1)}[x_1]-x) r + x$, which implies $p_{data} = (\phi_1)_\sharp p_0 = \delta_{\mathbb{E}_{x_1\sim p_{data}(x_1)}[x_1]}$, where $\delta$ denotes the Dirac distribution at a single point.
\end{proof}

\subsection{Proof of \autoref{lemma_for_u}}\label{sec:proof_lemma_for_u}
\paragraph{\autoref{lemma_for_u}} With $M_g = \sqrt{\mathbb{E}_{t,r,x\sim p_t(x)}[\frac{1}{m}\|\nabla_\theta u^\theta_{t\to r}(x)\|_2^2]}<+\infty$ holds in \autoref{assupmption}, our Euler Mean Flow loss $\mathcal{L}^{E}(\theta)$ and the approximated trajectory consistency loss $\mathcal{L}^{\tilde C}(\theta)$ satisfy
 \begin{equation}
     \begin{aligned}
         \text{MSE}(\nabla  \mathcal{L}^{E}(\theta),\nabla  \mathcal{L}^{ \tilde C}(\theta)) \le M_g \sqrt{\mathbb{E}_{t,r, x \sim  p_t(x)} [\|u_{t\to t}^\theta(x)-u_t(x)\|^2]}
     \end{aligned}
 \end{equation}
where $\mathrm{MSE}$ denotes the mean squared error. Consequently, during training, if $\|u^\theta_{t\to t}(x) - u_t(x)\|^2 \to 0$, then $\mathcal{L}^{E}(\theta)$ and $\mathcal{L}^{\tilde C}(\theta)$ share the same optimal target at $\theta$.  The term $u_t(x)$ denotes the reference velocity at $x$, defined as $u_t(x) = \mathbb{E}_{x_1 \sim p(x_1|x)}\big[u_t(x|x_1)\big]$,  which is intractable to compute analytically.

Here, the approximated trajectory consistency loss $\mathcal{L}^{\tilde C}(\theta)$ are defined as 
\begin{equation}
    \begin{aligned}
  \mathcal{L}^{\tilde C}(\theta) &= \mathbb{E}_{t,r,x \sim  p_t(x), x'=sg(\Delta t u^\theta_{t\to t}(x))+x}\\
  &[|u_{t\to r}^\theta(x)  - (u^\theta_{t\to t}(x) + (r-t-\Delta t) \\
 &\text{sg}(\frac{u^\theta_{t+\Delta t\to r}(x')-u^\theta_{t\to r}(x)}{\Delta t}))\|^2]       
    \end{aligned}
\end{equation}
It is straightforward to verify that the loss $\mathcal{L}^{\tilde C}(\theta)$ is the mean-velocity formulation of $\mathcal{L}^{C}(\theta)$ under the local linear approximation in \autoref{eq:EFM_equation}, expressed via $u_{t\to r}(x) = \frac{\phi_{t\to r}(x)-x}{r-t}$, and differs by a temporal scaling factor $1/\Delta t$.
\begin{proof}

We first define the reference regression loss $\mathcal{L}^{R}(\theta)$ as
\begin{equation} 
\begin{aligned} 
    \mathcal{L}^{R}(\theta) = 
    &\mathbb{E}_{t,r,x \sim p_t(x), x'=sg(\Delta t u^\theta_{t\to t}(x))+x}\\ 
    &[|u_{t\to r}^\theta(x) - (u_t(x) + (r-t-\Delta t) \\ 
    &\text{sg}(\frac{u^\theta_{t+\Delta t\to r}(x')-u^\theta_{t\to r}(x)}{\Delta t}))\|^2] \\ 
\end{aligned} 
\end{equation}

Let $B(\theta;r,t,x) = (r-t-\Delta t)sg(\frac{u_{t+\Delta t\to r}^\theta(x')-u_{t\to r}^\theta(x)}{\Delta t})$. Since $B(\theta; r,t,x)$ contains the stop-gradient operator $\mathrm{sg}(\cdot)$, it satisfies $\nabla_\theta B(\theta;r,t,x) = 0$.  Using $B(\theta; r,t,x)$, the Euler Mean Flow loss $\mathcal{L}^E(\theta)$, the reference regression loss $\mathcal{L}^R(\theta)$, and the approximated trajectory consistency loss $\mathcal{L}^{\tilde C}(\theta)$ can be written as
\begin{equation} 
\begin{aligned} \mathcal{L}^E(\theta) &= \mathbb{E}_{t,r,x_1\sim p_{data}, x \sim  p_t(x|x_1)}[\|u_{t\to r}^\theta(x) -(u_t(x|x_1)+B(\theta;r,t,x))\|^2]\\ 
\mathcal{L}^R(\theta) &= \mathbb{E}_{t,r,x \sim p_t(x)}[\|u_{t\to r}^\theta(x) -(u_t(x)+B(\theta;r,t,x))\|^2]\\ 
\mathcal{L}^{\tilde C}(\theta) &= \mathbb{E}_{t,r,x \sim  p_t(x)}[\|u_{t\to r}^\theta(x) -(u_{t\to t}^\theta(x)+B(\theta;r,t,x))\|^2] 
\end{aligned} 
\end{equation}

We first show that the Euler Mean Flow loss $\mathcal{L}^{E}(\theta)$ and the reference regression loss $\mathcal{L}^{R}(\theta)$ satisfy $\nabla_\theta \mathcal{L}^{E}(\theta)$ = $\nabla_\theta \mathcal{L}^{R}(\theta)$. Expanding $\nabla \mathcal{L}^{E}(\theta)$, we obtain
\begin{equation} 
\begin{aligned} 
\nabla_\theta  \mathcal{L}^{E}(\theta) &= \mathbb{E}_{t,r,x_1\sim p_{data}, x \sim  p_t(x|x_1), }[\nabla \|u_{t\to r}^\theta(x) -(u_t(x|x_1)+B(\theta;r,t,x))\|^2] \\
&\overset{\nabla B(\theta;r,t,x) = 0}{=} \mathbb{E}_{t,r,x_1\sim p_{data}, x \sim  p_t(x|x_1)}[\nabla u_{t\to r}^\theta(x) (u_{t\to r}^\theta(x)-u_t(x|x_1)-B(\theta;r,t,x))]  \\
&\overset{\nabla B(\theta;r,t,x) = 0}{=} \mathbb{E}_{t,r,x_1\sim p_{data}, x \sim  p_t(x|x_1)}[\nabla u_{t\to r}^\theta(x) (u_{t\to r}^\theta(x)-u_t(x|x_1)-B(\theta;r,t,x))]  \\
\end{aligned} \end{equation}
where $\mathbb{E}_{t,r,x_1\sim p_{data}, x \sim  p_t(x|x_1)}[\nabla u_{t\to r}^\theta(x) (u_{t\to r}^\theta(x)-B(\theta;r,t,x))]$ can be computed as:
\begin{equation}
    \begin{aligned}
        &\mathbb{E}_{t,r,x_1\sim p_{data}, x \sim  p_t(x|x_1)}[\nabla u_{t\to r}^\theta(x) (u_{t\to r}^\theta(x)-B(\theta;r,t,x))]\\
        &= \int_{t,r} \int_{x_1}\int_x (\nabla u_{t\to r}^\theta(x) (u_{t\to r}^\theta(x)-B(\theta;r,t,x))) p(x|x_1)p_{data}(x_1) p(t,r)dx dx_1dt dr\\
        &= \int_{t,r} \int_x (\nabla u_{t\to r}^\theta(x) (u_{t\to r}^\theta(x)-B(\theta;r,t,x))) \int_{x_1}p(x|x_1)p_{data}(x_1)dx_1 p(t,r)dx dt dr\\
        &= \int_{t,r} \int_x (\nabla u_{t\to r}^\theta(x) (u_{t\to r}^\theta(x)-B(\theta;r,t,x))) p(x) p(t,r)dx dt dr\\
        &= \mathbb{E}_{t,r,x \sim  p_t(x)}[\nabla u_{t\to r}^\theta(x) (u_{t\to r}^\theta(x)-B(\theta;r,t,x))]
    \end{aligned}
\end{equation}
And $\mathbb{E}_{t,r,x_1\sim p_{data}, x \sim  p_t(x|x_1)}[\nabla u_{t\to r}^\theta(x) u_t(x|x_1)]$ can be calculated as
\begin{equation}
    \begin{aligned}
        &\mathbb{E}_{t,r,x_1\sim p_{data}, x \sim  p_t(x|x_1)}[\nabla u_{t\to r}^\theta(x) u_t(x|x_1)]\\
        &= \int_{t,r} \int_{x_1}\int_x \nabla u_{t\to r}^\theta(x) u_t(x|x_1) p(x|x_1)p_{data}(x_1) p(t,r)dx dx_1dt dr\\
        &= \int_{t,r} \int_x \nabla u_{t\to r}^\theta(x) (\int_{x_1}u_t(x|x_1) p(x|x_1) p_{data}(x_1) dx_1) p(t,r)dx dt dr\\
        &= \int_{t,r} \int_x \nabla u_{t\to r}^\theta(x) (\int_{x_1}u_t(x|x_1) p(x_1|x) p(x) dx_1) p(t,r)dx dt dr\\
        &= \int_{t,r} \int_x \nabla u_{t\to r}^\theta(x) p(x)u_t(x) p(t,r)dx dt dr\\
        &=  \mathbb{E}_{t,r,x \sim  p_t(x)}[\nabla u_{t\to r}^\theta(x) u_t(x)]
    \end{aligned}
\end{equation}
Therefore, we have 
\begin{equation} \label{eq:relation1}
\begin{aligned} 
\nabla_\theta  \mathcal{L}^{E}(\theta) &= \mathbb{E}_{t,r,x_1\sim p_{data}, x \sim  p_t(x|x_1)}[\nabla u_{t\to r}^\theta(x) (u_{t\to r}^\theta(x)-u_t(x|x_1)-B(\theta;r,t,x))]  \\
&=\mathbb{E}_{t,r, x \sim  p_t(x)}[\nabla u_{t\to r}^\theta(x) (u_{t\to r}^\theta(x)-u_t(x)-B(\theta;r,t,x))]  \\
&= \nabla_\theta  \mathcal{L}^{R}(\theta)
\end{aligned} \end{equation}

We then calculate the difference between $\nabla_\theta  \mathcal{L}^{R}(\theta)$ and $\nabla_\theta  \mathcal{L}^{\tilde C}(\theta)$ as
\begin{equation}
    \begin{aligned}
        \nabla_\theta  \mathcal{L}^{R}(\theta) - \nabla_\theta  \mathcal{L}^{ \tilde C}(\theta) &= \mathbb{E}_{t,r,x \sim  p_t(x)}[\nabla_\theta u_{t\to r}^\theta(x) (u_{t\to r}^\theta(x)-u_t(x)-B(\theta;r,t,x))\\
        &-\nabla_\theta u_{t\to r}^\theta(x) (u_{t\to r}^\theta(x)-u_{t\to t}^\theta(x)-B(\theta;r,t,x))\|^2]\\
        &=\mathbb{E}_{t,r, x \sim  p_t(x)} [\nabla_\theta u_{t\to r}^\theta(x)(u_{t\to t}^\theta(x)-u_t(x))]        
    \end{aligned}
\end{equation}
Applying the Cauchy-Schwarz inequality and using the assumption $M_g = \sqrt{\mathbb{E}_{t,r,x\sim p_t(x)}[\frac{1}{m}\|\nabla_\theta u^\theta_{t\to r}(x)\|_2^2]}<+\infty$ in \autoref{assupmption}, we further obtain the following bound:
\begin{equation}\label{eq:relation2}
    \begin{aligned}
        MSE(\nabla_\theta  \mathcal{L}^{R}(\theta) , \nabla_\theta  \mathcal{L}^{ \tilde C}(\theta)) &=  \frac{1}{\sqrt{m}}\|\mathbb{E}_{t,r, x \sim  p_t(x)} [\nabla_\theta u_{t\to r}^\theta(x)(u_{t\to t}^\theta(x)-u_t(x))]\|\\
        &\le \frac{1}{\sqrt{m}}\mathbb{E}_{t,r,x \sim  p_t(x)} [\|\nabla_\theta u_{t\to r}^\theta(x)(u_{t\to t}^\theta(x)-u_t(x))\|]\\
        &\le \frac{1}{\sqrt{m}}\mathbb{E}_{t,r, x \sim  p_t(x)} [\|\nabla_\theta u_{t\to r}^\theta(x)\|_2\|u_{t\to t}^\theta(x)-u_t(x)\|]\\
        &\le \sqrt{\mathbb{E}_{t,r, x \sim  p_t(x)} [\frac{1}{m}\|\nabla_\theta u_{t\to r}^\theta(x)\|_2^2]\mathbb{E}_{t,r, x \sim  p_t(x)}[\|u_{t\to t}^\theta(x)-u_t(x)\|^2]}\\
        &\le M_g \sqrt{\mathbb{E}_{t,r, x \sim  p_t(x)} [\|u_{t\to t}^\theta(x)-u_t(x)\|^2]}\\
    \end{aligned}
\end{equation}
Combine \autoref{eq:relation1} and \autoref{eq:relation2}, we have
\begin{equation}
    \begin{aligned}
        MSE(\nabla  \mathcal{L}^{E}(\theta), \nabla  \mathcal{L}^{ \tilde C}(\theta)) \le M_g \sqrt{\mathbb{E}_{t,r, x \sim  p_t(x)} [\|u_{t\to t}^\theta(x)-u_t(x)\|^2]}
    \end{aligned}
\end{equation}
\end{proof}

\subsection{Proof of \autoref{thm:validity}}\label{sec:u_validity_proof}
\paragraph{\autoref{thm:validity}}(Surrogate Loss Validity) With $M_g = \sqrt{\mathbb{E}_{t,r,x \sim  p_t(x) }[\frac{1}{m}\|\nabla_\theta u_{t\to r}^\theta(x)\|_2^2]} < +\infty$, $M_x = \sqrt{\mathbb{E}_{t,r,x \sim  p_t(x) }[\frac{1}{m}\|\partial_x u_{t+\Delta t\to r}^\theta(x')\|_2^2]} < +\infty$, and $M_t = \sqrt{\mathbb{E}_{t,r,x \sim  p_t(x) }[\|\partial_s u_{t\to s}^\theta|_{s=t}\|^2]} < +\infty$ hold in \autoref{assupmption}, Our Euler Mean Flow loss $\mathcal{L}^{E}(\theta)$ and the trajectory consistency loss $\mathcal{L}^{C}(\theta)$ satisfy
\begin{equation}
    \begin{aligned}
        MSE(\nabla_\theta L^E(\theta),\nabla_\theta L^C(\theta)) \le  M_g \sqrt{\mathbb{E}_{t,r, x \sim  p_t(x)} [\|u_{t\to t}^\theta(x)-u_t(x)\|^2]} + (M_gM_t + M_x M_t)\Delta t + O(\Delta t^2)
    \end{aligned}
\end{equation}
\begin{proof}
We define $C(\theta;r,t,x) = sg(\frac{u_{t+\Delta t\to r}^\theta(x')-u_{t\to r}^\theta(x)}{\Delta t})$, $x' = x+ u_{t\to t}^\theta(x)\Delta t$ and $D(\theta;r,t,x) =sg(\frac{u_{t+\Delta t\to r}^\theta(x'')-u_{t\to r}^\theta(x)}{\Delta t})$, $x'' = x+ u_{t\to t+\Delta t}^\theta(x)\Delta t$.  With these definitions, the approximated trajectory consistency loss
$\mathcal{L}^{\tilde C}(\theta)$ and the trajectory consistency loss $\mathcal{L}^{C}(\theta)$ can be written as
\begin{equation}
    \begin{aligned}
\mathcal{L}^{\tilde C}(\theta) &= \mathbb{E}_{t,r,x \sim  p_t(x)}[\|u_{t\to r}^\theta(x) -sg(u_{t\to t}^\theta(x)+(r-t-\Delta t)C(\theta;r,t,x))\|^2] \\
\mathcal{L}^{C}(\theta) &= \mathbb{E}_{t,r,x \sim  p_t(x)}[\|u_{t\to r}^\theta(x) -sg(u_{t\to t+\Delta t}^\theta(x)+(r-t-\Delta t)D(\theta;r,t,x))\|^2] 
    \end{aligned}
\end{equation}

We now analyze the difference $\nabla_\theta \mathcal{L}^{\tilde C}(\theta)  - \nabla_\theta \mathcal{L}^{C}(\theta)  $ between the gradients of these two objectives.  A direct computation yields
\begin{equation}
\begin{aligned}
\nabla \mathcal{L}_\theta^{\tilde C}(\theta)  - \nabla \mathcal{L}_\theta^{C}(\theta)  &=  \mathbb{E}_{t,r,x \sim  p_t(x)}[\nabla_\theta u_{t\to r}^\theta(x)(u_{t\to r}^\theta(x) -(u_t^\theta(x)+(r-t-\Delta t)C(\theta;r,t,x))]\\
&-\mathbb{E}_{t,r,x \sim  p_t(x)}[\nabla_\theta u_{t\to r}^\theta(x)(u_{t\to r}^\theta(x) -(u_{t\to t+\Delta t}^\theta(x)+(r-t-\Delta t)D(\theta;r,t,x))]\\
&= \mathbb{E}_{t,r,x \sim  p_t(x)}[\nabla_\theta u_{t\to r}^\theta(x)((u_{t\to t+\Delta t}^\theta(x)-u_{t\to t}^\theta(x))+(r-t-\Delta t)(D(\theta;r,t,x)-C(\theta;r,t,x)))]\\
\end{aligned}
\end{equation}

We first bound the difference $D(\theta; r,t,x)-C(\theta; r,t,x)$. By definition,
\begin{equation}
\begin{aligned}
\|D(\theta;r,t,x)- C(\theta;r,t,x)\|  &=\|\frac{u_{t+\Delta t\to r}^\theta(x')-u_{t\to r}^\theta(x)}{\Delta t}- \frac{u_{t+\Delta t\to r}^\theta(x'')-u_{t\to r}^\theta(x)}{\Delta t}\|\\
&= \frac{1}{\Delta t}\|u_{t+\Delta t\to r}^\theta(x') - u_{t+\Delta t\to r}^\theta(x'')\|\\
&= \frac{1}{\Delta t}\|\partial_x u_{t+\Delta t\to r}^\theta(x')(x''-x')\|\\
&= \frac{1}{\Delta t}\|\partial_x u_{t+\Delta t\to r}^\theta(x')(u^\theta_{t\to t+\Delta t}(x)-u^\theta_{t\to t}(x)) \Delta t\|\\
&= \|\partial_x u_{t+\Delta t\to r}^\theta(x')(u^\theta_{t\to t+\Delta t}(x)-u^\theta_{t\to t}(x))\|\\
&\le \|\partial_x u_{t+\Delta t\to r}^\theta(x')\|_2\|u^\theta_{t\to t+\Delta t}(x)-u^\theta_{t\to t}(x)\|
\end{aligned}
\end{equation}

Next, the difference $u_{t\to t+\Delta t}^\theta(x) -u_{t\to t}^\theta(x)$  admits a first-order expansion:
\begin{equation}
    \begin{aligned}
\|u_{t\to t+\Delta t}^\theta(x) -u_{t\to t}^\theta(x) \|  &= \|u_{t\to t}^\theta(x)  + \Delta t\partial_su^\theta_{t\to s}|_{s=t} + O(\Delta t^2)-u_{t\to t}^\theta(x) \|\\
&\le \Delta t \|\partial_su_{t\to s}^\theta|_{s= t}\|+O(\Delta t^2)\\
\end{aligned}
\end{equation}

Combining the above estimates, we can bound $\|\nabla_\theta \mathcal{L}^{\tilde C}(\theta)  - \nabla_\theta \mathcal{L}^{C}(\theta)\|$ as
\begin{equation}
    \begin{aligned}
&MSE(\nabla_\theta \mathcal{L}^{\tilde C}(\theta), \nabla_\theta \mathcal{L}^{C}(\theta))\\
&\le \frac{1}{\sqrt{m}} \mathbb{E}_{t,r,x \sim  p_t(x) }\|\nabla_\theta u_{t\to r}^\theta(x)((u_{t\to t+\Delta t}^\theta(x)-u_{t\to t}^\theta(x))+(r-t-\Delta t)(D(\theta;r,t,x)-C(\theta;r,t,x)))\|\\
&\le \frac{1}{\sqrt{m}}  \mathbb{E}_{t,r,x \sim  p_t(x) }[\|\nabla_\theta u_{t\to r}^\theta(x)(u_{t\to t+\Delta t}^\theta(x)-u_{t\to t}^\theta(x))\|+(r-t-\Delta t)\|D(\theta;r,t,x)-C(\theta;r,t,x)\|]\\
&\le \frac{1}{\sqrt{m}}  \mathbb{E}_{t,r,x \sim  p_t(x) }[\|\nabla_\theta u_{t\to r}^\theta(x)\|_2\|u_{t\to t+\Delta t}^\theta(x)-u_{t\to t}^\theta(x)\|+(r-t-\Delta t)\|\partial_x u_{t+\Delta t\to r}^\theta(x')\|_2\|u^\theta_{t\to t+\Delta t}(x)-u^\theta_{t\to t}(x)\|]\\
&\le \frac{1}{\sqrt{m}}  \mathbb{E}_{t,r,x \sim  p_t(x) }[\|\nabla_\theta u_{t\to r}^\theta(x)\|_2\|u_{t\to t+\Delta t}^\theta(x)-u_{t\to t}^\theta(x)\|]+\mathbb{E}_{t,r,x \sim  p_t(x) }[\|\partial_x u_{t+\Delta t\to r}^\theta(x')\|_2\|u_{t\to t+\Delta t}^\theta(x)-u_{t\to t}^\theta(x)\|]\\
&\le  \sqrt{\mathbb{E}_{t,r,x \sim  p_t(x) }[\frac{1}{m}\|\nabla_\theta u_{t\to r}^\theta(x)\|_2^2]\mathbb{E}_{t,r,x \sim  p_t(x) }[\|u_{t\to t+\Delta t}^\theta(x)-u_{t\to t}^\theta(x)\|^2]}\\
&+\sqrt{\mathbb{E}_{t,r,x \sim  p_t(x) }[\frac{1}{m}\|\partial_x u_{t+\Delta t\to r}^\theta(x')\|_2^2]\mathbb{E}_{t,r,x \sim  p_t(x) }[\|u_{t\to t+\Delta t}^\theta(x)-u_{t\to t}^\theta(x)\|^2]}\\
&\le (\sqrt{\mathbb{E}_{t,r,x \sim  p_t(x) }[\frac{1}{m}\|\nabla_\theta u_{t\to r}^\theta(x)\|_2^2]\mathbb{E}_{t,r,x \sim  p_t(x) }[\|\partial_s u_{t\to s}^\theta|_{s=t}\|^2}\\
&+\sqrt{\mathbb{E}_{t,r,x \sim  p_t(x) }[\frac{1}{m}\|\partial_x u_{t+\Delta t\to r}^\theta(x')\|_2^2]\mathbb{E}_{t,r,x \sim  p_t(x) }[\|\partial_s u_{t\to s}^\theta|_{s=t}\|^2]})\Delta t + O(\Delta t^2)\\
& =  (M_gM_t + M_x M_t)\Delta t + O(\Delta t^2)
\end{aligned}
\end{equation}
where $M_g = \sqrt{\mathbb{E}_{t,r,x \sim  p_t(x) }[\frac{1}{m}\|\nabla_\theta u_{t\to r}^\theta(x)\|_2^2]} < +\infty$, $M_x = \sqrt{\mathbb{E}_{t,r,x \sim  p_t(x) }[\frac{1}{m}\|\partial_x u_{t+\Delta t\to r}^\theta(x')\|_2^2]} < +\infty$, and $M_t = \sqrt{\mathbb{E}_{t,r,x \sim  p_t(x) }[\|\partial_s u_{t\to s}^\theta|_{s=t}\|^2]} < +\infty$ as \autoref{assupmption}.

Combine with \autoref{lemma_for_u}, we have
\begin{equation}
    \begin{aligned}
        MSE(\nabla_\theta L^E(\theta),\nabla_\theta L^C(\theta))&\le MSE(\nabla_\theta L^E(\theta), \nabla_\theta L^{\tilde C}(\theta))+MSE(\nabla_\theta L^{\tilde C}(\theta), \nabla_\theta L^C(\theta))\\
        &\le  M_g \sqrt{\mathbb{E}_{t,r, x \sim  p_t(x)} [\|u_{t\to t}^\theta(x)-u_t(x)\|^2]} + (M_gM_t + M_x M_t)\Delta t + O(\Delta t^2)
    \end{aligned}
\end{equation}
\end{proof}

\subsection{Derivation of \autoref{traj_consist_x1}}
Substituting this relation $\tilde{x}_{t\to r}(x) = (1-t)\,u_{t\to r}(x) + x$ and $x_s = \phi_{t\to s}(x_t) = \frac{s-t}{1-t}(\tilde{x}_{t\to s}(x_t) - x_t)+x_t$ into \autoref{eq:traj_u_consistency}, we obtain 
\begin{equation}
    \begin{aligned}
        (r-t)u_{t\to r}(x_t)&= (s-t)u_{t\to s}(x_t) + (r-s)u_{s\to r}(x_s)\\
        (r-t)(\frac{\tilde{x}_{t\to r}(x_t)-x_t}{1-t})&= (s-t)\frac{\tilde{x}_{t\to s}(x_t)-x_t}{1-t} + (r-s)\frac{\tilde{x}_{s\to r}(x_s)-x_s}{1-s}\\
        \tilde{x}_{t\to r}(x_t)&= x_t+ \frac{s-t}{r-t}(\tilde{x}_{t\to s}(x_t)-x_t) + \frac{(1-t)(r-s)}{(1-s)(r-t)}(\tilde{x}_{s\to r}(x_s)-x_s)\\
        \tilde{x}_{t\to r}(x_t)&= x_t+ \frac{s-t}{r-t}(\tilde{x}_{t\to s}(x_t)-x_t) + \frac{(1-t)(r-s)}{(1-s)(r-t)}(\tilde{x}_{s\to r}(x_s)-\frac{s-t}{1-t}(\tilde{x}_{t\to s}(x_t) - x_t)-x_t)\\
        \tilde{x}_{t\to r}(x_t)&= x_t+ \frac{s-t}{r-t}(\tilde{x}_{t\to s}(x_t)-x_t) + \frac{(1-t)(r-s)}{(1-s)(r-t)}(\tilde{x}_{s\to r}(x_s)-\frac{s-t}{1-t}\tilde{x}_{t\to s}(x_t) - \frac{1-s}{1-t}x_t)\\
        \tilde{x}_{t\to r}(x_t)&= \frac{s-t}{r-t}\tilde{x}_{t\to s}(x_t) + \frac{(1-t)(r-s)}{(1-s)(r-t)}(\tilde{x}_{s\to r}(x_s)-\frac{s-t}{1-t}\tilde{x}_{t\to s}(x_t))\\
        \tilde{x}_{t\to r}(x_t)&= \frac{1-r}{1-s}\frac{s-t}{r-t}\tilde{x}_{t\to s}(x_t) + \frac{(1-t)(r-s)}{(1-s)(r-t)}\tilde{x}_{s\to r}(x_s)\\
       \frac{(s-t)(1-r)}{(1-s)(r-t)}\tilde{x}_{t\to r}(x_t)&= \frac{1-r}{1-s}\frac{s-t}{r-t}\tilde{x}_{t\to s}(x_t) + \frac{(1-t)(r-s)}{(1-s)(r-t)}(\tilde{x}_{s\to r}(x_s)-\tilde{x}_{t\to r}(x_t))\\
       \tilde{x}_{t\to r}(x_t)&= \tilde{x}_{t\to s}(x_t) + (r-s)\frac{(1-t)}{(1-r)}\frac{\tilde{x}_{s\to r}(x_s)-\tilde{x}_{t\to r}(x_t)}{s-t}\\
    \end{aligned}
\end{equation}


\subsection{\autoref{lemma_for_x1} and its proof}
\begin{lemma}\label{lemma_for_x1}
With $M'_g = \sqrt{\mathbb{E}_{t,r,x\sim p_t(x)}[\frac{1}{m}\|\nabla_\theta \tilde{x}^\theta_{t\to r}(x)\|_2^2]}<+\infty$ holds in \autoref{assupmption_x1}, $x_1$-prediction Euler Mean Flow loss $\mathcal{L}^{E'}(\theta)$ and the approximated $x_1$-prediction trajectory consistency loss $\mathcal{L}^{\tilde C'}(\theta)$ satisfy
 \begin{equation}
     \begin{aligned}
         MSE(\nabla\mathcal{L}^{E'}(\theta), \nabla  \mathcal{L}^{ \tilde C'}(\theta)) \le M'_g \sqrt{\mathbb{E}_{t,r, x \sim  p_t(x)} [\|\tilde{x}_{t\to t}^\theta(x)-\tilde{x}_t(x)\|^2]}
     \end{aligned}
 \end{equation}
Consequently, during training, if $\|\tilde{x}^\theta_{t\to t}(x) - \tilde{x}_t(x)\|^2 \to 0$, then $\mathcal{L}^{E'}(\theta)$ and $\mathcal{L}^{\tilde C'}(\theta)$ share the same optimal target at $\theta$.  The term $\tilde{x}_t(x)$ denotes the reference instantaneous velocity at $x$, defined as $\tilde{x}_t(x) = \mathbb{E}_{x_1 \sim p(x_1|x)}\big[\tilde{x}_t(x|x_1)\big]$,  which is generally intractable to compute analytically.

Here, the approximated $x_1$-prediction trajectory consistency loss $\mathcal{L}^{\tilde C'}(\theta)$ are defined as 
\begin{equation}
    \begin{aligned}
  \mathcal{L}^{\tilde C'}(\theta) &= \mathbb{E}_{t,r,x \sim  p_t(x), x'=sg(\Delta t \frac{\tilde{x}^\theta_{t\to t}(x)-x}{1-t})+x}\\
  &[|\tilde{x}^\theta_{t\to r}(x)  - (\tilde{x}^\theta_{t\to t}(x) + (r-t-\Delta t) \\
 &\frac{1-t}{1-r}\text{sg}(\frac{\tilde{x}^\theta_{t+\Delta t\to r}(x')-\tilde{x}^\theta_{t\to r}(x)}{\Delta t}))\|^2]       
    \end{aligned}
\end{equation}
It is straightforward to verify that the loss $\mathcal{L}^{\tilde C'}(\theta)$ is the mean-velocity formulation of $\mathcal{L}^{C'}(\theta)$ under the local linear approximation in \autoref{eq:EFM_equation}, expressed via $\tilde{x}_{t\to r}(x) =(1-t) \frac{\phi_{t\to r}(x)-x}{r-t} + x$, and differs by a temporal scaling factor $\frac{(1-t-\Delta t)}{\Delta t(1-t)(1-r)}$. 
\end{lemma}
\begin{proof}

We first define the reference regression loss $\mathcal{L}^{R}(\theta)$ as
\begin{equation} 
\begin{aligned} 
    \mathcal{L}^{R'}(\theta) = 
    &\mathbb{E}_{t,r,x \sim p_t(x), x'=sg(\Delta t \frac{\tilde{x}^\theta_{t\to t}(x)-x}{1-t})+x}\\ 
    &[|\tilde{x}_{t\to r}^\theta(x) - (\tilde{x}_t(x) + (r-t-\Delta t) \\ 
    &\frac{1-t}{1-r}\text{sg}(\frac{\tilde{x}^\theta_{t+\Delta t\to r}(x')-\tilde{x}^\theta_{t\to r}(x)}{\Delta t}))\|^2] \\ 
\end{aligned} 
\end{equation}

Let $B'(\theta;r,t,x) = (r-t-\Delta t)\frac{1-t}{1-r}sg(\frac{\tilde{x}_{t+\Delta t\to r}^\theta(x')-\tilde{x}_{t\to r}^\theta(x)}{\Delta t})$. Since $B'(\theta; r,t,x)$ contains the stop-gradient operator $\mathrm{sg}(\cdot)$, it satisfies $\nabla_\theta B'(\theta;r,t,x) = 0$.  Using $B'(\theta; r,t,x)$, the $x_1$-prediction Euler Mean Flow loss $\mathcal{L}^{E'}(\theta)$, the $x_1$-prediction reference regression loss $\mathcal{L}^{R'}(\theta)$, and the approximated trajectory consistency loss $\mathcal{L}^{\tilde C'}(\theta)$ can be written as
\begin{equation} 
\begin{aligned} \mathcal{L}^{E'}(\theta) &= \mathbb{E}_{t,r,x_1\sim p_{data}, x \sim  p_t(x|x_1)}[\|\tilde{x}_{t\to r}^\theta(x) -(\tilde{x}_t(x|x_1)+B'(\theta;r,t,x))\|^2]\\ 
\mathcal{L}^{R'}(\theta) &= \mathbb{E}_{t,r,x \sim p_t(x)}[\|\tilde{x}_{t\to r}^\theta(x) -(\tilde{x}_t(x)+B'(\theta;r,t,x))\|^2]\\ 
\mathcal{L}^{\tilde C'}(\theta) &= \mathbb{E}_{t,r,x \sim  p_t(x)}[\|\tilde{x}_{t\to r}^\theta(x) -(\tilde{x}_{t\to t}^\theta(x)+B'(\theta;r,t,x))\|^2] 
\end{aligned} 
\end{equation}

We first show that the $x_1$-prediction Euler Mean Flow loss $\mathcal{L}^{E'}(\theta)$ and $x_1$-prediction the reference regression loss $\mathcal{L}^{R'}(\theta)$ satisfy $\nabla_\theta \mathcal{L}^{E'}(\theta)$ = $\nabla_\theta \mathcal{L}^{R'}(\theta)$. Expanding $\nabla_\theta \mathcal{L}^{E'}(\theta)$, we obtain
\begin{equation} 
\begin{aligned} 
\nabla_\theta  \mathcal{L}^{E'}(\theta) &= \mathbb{E}_{t,r,x_1\sim p_{data}, x \sim  p_t(x|x_1), }[\nabla_\theta \|\tilde{x}_{t\to r}^\theta(x) -(\tilde{x}_t(x|x_1)+B'(\theta;r,t,x))\|^2] \\
&\overset{\nabla_\theta B'(\theta;r,t,x) = 0}{=} \mathbb{E}_{t,r,x_1\sim p_{data}, x \sim  p_t(x|x_1)}[\nabla_\theta \tilde{x}_{t\to r}^\theta(x) (\tilde{x}_{t\to r}^\theta(x)-\tilde{x}_t(x|x_1)-B'(\theta;r,t,x))]  \\
&\overset{\nabla_\theta B'(\theta;r,t,x) = 0}{=} \mathbb{E}_{t,r,x_1\sim p_{data}, x \sim  p_t(x|x_1)}[\nabla_\theta \tilde{x}_{t\to r}^\theta(x) (\tilde{x}_{t\to r}^\theta(x)-\tilde{x}_t(x|x_1)-B'(\theta;r,t,x))]  \\
\end{aligned} \end{equation}
where $\mathbb{E}_{t,r,x_1\sim p_{data}, x \sim  p_t(x|x_1)}[\nabla_\theta \tilde{x}_{t\to r}^\theta(x) (\tilde{x}_{t\to r}^\theta(x)-B'(\theta;r,t,x))]$ can be computed as:
\begin{equation}
    \begin{aligned}
        &\mathbb{E}_{t,r,x_1\sim p_{data}, x \sim  p_t(x|x_1)}[\nabla_\theta \tilde{x}_{t\to r}^\theta(x) (\tilde{x}_{t\to r}^\theta(x)-B'(\theta;r,t,x))]\\
        &= \int_{t,r} \int_{x_1}\int_x (\nabla_\theta \tilde{x}_{t\to r}^\theta(x) (\tilde{x}_{t\to r}^\theta(x)-B'(\theta;r,t,x))) p(x|x_1)p_{data}(x_1) p(t,r)dx dx_1dt dr\\
        &= \int_{t,r} \int_x (\nabla_\theta \tilde{x}_{t\to r}^\theta(x) (\tilde{x}_{t\to r}^\theta(x)-B'(\theta;r,t,x))) \int_{x_1}p(x|x_1)p_{data}(x_1)dx_1 p(t,r)dx dt dr\\
        &= \int_{t,r} \int_x (\nabla_\theta \tilde{x}_{t\to r}^\theta(x) (\tilde{x}_{t\to r}^\theta(x)-B'(\theta;r,t,x))) p(x) p(t,r)dx dt dr\\
        &= \mathbb{E}_{t,r,x \sim  p_t(x)}[\nabla \tilde{x}_{t\to r}^\theta(x) (\tilde{x}_{t\to r}^\theta(x)-B'(\theta;r,t,x))]
    \end{aligned}
\end{equation}
And $\mathbb{E}_{t,r,x_1\sim p_{data}, x \sim  p_t(x|x_1)}[\nabla \tilde{x}_{t\to r}^\theta(x) \tilde{x}_t(x|x_1)]$ can be calculated as
\begin{equation}
    \begin{aligned}
        &\mathbb{E}_{t,r,x_1\sim p_{data}, x \sim  p_t(x|x_1)}[\nabla_\theta \tilde{x}_{t\to r}^\theta(x) \tilde{x}_t(x|x_1)]\\
        &= \int_{t,r} \int_{x_1}\int_x \nabla_\theta \tilde{x}_{t\to r}^\theta(x) \tilde{x}_t(x|x_1) p(x|x_1)p_{data}(x_1) p(t,r)dx dx_1dt dr\\
        &= \int_{t,r} \int_x \nabla_\theta \tilde{x}_{t\to r}^\theta(x) (\int_{x_1}\tilde{x}_t(x|x_1) p(x|x_1) p_{data}(x_1) dx_1) p(t,r)dx dt dr\\
        &= \int_{t,r} \int_x \nabla_\theta \tilde{x}_{t\to r}^\theta(x) (\int_{x_1}\tilde{x}_t(x|x_1) p(x_1|x) p(x) dx_1) p(t,r)dx dt dr\\
        &= \int_{t,r} \int_x \nabla_\theta \tilde{x}_{t\to r}^\theta(x) p(x)\tilde{x}_t(x) p(t,r)dx dt dr\\
        &=  \mathbb{E}_{t,r,x \sim  p_t(x)}[\nabla_\theta \tilde{x}_{t\to r}^\theta(x) \tilde{x}_t(x)]
    \end{aligned}
\end{equation}
Therefore, we have 
\begin{equation} \label{eq:relation1}
\begin{aligned} 
\nabla_\theta  \mathcal{L}^{E'}(\theta) &= \mathbb{E}_{t,r,x_1\sim p_{data}, x \sim  p_t(x|x_1)}[\nabla \tilde{x}_{t\to r}^\theta(x) (\tilde{x}_{t\to r}^\theta(x)-\tilde{x}_t(x|x_1)-B(\theta;r,t,x))]  \\
&=\mathbb{E}_{t,r, x \sim  p_t(x)}[\nabla \tilde{x}_{t\to r}^\theta(x) (\tilde{x}_{t\to r}^\theta(x)-\tilde{x}_t(x)-B(\theta;r,t,x))]  \\
&= \nabla_\theta  \mathcal{L}^{R'}(\theta)
\end{aligned} \end{equation}

We then calculate the difference between $\nabla_\theta  \mathcal{L}^{R'}(\theta)$ and $\nabla_\theta  \mathcal{L}^{\tilde C'}(\theta)$ as
\begin{equation}
    \begin{aligned}
        \nabla_\theta  \mathcal{L}^{R'}(\theta) - \nabla_\theta  \mathcal{L}^{ \tilde C'}(\theta) &= \mathbb{E}_{t,r,x \sim  p_t(x)}[\nabla_\theta \tilde{x}_{t\to r}^\theta(x) (\tilde{x}_{t\to r}^\theta(x)-\tilde{x}_t(x)-B(\theta;r,t,x))\\
        &-\nabla_\theta \tilde{x}_{t\to r}^\theta(x) (\tilde{x}_{t\to r}^\theta(x)-\tilde{x}_{t\to t}^\theta(x)-B(\theta;r,t,x))\|^2]\\
        &=\mathbb{E}_{t,r, x \sim  p_t(x)} [\nabla_\theta \tilde{x}_{t\to r}^\theta(x)(\tilde{x}_{t\to t}^\theta(x)-\tilde{x}_t(x))]        
    \end{aligned}
\end{equation}
Applying the Cauchy-Schwarz inequality and using the assumption $M'_g = \sqrt{\mathbb{E}_{t,r,x\sim p_t(x)}[\frac{1}{m}\|\nabla_\theta \tilde{x}^\theta_{t\to r}(x)\|^2]}<+\infty$ in \autoref{assupmption_x1}, we further obtain the following bound:
\begin{equation}\label{eq:relation2}
    \begin{aligned}
        MSE(\nabla_\theta  \mathcal{L}^{R'}(\theta), \nabla_\theta  \mathcal{L}^{ \tilde C'}(\theta)) &=  \frac{1}{\sqrt{m}}\|\mathbb{E}_{t,r, x \sim  p_t(x)} [\nabla_\theta \tilde{x}_{t\to r}^\theta(x)(\tilde{x}_{t\to t}^\theta(x)-\tilde{x}_t(x))]\|\\
        &\le \frac{1}{\sqrt{m}}\mathbb{E}_{t,r,x \sim  p_t(x)} [\|\nabla_\theta \tilde{x}_{t\to r}^\theta(x)(\tilde{x}_{t\to t}^\theta(x)-\tilde{x}_t(x))\|]\\
        &\le \frac{1}{\sqrt{m}}\mathbb{E}_{t,r, x \sim  p_t(x)} [\|\nabla_\theta \tilde{x}_{t\to r}^\theta(x)\|_2\|\tilde{x}_{t\to t}^\theta(x)-\tilde{x}_t(x)\|]\\
        &\le \sqrt{\mathbb{E}_{t,r, x \sim  p_t(x)} [\frac{1}{m}\|\nabla_\theta \tilde{x}_{t\to r}^\theta(x)\|_2^2]\mathbb{E}_{t,r, x \sim  p_t(x)}[\|\tilde{x}_{t\to t}^\theta(x)-\tilde{x}_t(x)\|^2]}\\
        &\le M'_g \sqrt{\mathbb{E}_{t,r, x \sim  p_t(x)} [\|\tilde{x}_{t\to t}^\theta(x)-\tilde{x}_t(x)\|^2]}\\
    \end{aligned}
\end{equation}
Combine \autoref{eq:relation1} and \autoref{eq:relation2}, we have
\begin{equation}
    \begin{aligned}
        MSE(\nabla  \mathcal{L}^{E'}(\theta) , \nabla  \mathcal{L}^{ \tilde C'}(\theta)) \le M'_g \sqrt{\mathbb{E}_{t,r, x \sim  p_t(x)} [\|\tilde{x}_{t\to t}^\theta(x)-\tilde{x}_t(x)\|^2]}
    \end{aligned}
\end{equation}
\end{proof}

\subsection{Proof of \autoref{thm:validity_x1}}\label{thm:x1_validity_proof}
\paragraph{\autoref{thm:validity_x1}}(Surrogate Loss Validity for $x_1$-Prediction) With $M'_g = \sqrt{\mathbb{E}_{t,r,x \sim  p_t(x) }[\frac{1}{m}\|\nabla_\theta u_{t\to r}^\theta(x)\|_2^2]} < +\infty$, $M'_x = \sqrt{\mathbb{E}_{t,r,x \sim  p_t(x) }[\frac{1}{m}\|\partial_x u_{t+\Delta t\to r}^\theta(x')\|_2^2]} < +\infty$, and $M'_t = \sqrt{\mathbb{E}_{t,r,x \sim  p_t(x) }[\|\partial_s u_{t\to s}^\theta|_{s=t}\|^2]} < +\infty$ hold in \autoref{assupmption_x1} and \autoref{lemma_for_x1}, our Euler Mean Flow loss $\mathcal{L}^{E}(\theta)$ and the trajectory consistency loss $\mathcal{L}^{C}(\theta)$ satisfy
\begin{equation}
    \begin{aligned}
        &MSE(\nabla_\theta L^E(\theta) , \nabla_\theta L^C(\theta))\\
        &\le  M_g \sqrt{\mathbb{E}_{t,r, x \sim  p_t(x)} [\|\tilde{x}_{t\to t}^\theta(x)-\tilde{x}_t(x)\|^2]}\\
        &+ (M_gM_t + M_x M_t)\Delta t + O(\Delta t^2)
    \end{aligned}
\end{equation}
\begin{proof}
We define $C'(\theta;r,t,x) = sg(\frac{\tilde{x}_{t+\Delta t\to r}^\theta(x')-\tilde{x}_{t\to r}^\theta(x)}{\Delta t})$, $x' = x+ \frac{\tilde{x}_{t\to t}^\theta(x)-x}{1-t}\Delta t$ and $D'(\theta;r,t,x) =sg(\frac{\tilde{x}_{t+\Delta t\to r}^\theta(x'')-\tilde{x}_{t\to r}^\theta(x)}{\Delta t})$, $x'' = x+ \frac{\tilde{x}_{t\to t+\Delta t}^\theta(x)-x}{1-t}\Delta t$.  With these definitions, the approximated trajectory consistency loss
$\mathcal{L}^{\tilde C'}(\theta)$ and the trajectory consistency loss $\mathcal{L}^{C'}(\theta)$ can be written as
\begin{equation}
    \begin{aligned}
\mathcal{L}^{\tilde C'}(\theta) &= \mathbb{E}_{t,r,x \sim  p_t(x)}[\|\tilde{x}_{t\to r}^\theta(x) -sg(\tilde{x}_{t\to t}^\theta(x)+(r-t-\Delta t)\frac{1-t}{1-r}C(\theta;r,t,x))\|^2] \\
\mathcal{L}^{C'}(\theta) &= \mathbb{E}_{t,r,x \sim  p_t(x)}[\|\tilde{x}_{t\to r}^\theta(x) -sg(\tilde{x}_{t\to t+\Delta t}^\theta(x)+(r-t-\Delta t)\frac{1-t}{1-r}D(\theta;r,t,x))\|^2] 
    \end{aligned}
\end{equation}

We now analyze the difference $\nabla_\theta \mathcal{L}^{\tilde C'}(\theta)  - \nabla_\theta \mathcal{L}^{C'}(\theta)  $ between the gradients of these two objectives.  A direct computation yields
\begin{equation}
\begin{aligned}
\nabla \mathcal{L}_\theta^{\tilde C'}(\theta)  - \nabla \mathcal{L}_\theta^{C'}(\theta)  &=  \mathbb{E}_{t,r,x \sim  p_t(x)}[\nabla_\theta \tilde{x}_{t\to r}^\theta(x)(\tilde{x}_{t\to r}^\theta(x) -(\tilde{x}_t^\theta(x)+(r-t-\Delta t)\frac{1-t}{1-r}C'(\theta;r,t,x))]\\
&-\mathbb{E}_{t,r,x \sim  p_t(x)}[\nabla_\theta \tilde{x}_{t\to r}^\theta(x)(\tilde{x}_{t\to r}^\theta(x) -(\tilde{x}_{t\to t+\Delta t}^\theta(x)+(r-t-\Delta t)\frac{1-t}{1-r}D'(\theta;r,t,x))]\\
&= \mathbb{E}_{t,r,x \sim  p_t(x)}[\nabla_\theta \tilde{x}_{t\to r}^\theta(x)((\tilde{x}_{t\to t+\Delta t}^\theta(x)-\tilde{x}_{t\to t}^\theta(x))+(r-t-\Delta t)\frac{1-t}{1-r}(D'(\theta;r,t,x)-C'(\theta;r,t,x)))]\\
\end{aligned}
\end{equation}

We first bound the difference $D'(\theta; r,t,x)-C'(\theta; r,t,x)$. By definition,
\begin{equation}
\begin{aligned}
\|D'(\theta;r,t,x)- C'(\theta;r,t,x)\|  &=\|\frac{\tilde{x}_{t+\Delta t\to r}^\theta(x')-\tilde{x}_{t\to r}^\theta(x)}{\Delta t}- \frac{\tilde{x}_{t+\Delta t\to r}^\theta(x'')-\tilde{x}_{t\to r}^\theta(x)}{\Delta t}\|\\
&= \frac{1}{\Delta t}\|\tilde{x}_{t+\Delta t\to r}^\theta(x') - \tilde{x}_{t+\Delta t\to r}^\theta(x'')\|\\
&= \frac{1}{\Delta t}\|\partial_x \tilde{x}_{t+\Delta t\to r}^\theta(x')(x''-x')\|\\
&= \frac{1}{\Delta t}\|\partial_x \tilde{x}_{t+\Delta t\to r}^\theta(x')(\frac{\tilde{x}^\theta_{t\to t+\Delta t}(x)-\tilde{x}^\theta_{t\to t}(x)}{1-t}) \Delta t\|\\
&= \|\partial_x \tilde{x}_{t+\Delta t\to r}^\theta(x')(\frac{\tilde{x}^\theta_{t\to t+\Delta t}(x)-\tilde{x}^\theta_{t\to t}(x)}{1-t})\|\\
&\le \frac{1}{1-t}\|\partial_x \tilde{x}_{t+\Delta t\to r}^\theta(x')\|_2\|\tilde{x}^\theta_{t\to t+\Delta t}(x)-\tilde{x}^\theta_{t\to t}(x)\|
\end{aligned}
\end{equation}

Next, the difference $\tilde{x}_{t\to t+\Delta t}^\theta(x) -\tilde{x}_{t\to t}^\theta(x)$  admits a first-order expansion:
\begin{equation}
    \begin{aligned}
\|\tilde{x}_{t\to t+\Delta t}^\theta(x) -\tilde{x}_{t\to t}^\theta(x) \|  &= \|\tilde{x}_{t\to t}^\theta(x)  + \Delta t\partial_s\tilde{x}^\theta_{t\to s}|_{s=t} + O(\Delta t^2)-\tilde{x}_{t\to t}^\theta(x) \|\\
&\le \Delta t \|\partial_s\tilde{x}_{t\to s}^\theta|_{s= t}\|+O(\Delta t^2)\\
\end{aligned}
\end{equation}

Combining the above estimates, we can bound $MSE(\nabla_\theta \mathcal{L}^{\tilde C'}(\theta), \nabla_\theta \mathcal{L}^{C'}(\theta))$ as
\begin{equation}
    \begin{aligned}
&MSE(\nabla_\theta \mathcal{L}^{\tilde C'}(\theta),\nabla_\theta \mathcal{L}^{C'}(\theta))\\
&\le  \frac{1}{\sqrt{m}}\mathbb{E}_{t,r,x \sim  p_t(x) }\|\nabla_\theta \tilde{x}_{t\to r}^\theta(x)((\tilde{x}_{t\to t+\Delta t}^\theta(x)-\tilde{x}_{t\to t}^\theta(x))+(r-t-\Delta t)\frac{1-t}{1-r}(D'(\theta;r,t,x)-C'(\theta;r,t,x)))\|\\
&\le  \frac{1}{\sqrt{m}}\mathbb{E}_{t,r,x \sim  p_t(x) }[\|\nabla_\theta \tilde{x}_{t\to r}^\theta(x)(\tilde{x}_{t\to t+\Delta t}^\theta(x)-\tilde{x}_{t\to t}^\theta(x))\|+\frac{1}{\sqrt{m}}(r-t-\Delta t)\frac{1-t}{1-r}\|D'(\theta;r,t,x)-C'(\theta;r,t,x)\|]\\
&\le \frac{1}{\sqrt{m}} \mathbb{E}_{t,r,x \sim  p_t(x) }[\|\nabla_\theta \tilde{x}_{t\to r}^\theta(x)\|_2\|\tilde{x}_{t\to t+\Delta t}^\theta(x)-\tilde{x}_{t\to t}^\theta(x)\|+\frac{1}{\sqrt{m}}(r-t-\Delta t)\frac{1}{1-r}\|\partial_x \tilde{x}_{t+\Delta t\to r}^\theta(x')\|_2\|\tilde{x}^\theta_{t\to t+\Delta t}(x)-\tilde{x}^\theta_{t\to t}(x)\|]\\
&\le  \frac{1}{\sqrt{m}}\mathbb{E}_{t,r,x \sim  p_t(x) }[\|\nabla_\theta \tilde{x}_{t\to r}^\theta(x)\|_2\|\tilde{x}_{t\to t+\Delta t}^\theta(x)-\tilde{x}_{t\to t}^\theta(x)\|]+\frac{1}{\sqrt{m}}\frac{1}{1-r}\mathbb{E}_{t,r,x \sim  p_t(x) }[\|\partial_x \tilde{x}_{t+\Delta t\to r}^\theta(x')\|_2\|\tilde{x}_{t\to t+\Delta t}^\theta(x)-\tilde{x}_{t\to t}^\theta(x)\|]\\
&\le  \sqrt{\mathbb{E}_{t,r,x \sim  p_t(x) }[\frac{1}{m}\|\nabla_\theta \tilde{x}_{t\to r}^\theta(x)\|_2^2]\mathbb{E}_{t,r,x \sim  p_t(x) }[\|\tilde{x}_{t\to t+\Delta t}^\theta(x)-\tilde{x}_{t\to t}^\theta(x)\|^2]}\\
&+\frac{1}{1-r}\sqrt{\mathbb{E}_{t,r,x \sim  p_t(x) }[\frac{1}{m}\|\partial_x \tilde{x}_{t+\Delta t\to r}^\theta(x')\|_2^2]\mathbb{E}_{t,r,x \sim  p_t(x) }[\|\tilde{x}_{t\to t+\Delta t}^\theta(x)-\tilde{x}_{t\to t}^\theta(x)\|^2]}\\
&\le (\sqrt{\mathbb{E}_{t,r,x \sim  p_t(x) }[\frac{1}{m}\|\nabla_\theta \tilde{x}_{t\to r}^\theta(x)\|_2^2]\mathbb{E}_{t,r,x \sim  p_t(x) }[\|\partial_s \tilde{x}_{t\to s}^\theta|_{s=t}\|^2}\\
&+\frac{1}{1-r}\sqrt{\mathbb{E}_{t,r,x \sim  p_t(x) }[\frac{1}{m}\|\partial_x \tilde{x}_{t+\Delta t\to r}^\theta(x')\|_2^2]\mathbb{E}_{t,r,x \sim  p_t(x) }[\|\partial_s \tilde{x}_{t\to s}^\theta|_{s=t}\|^2]})\Delta t + O(\Delta t^2)\\
& =  (M_gM_t + M_x M_t)\Delta t + O(\Delta t^2)
\end{aligned}
\end{equation}
where $M'_g = \sqrt{\mathbb{E}_{t,r,x \sim  p_t(x) }[\frac{1}{m}\|\nabla_\theta \tilde{x}_{t\to r}^\theta(x)\|_2^2]} < +\infty$, $M_x = \frac{1}{1-r}\sqrt{\mathbb{E}_{t,r,x \sim  p_t(x) }[\frac{1}{m}\|\partial_x \tilde{x}_{t+\Delta t\to r}^\theta(x')\|_2^2]} < +\infty$, and $M'_t = \sqrt{\mathbb{E}_{t,r,x \sim  p_t(x) }[\|\partial_s \tilde{x}_{t\to s}^\theta|_{s=t}\|^2]} < +\infty$ as \autoref{assupmption_x1}.

Combine with \autoref{lemma_for_x1}, we have
\begin{equation}
    \begin{aligned}
        MSE(\nabla_\theta L^{E'}(\theta), \nabla_\theta L^{C'}(\theta))&\le MSE(\nabla_\theta L^{E'}(\theta) - \nabla_\theta L^{\tilde C'}(\theta))+MSE(\nabla_\theta L^{\tilde C'}(\theta) - \nabla_\theta L^{C'}(\theta))\\
        &\le  M'_g \sqrt{\mathbb{E}_{t,r, x \sim  p_t(x)} [\|\tilde{x}_{t\to t}^\theta(x)-\tilde{x}_t(x)\|^2]} + (M'_gM'_t + M'_x M'_t)\Delta t + O(\Delta t^2)
    \end{aligned}
\end{equation}
\end{proof}

\section{Model Architecture and Details of Dataset, Training, Sampling and Results}

\begin{algorithm}[t]
\caption{Euler Mean Flow: Sampling\\
\footnotesize\emph{Highlighted parts are used for conditional generation.}}
\label{alg:sampling}
\begin{algorithmic}[1]
\REQUIRE parameters $\theta$, learning rate $\eta$, noise sampler $\mathcal{N}$  
\REPEAT
    \STATE Sample $x_0\sim \mathcal{N}$
    \IF{$u$-prediction}
        \STATE $x_1 = u^\theta_{0\to 1}(x,\highlight{C}) + x$\\
    \hspace{-3.7mm}\textbf{else if }{$x_1$-prediction} \textbf{then}
        \STATE $x_1 = \tilde{x}^{\theta}_{0\to 1}(x,\highlight{C})$
    \ENDIF
\UNTIL{convergence}
\end{algorithmic}
\end{algorithm}




\subsection{Algorithm Details} \label{sec:alg_details}

\paragraph{Classifier-Free Guidance (CFG)}  For conditional generation, we follow \cite{geng2025mean} and apply classifier-free guidance (CFG) during training by modifying the conditional field.  Specifically, $u(x|x_1)$ is replaced by $wu(x | x_1) + (1 - w - k)u^\theta_{t \to t}(x_t, C_0) + ku^\theta_{t \to t}(x_t, C)$, where $C$ is the label of $x_1$ and $C_0$ denotes the null label. The effective guidance scale is $w' = \frac{w}{1 - k}$.  Unconditional capability for $u^\theta_{t \to t}(x_t, C_0)$ is enabled by dropping labels with probability $p = 0.1$ during training.

\paragraph{Time Sampler}  Following \cite{geng2025mean}, we independently sample $t, r \sim \mathcal{T}_1$ and swap them if $t > r$, forming the sampler $\mathcal{T}$. We use $\mathcal{T}_1=\mathcal{U}[0,1]$ by default and a log-normal distribution for ImageNet.  In addition, a fraction $\alpha$ of samples is constructed with $r = t$, corresponding to training the instantaneous model, e.g., $u_{t \to t}^\theta$. This ensures that the validity condition $\|u^\theta_{t\to t}(x) - u_t(x)\| \to 0$ required by \autoref{thm:validity} for $\mathcal{L}^{EFM}$ is satisfied. 

\paragraph{Adaptive Loss}  For Training stability, we follow \cite{geng2025mean} and adopt an adaptive loss \cite{geng2024consistency} to reweight loss as $w\|\Delta\|_2^2$, $w = \frac{1}{(\|\Delta\|_2^2 + c)^p}$ to stabilize learning, where $\Delta$ denotes the discrepancy in the loss.

\begin{figure}[t]
    \centering
    \includegraphics[width=0.99\linewidth]{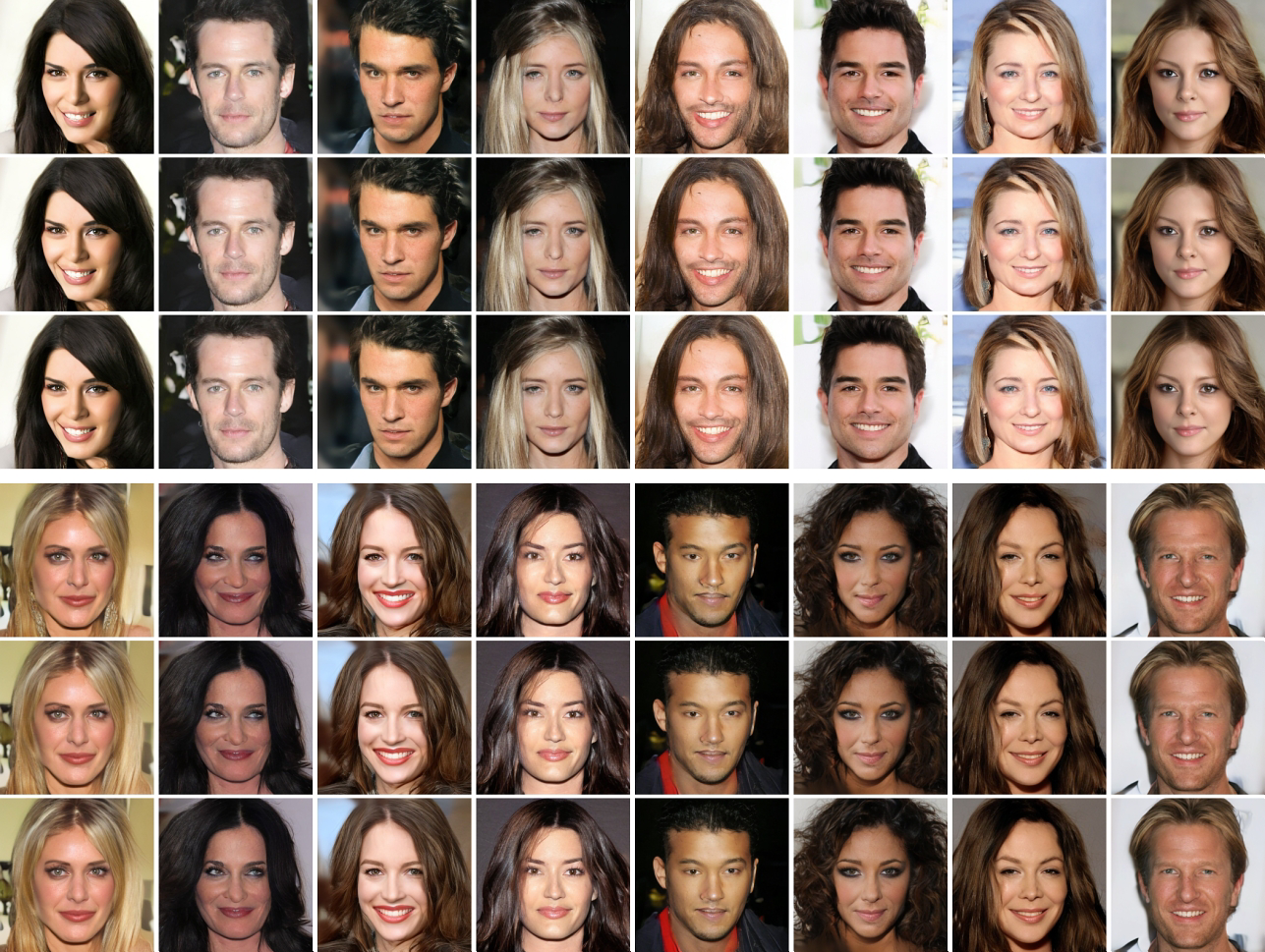}
    \caption{Latent image unconditional generation result trained on CelebA-HQ dataset~\cite{liu2015faceattributes}. First, second and third rows shows 1-step, 2-steps and 4-steps generation respectively.}
    \label{fig:latent_celeb}
\end{figure}

\begin{figure}[t]
    \centering
    \includegraphics[width=0.99\linewidth]{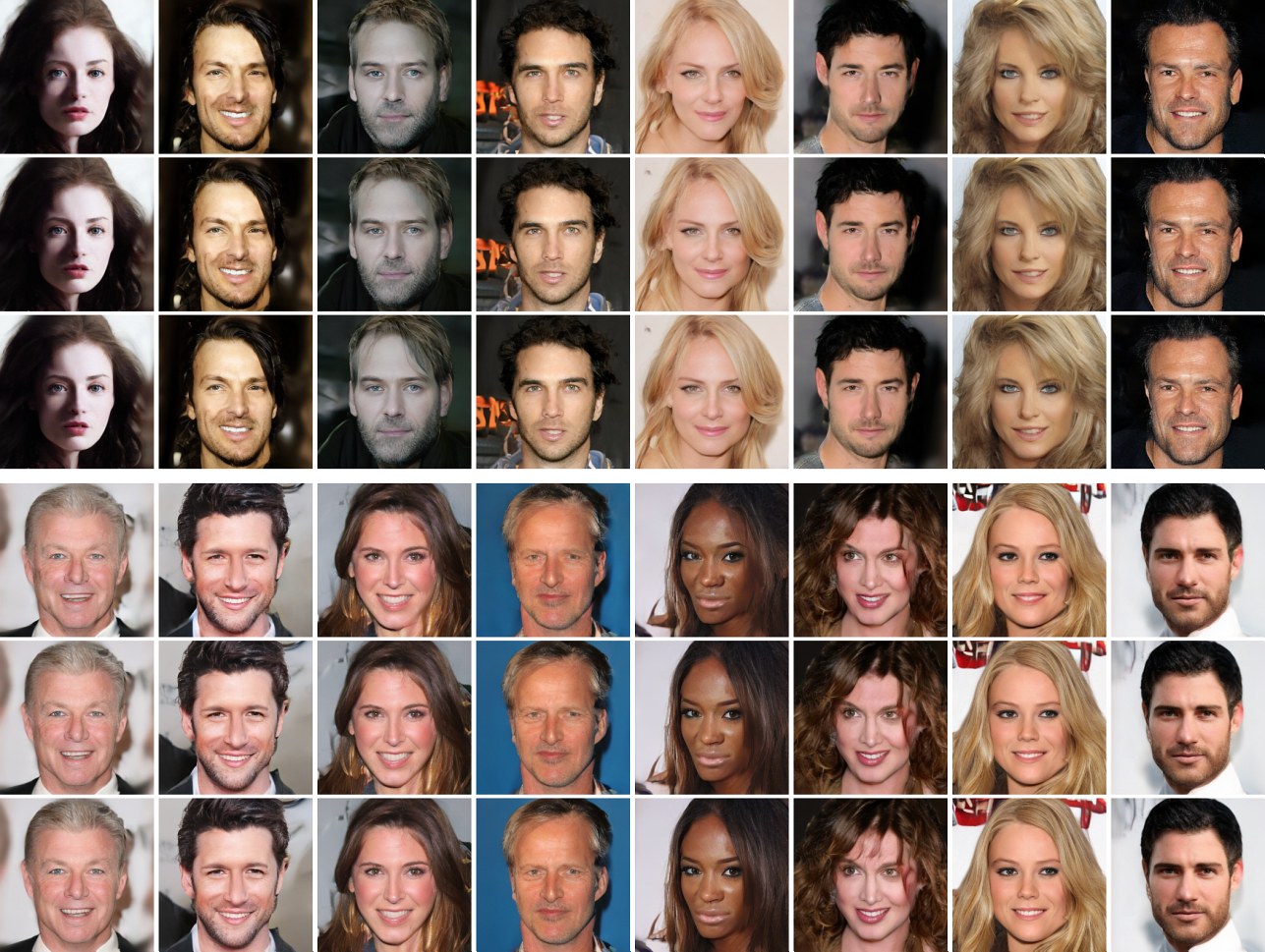}
    \caption{Latent image unconditional generation result trained on CelebA-HQ dataset~\cite{liu2015faceattributes}. First, second and third rows shows 1-step, 2-steps and 4-steps generation respectively.}
    \label{fig:latent_celeb2}
\end{figure}

\begin{figure}[t]
    \centering
    \includegraphics[width=0.99\linewidth]{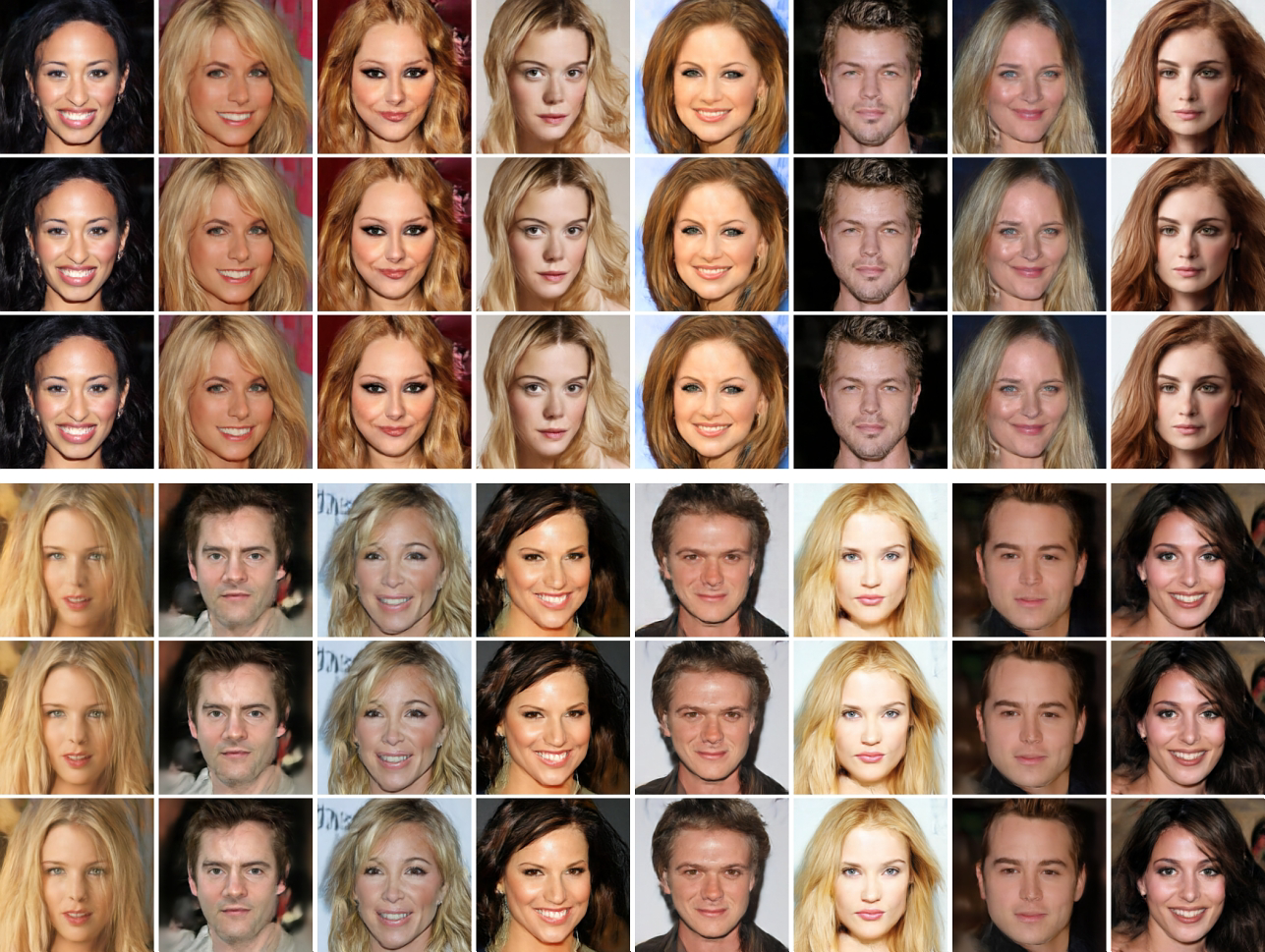}
    \caption{JiT image unconditional generation result trained on CelebA-HQ dataset~\cite{liu2015faceattributes}. First, second and third rows shows 1-step, 2-steps and 4-steps generation respectively.}
    \label{fig:jit_celeb1}
\end{figure}

\begin{figure}[t]
    \centering
    \includegraphics[width=0.99\linewidth]{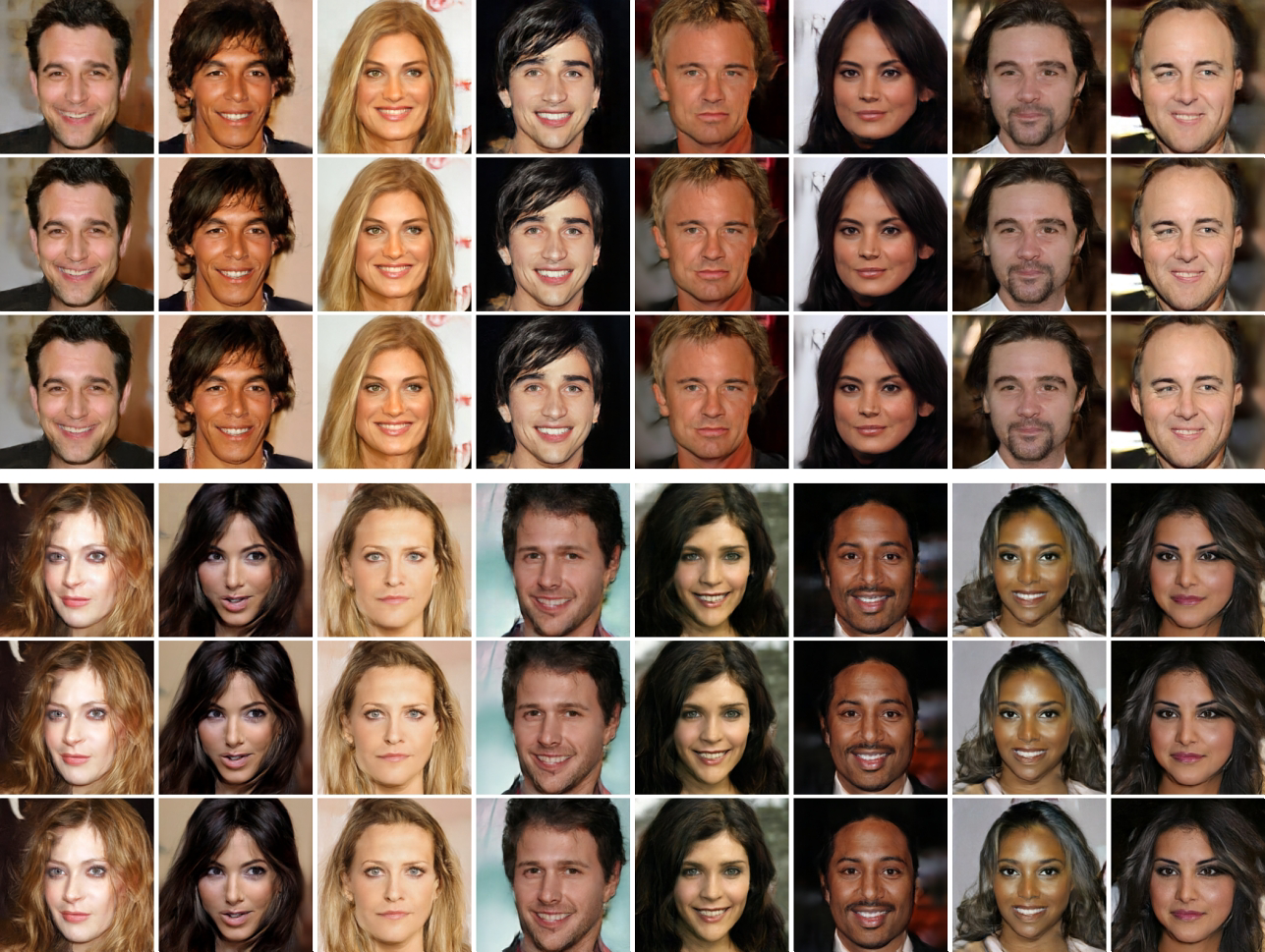}
    \caption{JiT image unconditional generation result trained on CelebA-HQ dataset~\cite{liu2015faceattributes}. First, second and third rows shows 1-step, 2-steps and 4-steps generation respectively.}
    \label{fig:jit_celeb2}
\end{figure}

\begin{figure*}[t]
    \centering
    \includegraphics[width=0.99\linewidth]{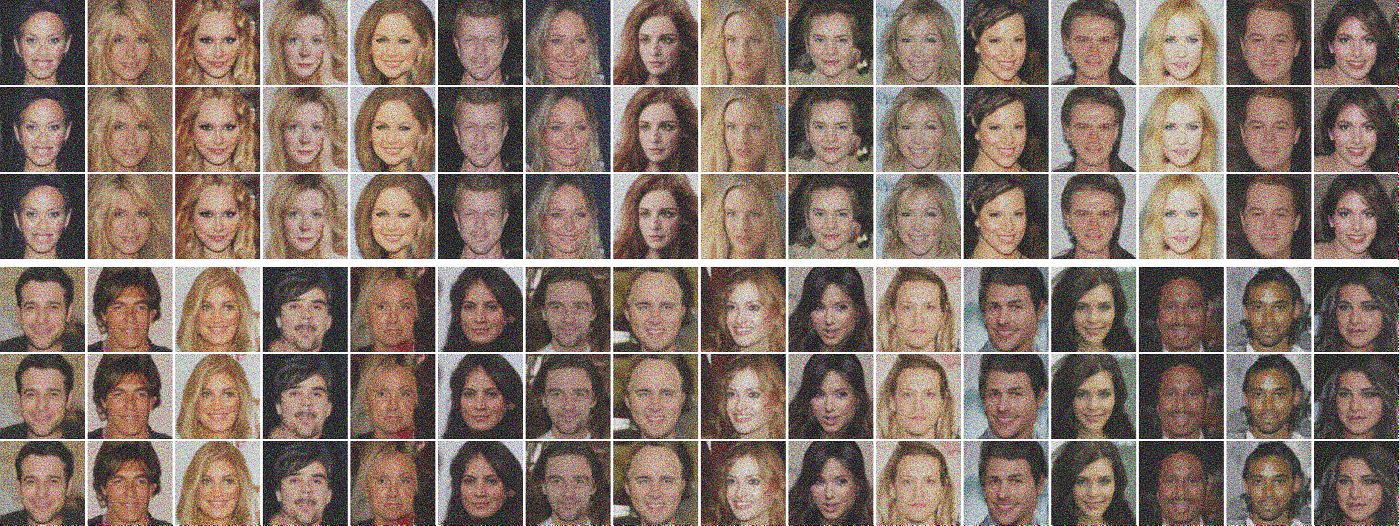}
    \caption{We show u-prediction EMF will fail to generate clean images}
    \label{fig:jit_u_celeb}
\end{figure*}


\subsection{Latent Space Image Generation} \label{sec:latent_image_generation}

\begin{table}[t]
\centering
\small
\begin{minipage}[t]{0.33\textwidth}
\centering
\setlength{\tabcolsep}{8pt}
\renewcommand{\arraystretch}{1.2}
\begin{tabular}{|l|l|}
\hline
Batch Size & 64 \\
Training Steps & 400K \\
Classifier Free Guidance & - \\
Class Dropout Probability & -\\
EMA Ratio & 0.9999 \\ \hline
Optimizer & Adam \\
Learning Rate & 1e-4\\
Weight Decay &  0 \\
Patch Size & 2 \\
Backbone & DiT-B/2\\
\hline
\end{tabular}
\caption{Latent-Based CelebA-HQ.}
\label{tab:celeba_latent_hyperparams}
\end{minipage}
\begin{minipage}[t]{0.33\textwidth}
\centering
\setlength{\tabcolsep}{8pt}
\renewcommand{\arraystretch}{1.2}
\begin{tabular}{|l|l|}
\hline
Batch Size & 256 \\
Training Steps & 800K \\
Classifier Free Guidance & 2.5 \\
Class Dropout Probability & 0.1\\
EMA Ratio & 0.9999 \\ \hline
Optimizer & Adam \\
Learning Rate & 1e-4\\
Weight Decay &  0 \\
Patch Size & 2 \\
Backbone & DiT-B/2\\
\hline
\end{tabular}
\caption{Latent-Based ImageNet.}
\label{tab:imagenet_latent_hyperparams}
\end{minipage}
\begin{minipage}[t]{0.33\textwidth}
\centering
\setlength{\tabcolsep}{8pt}
\renewcommand{\arraystretch}{1.2}
\begin{tabular}{|l|l|}
\hline
Batch Size & 64 \\
Training Steps & 600K \\
Classifier Free Guidance & - \\
Class Dropout Probability & -\\
EMA Ratio & 0.9999 \\ \hline
Optimizer & Adam \\
Learning Rate & 1e-4\\
Weight Decay &  0 \\
Patch Size & 16 \\
Backbone & JiT-B/16\\
\hline
\end{tabular}
\caption{Pixel-Based CelebA-HQ.}
\label{tab:pixel_hyperparams}
\end{minipage}

\end{table}

\paragraph{Model} We adopt a Diffusion Transformer (DiT) \cite{peebles2023scalable} architecture with DiT-B/2 configuration as our backbone for image generation. The input image $x$ is first encoded as a latent $z$ by a pretrained variational autoencoder (VAE) model~\cite{kingma2022auto} from Stable Diffusion~\cite{rombach2021highresolution}. For $256 \times 256 \times 3$ images, the shape of $z$ is $32 \times 32 \times 4$. The latent $z$ is first partitioned into non-overlapping patches of size $2 \times 2$, resulting in a token sequence. Each patch is linearly projected into a $D$-dimensional embedding space, where $D=768$ for DiT-B/2. The backbone consists of a stack of Transformer blocks with multi-head self-attention and MLP layers. For conditioning, we follow the AdaLN-Zero design introduced in DiT. Specifically, the time embeddings for $t$ and $r$, together with optional class embeddings for conditional generation, are first projected through a small MLP and then used to modulate the Transformer blocks via adaptive layer normalization.  See \autoref{tab:celeba_latent_hyperparams} and \autoref{tab:imagenet_latent_hyperparams} for hyperparameters in detail.

\paragraph{Datasets}
We trained DiT-B/2 on two image datasets: ImageNet-1000~\cite{deng2009imagenet} and CelebA-HQ~\cite{liu2015faceattributes}. ImageNet-1000 contains approximately 1.28M training images and 50K validation images spanning 1,000 object categories. CelebA-HQ contains 30,000 high-resolution human face images derived from CelebA. All dataset are resized to a resolution of $256\times256$. 

\paragraph{Metric}To evaluate generative performance, we generate 50K samples for each trained model and compare them against the corresponding real datasets. We report the Fréchet Inception Distance (FID)~\cite{heusel2017gans} computed using Inception-V3 features.  We follow the same evaluation protocol as in~\cite{geng2025mean} for FID computation.

\begin{table}[t]
\centering
\caption{Comparison of memory and computational cost between our method and MeanFlow for unconditional (CelebA-HQ) and conditional (ImageNet-1000) generation using the DiT-B/2 model. ``Peak'' denotes the maximum GPU memory usage during training, ``Fixed'' refers to the constant memory overhead, and aux-EMF indicates our method with a 4-block auxiliary head.  All experiments are conducted on a single H200 GPU with batch sizes of 64 for CelebA-HQ and 128 for ImageNet-1000, using EMA and AdamW optimization with mixed-precision (FP16) training in PyTorch.}
\label{tab:memory_time_comparison}
\begin{tabular}{lccccc}
\hline
\textbf{Method} & \textbf{Dataset} & \textbf{Peak Memory} & \textbf{Fixed Memory} & \textbf{Speed / Iter} & \textbf{FID} \\
\hline
MeanFlow & CelebA-HQ & 32.1GB & 2.3GB  & 151.4ms & 12.4  \\
EMF (Ours)                  & CelebA-HQ & 23.3GB & 2.3GB & 91.2 ms & \textbf{10.9} \\
aux-EMF (Ours)              & CelebA-HQ & \textbf{17.6GB}  & 2.8 GB & \textbf{84.2 ms} & 11.7 \\
MeanFlow & ImageNet &101.9GB  & 2.4GB & 400.9ms & 11.1 \\
EMF (Ours)              & ImageNet & \textbf{57.9GB} & 2.4GB  &  \textbf{198.8ms} & \textbf{7.2} \\
\hline
\end{tabular}
\end{table}  

\begin{table}[t]
\centering
\caption{
Comparison of training objectives under equivalent architecture (DiT-B) and compute.
FID-50k scores (lower is better) are shown over 128-, 4-, and 1-step denoising.  
}
\label{tab:training_objective_comparison}
\begin{tabular}{lccc ccc}
\toprule
& \multicolumn{3}{c}{CelebA-HQ-256 (unconditioned)}
& \multicolumn{3}{c}{ImageNet-256 (class conditioned)} \\
\cmidrule(lr){2-4} \cmidrule(lr){5-7}
Method
& 128-Step & 4-Step & 1-Step
& 128-Step & 4-Step & 1-Step \\
Diffusion \cite{song2021denoising}
& 23.0 & 123.4 & 132.2
& 39.7 & 464.5 & 467.2 \\
FM \cite{lipman2023flow} & 7.3 & 63.3 & 280.5
& 17.3 & 108.2 & 324.8 \\
PD \cite{salimans2022progressive} & 302.9 & 251.3 & 14.8 & 201.9 & 142.5 & 35.6 \\
CD \cite{song2023consistency}& 59.5 & 39.6 & 38.2 & 132.8 & 98.01 & 136.5 \\
Reflow \cite{liu2023flow} & 16.1 & 18.4 & 23.2 & 16.9 & 32.8 & 44.8 \\
CM \cite{song2023consistency} & 53.7 & 19.0 & 33.2 & 42.8 & 43.0 & 69.7 \\
ShortCut \cite{frans2025shortcut}& 6.9 & 13.8 & 20.5 & 15.5 & 28.3 & 40.3 \\
MF \cite{geng2025mean} & -- & -- & 12.4 & 6.4 & 7.1 & 11.1 \\
EMF (Ours) & -- & -- & 10.8 & 5.6 & 6.9 & 7.2 \\
\bottomrule
\end{tabular}
\end{table}

\paragraph{Result}
For qualitative evaluation, we present unconditional 1-, 2-, and 4-step generation results on CelebA-HQ in Figures~\ref{fig:latent_celeb} and~\ref{fig:latent_celeb2}. We also show conditional 1-, 2-, and 4-step generation results on ImageNet in Figures~\ref{fig:imagenet_part1}--\ref{fig:imagenet_part3}, using the image category as the guidance condition. In both cases, we see our 1-step generation result give reasonably good result comparing against few-step generations. For quantitative comparison, we report FID scores for both datasets in Table~\ref{tab:training_objective_comparison} where our method achieves the best result overall.

\subsection{Pixel Space Image Generation} \label{sec:pixel_image_generation}

\paragraph{Model} We conduct pixel-space image generation experiments using the Just Image Transformers (JiT) architecture~\cite{li2025back}, training the JiT-B/16 model on CelebA-HQ. Conceptually, JiT is a plain Vision Transformer (ViT) applied to patches of pixels without latent encoding.  An input image of resolution $H\times W\times C$ is divided into non-overlapping $p \times p$ patches, producing a sequence of patch tokens. To ensure sufficient capacity to model high-dimensional images, JiT uses large patch size ($p=16$) to balance spatial token length and per-token dimensionality. Each patch token, of dimensionality $p^2C$, is linearly embedded and combined with sinusoidal positional embeddings before being processed by a stack of Transformer blocks. For conditioning on time $t$, $r$ and class information (when applicable), JiT uses AdaLN-Zero similar to DiT. The output tokens are projected back to patch RGB values to reconstruct the full high-resolution image. See \autoref{tab:pixel_hyperparams} for hyperparameters in detail. 

\begin{table}[t]
\centering
\caption{
Comparison of pixel-space generative methods under equivalent architectures (DiT-B) and computational budgets on CelebA-HQ-256.
FID-50k scores (lower is better) are reported for 2-step and 1-step denoising, while FID-10k scores are used for the 128-step setting.
}
\label{tab:pixel_space_method_comparison_celeba}
\begin{tabular}{lcccc}
\toprule
Method & Variant & 128-Step & 2-Step & 1-Step \\
\midrule
JiT \cite{li2025back} & $u$-pred, $u$-loss
& 339.7 & 384.4 & 407.0 \\
JiT \cite{li2025back} & $x_1$-pred, $x_1$-loss
& 27.9 & 441.6 & 440.1 \\
MeanFlow \cite{li2025back} & $x_1$-pred, $x_1$-loss
& 42.2 & 41.5 & 56.8 \\
EMF (Ours) & $u$-pred, $u$-loss
& 329.4 & 323.3 & 324.6 \\
EMF (Ours) & $x_1$-pred, $x_1$-loss
& 21.4 & 26.4 & 30.6\\
EMF (Ours) & $x_1$-pred, $u$-loss
& 35.8 & 34.8 & 36.3 \\
\bottomrule
\end{tabular}%
\end{table}

\paragraph{Result}

Unconditional pixel-space generation results using the JiT architecture combined with our EMF method, trained on CelebA-HQ with the $x_1$-prediction objective, are shown in Figures~\ref{fig:jit_celeb1} and~\ref{fig:jit_celeb2}. Our method maintains consistent visual quality across 1-, 2-, and 4-step sampling. We further verify that the $x_1$-prediction objective is essential: when trained with the $u$-prediction objective, the generated images remain noisy even as the number of inference steps increases (\autoref{fig:jit_u_celeb}). Quantitative results are reported in Table~\ref{tab:pixel_space_method_comparison_celeba}.

\subsection{Functional Image Generation}\label{sec:functional}
\paragraph{Model} We build upon the Infty-Diff architecture~\cite{bond2024infty}, which models both inputs and outputs as continuous image functions represented by randomly sampled pixel coordinates. As shown in \autoref{fig:infty-diff}, the network adopts a hybrid sparse–dense design composed of a Sparse Neural Operator and a Dense U-Net to support learning from sparse functional observations. The Sparse Neural Operator first embeds irregularly sampled pixels into feature vectors. These features are interpolated onto a coarse dense grid using KNN interpolation with neighborhood size 3, enabling subsequent dense processing. A U-Net is then applied on a $128\times128$ grid for $256\times256$ images, with 128 base channels and five resolution stages with channel multipliers $[1,2,4,8,8]$. Self-attention blocks are inserted at the $16\times16$ and $8\times8$ resolutions to enhance global context modeling. The dense features are subsequently mapped back to the original coordinate set via inverse KNN interpolation and further refined by a second Sparse Neural Operator, with a residual connection applied to the initial sparse features.

\newpage
\begin{wrapfigure}{r}{0.50\textwidth}
  \centering
  \vspace{-0.5em}
  \includegraphics[width=\linewidth]{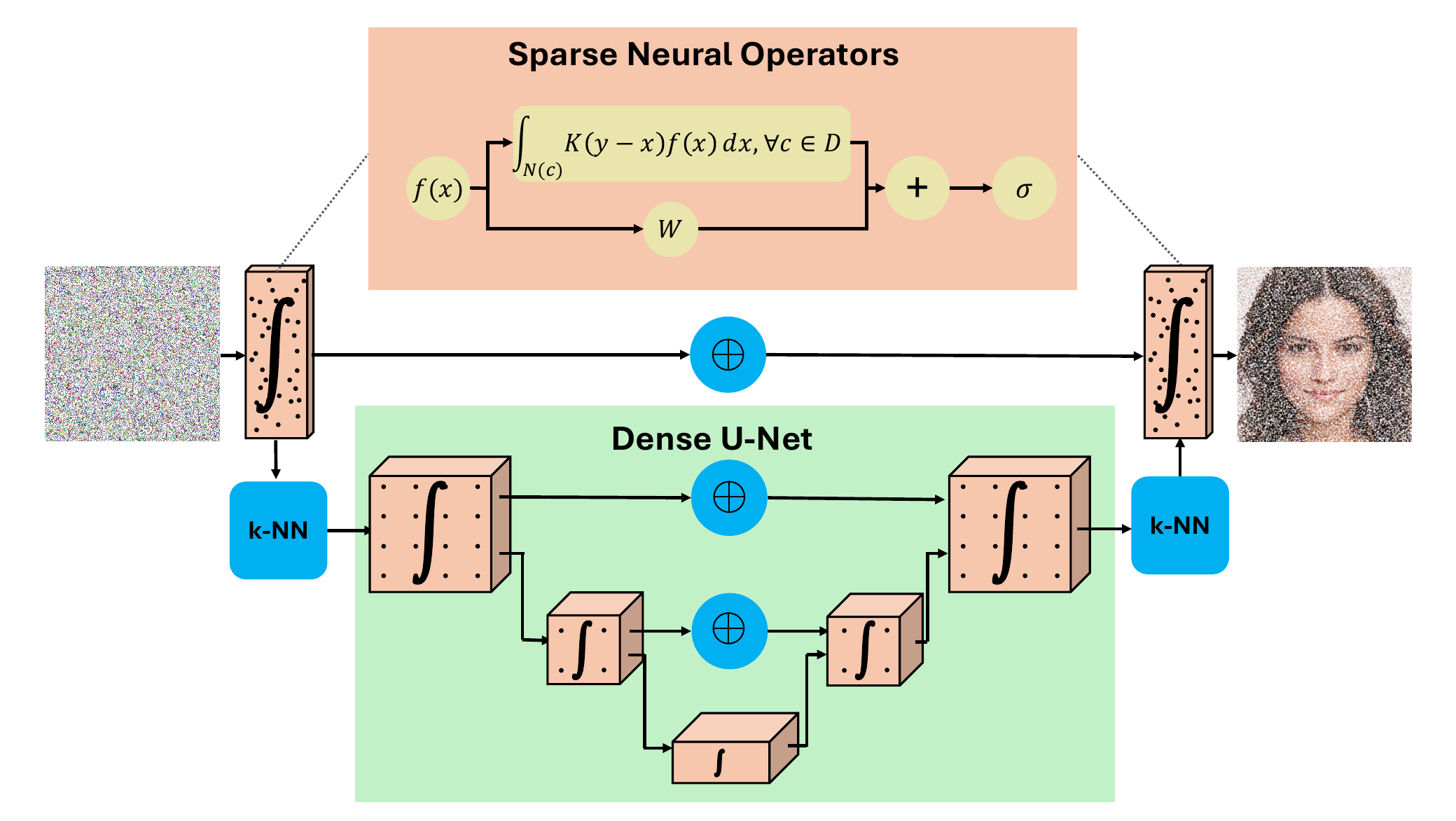}
  \caption{The network design of Infty-Diff.}
  \label{fig:infty-diff}
  \vspace{-0.8em}
\end{wrapfigure}
Following Infty-Diff, we implement the Sparse Neural Operator using linear-kernel sparse convolutions with TorchSparse for efficiency. Each Sparse Operator module is composed of five convolutional layers in sequence. It begins with a pointwise convolution, followed by three linear-kernel operator layers. Each operator layer applies a sparse depthwise convolution with 64 channels and a kernel size of 7 (for $256\times256$-resolution images), and is followed by two pointwise convolutions with 128 hidden channels to mix channel-wise information. A final pointwise convolution projects the features to the output dimension. Time conditioning is incorporated in both the sparse and dense components using sinusoidal positional embeddings~\cite{vaswani2017attention}, following the Mean Flow formulation~\cite{geng2025mean}. The embeddings of $t$ and $r$ are summed and injected in place of the original time conditioning used in Infty-Diff. The resulting model contains approximately 420M trainable parameters.

\paragraph{Dataset} We conduct experiments on two image datasets: FFHQ~\cite{karras2019style} and CelebA-HQ. FFHQ contains 70000 diverse face images. All images are resized to $256\times256$. Following Infty-Diff~\cite{bond2024infty}, we randomly sample 25\% of image pixels during training to evaluate functional-based generation.

\begin{figure*}[t]
    \centering
    \includegraphics[width=0.99\linewidth]{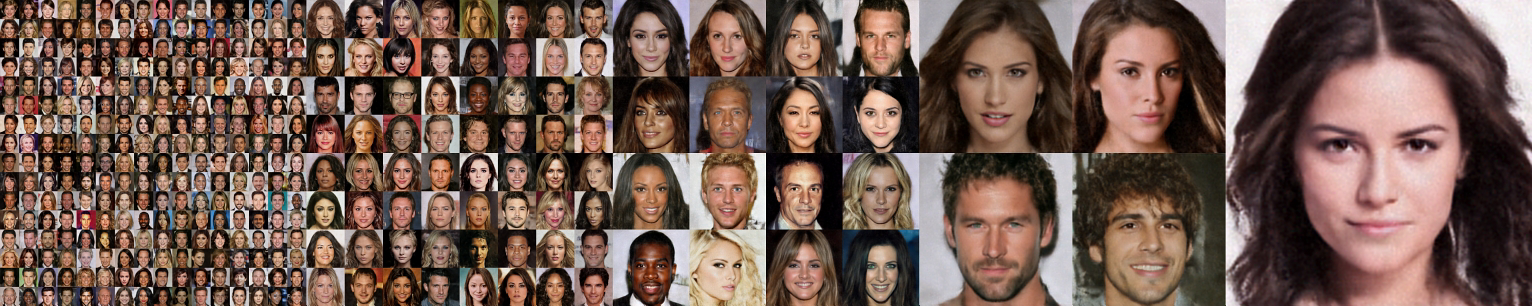}
    \caption{1-step functional generation for images on CelebA-HQ~\cite{liu2015faceattributes} dataset. From left to right, we show images generated at 64, 128, 256, 512 and 1024n resolution respectively.}
    \label{fig:celeb_functional}
\end{figure*}

\begin{figure}[t]
    \centering
    \includegraphics[width=0.99\linewidth]{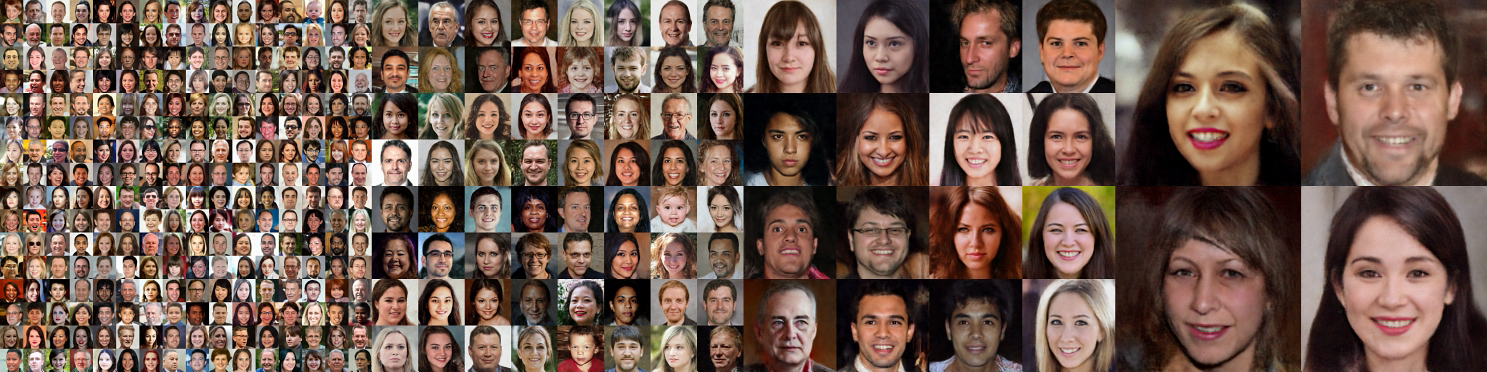}
    \caption{1-step functional generation for images on FFHQ~\cite{karras2019style} dataset. From left to right, we show images generated at 64, 128, 256 and 512 resolution respectively.}
    \label{fig:ffhjq}
\end{figure}

\begin{table}[t]
\centering
\small
\setlength{\tabcolsep}{2pt}
\renewcommand{\arraystretch}{0.9}
\caption{
Evaluation of FID$_\text{CLIP}$ \cite{kynkaanniemi2023role} against previous infinite-dimensional approaches trained on coordinate subsets.
The best results for the 1-step and multi-step settings are highlighted in bold. $^*$ indicates missing entries, where FID scores are reported instead of FID$_{\text{CLIP}}$.
}
\label{tab:unconditional_comparison}
{%
\begin{tabular}{lccccc}
\toprule
\textbf{Method} & \textbf{Step} 
& \textbf{CelebAHQ-64} 
& \textbf{CelebAHQ-128} 
& \textbf{FFHQ-256} \\
\midrule

D2F \cite{dupont2022data}
& 1 & 40.4$^*$ & -- & -- \\

GEM \cite{du2021learning}
& 1 & 14.65 & 23.73 & 35.62 \\

GASP \cite{dupont2022generative}
& 1 & 9.29 & 27.31 & 24.37 \\

EMF (Ours)
& 1 
& \textbf{4.32} 
& \textbf{8.86} 
& \textbf{15.0} 
\\

$\infty$-Diff \cite{bond2024infty}
& 100 
& \textbf{4.57} 
& \textbf{3.02} 
& \textbf{3.87} 
\\

DPF \cite{zhuang2023diffusion}
& 1000 
& 13.21$^*$ & -- & -- \\

\bottomrule
\end{tabular}%
}
\end{table}

\paragraph{Result} For 2D functional image generation, we present qualitative results in Figure~\ref{fig:ffhjq} for 1-step unconditional generation on FFHQ, and in Figure~\ref{fig:celeb_functional} for 1-step unconditional generation on CelebA-HQ. Following Infty-Diff, we use FID$_\text{CLIP}$ \cite{kynkaanniemi2023role} metric to assess function-based generative methods. Because our model generates a continuous function that represents an image, the output is resolution-agnostic. We therefore visualize samples at multiple resolutions, ranging from $64$ to $512$ on FFHQ and from $64$ to $1024$ on CelebA-HQ. For quantitative evaluation, Table~\ref{tab:unconditional_comparison} compares our method against prior approaches; despite using a single sampling step, our results are comparable to multi-step methods such as $\infty$-Diff.
\newpage
\subsection{SDF Generation} \label{sec:sdf_generation}

\begin{figure*}[t]
    \centering
    \includegraphics[width=0.99\linewidth]{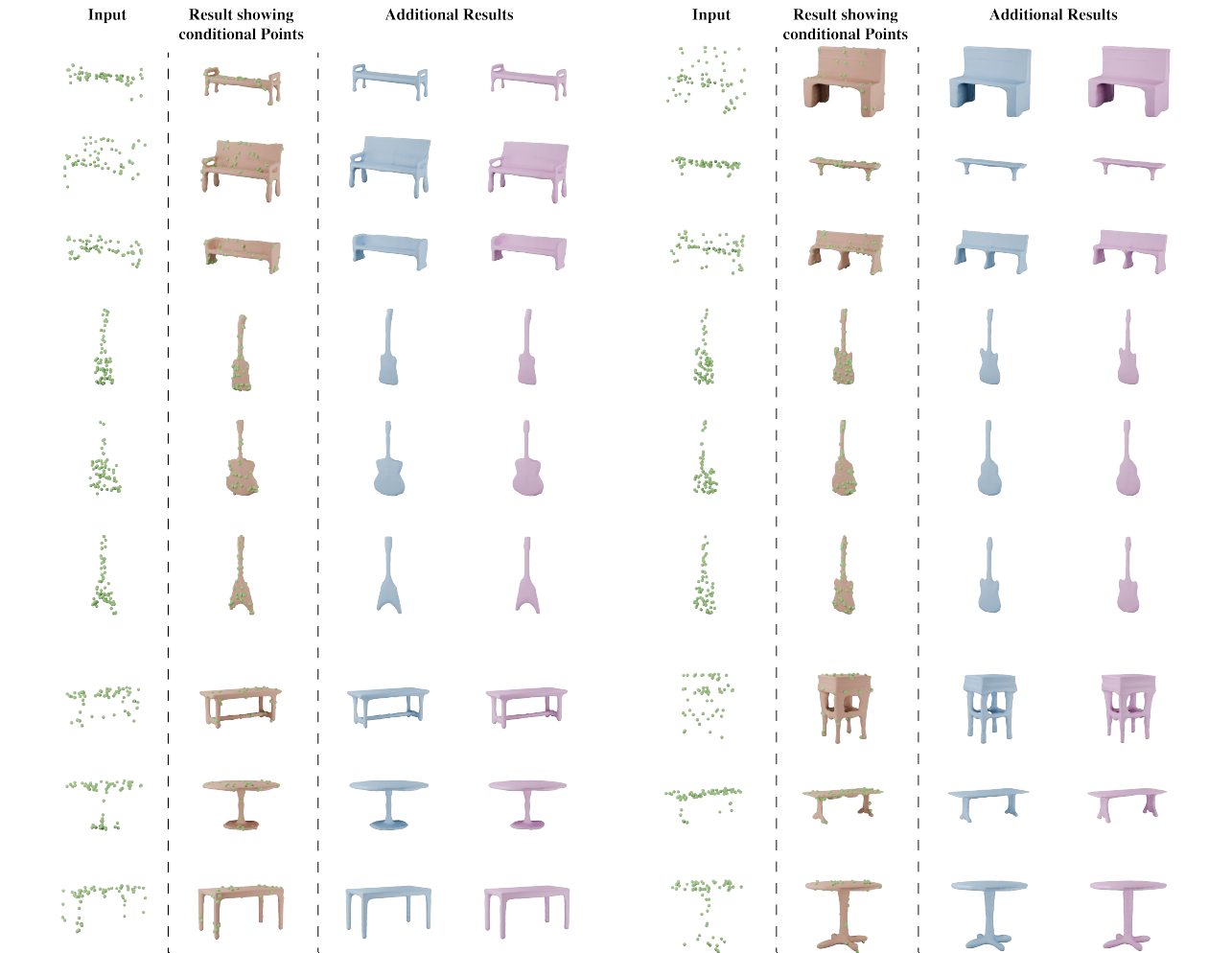}
    \caption{We show 1-step functional SDF generation results. The leftmost column visualizes the conditioning points, overlaid on the first generated mesh shown in the second column. Additional generated samples are presented in the remaining columns.}
    \label{fig:sdf_part1}
\end{figure*}

\begin{figure*}[t]
    \centering
    \includegraphics[width=0.99\linewidth]{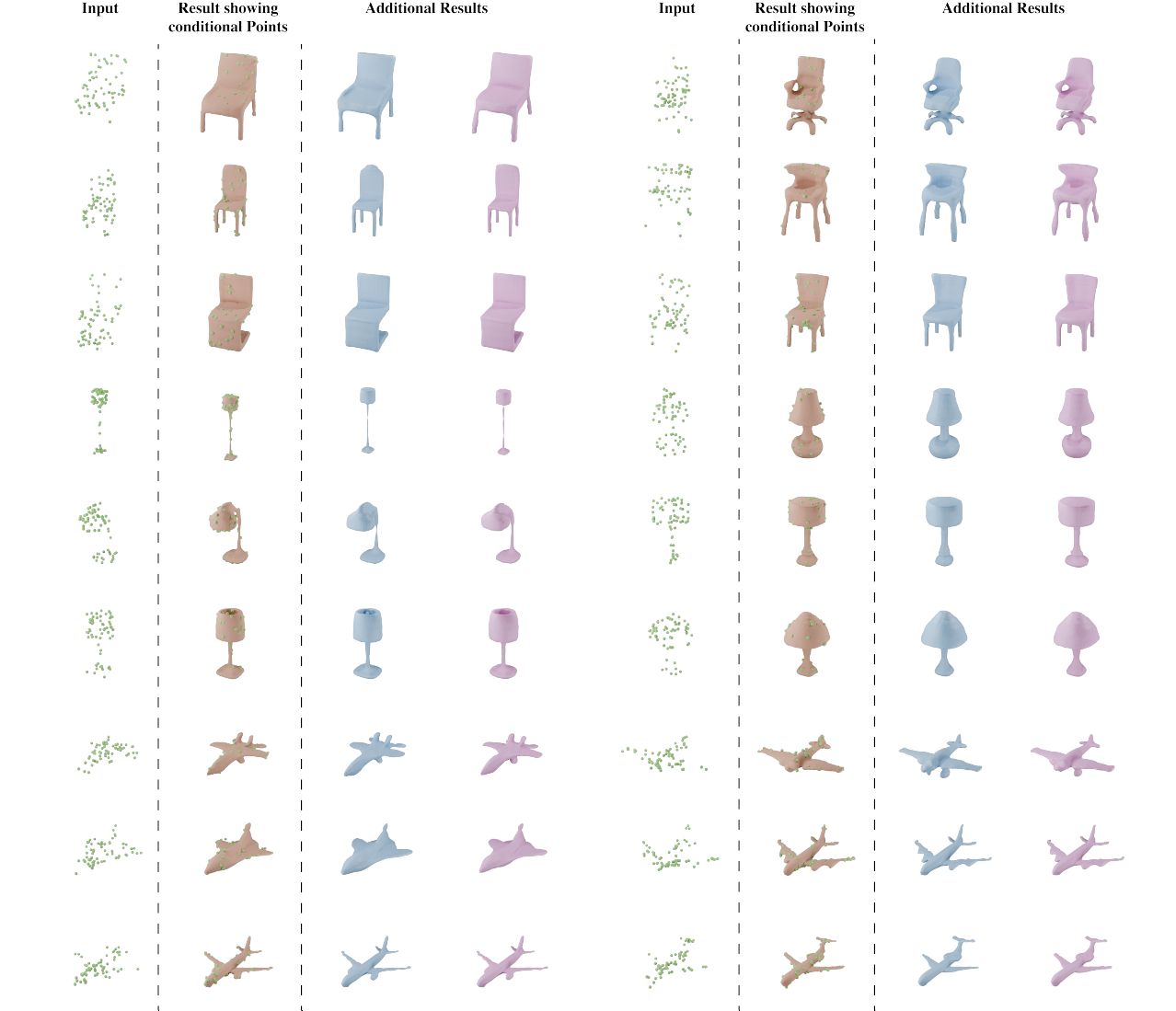}
    \caption{We show 1-step functional SDF generation results. The leftmost column visualizes the conditioning points, overlaid on the first generated mesh shown in the second column. Additional generated samples are presented in the remaining columns.}
    \label{fig:sdf_part2}
\end{figure*}


\paragraph{Model} 
\begin{wrapfigure}{r}{0.36\textwidth}
  \centering
  \vspace{-0.8em}
  \includegraphics[width=\linewidth]{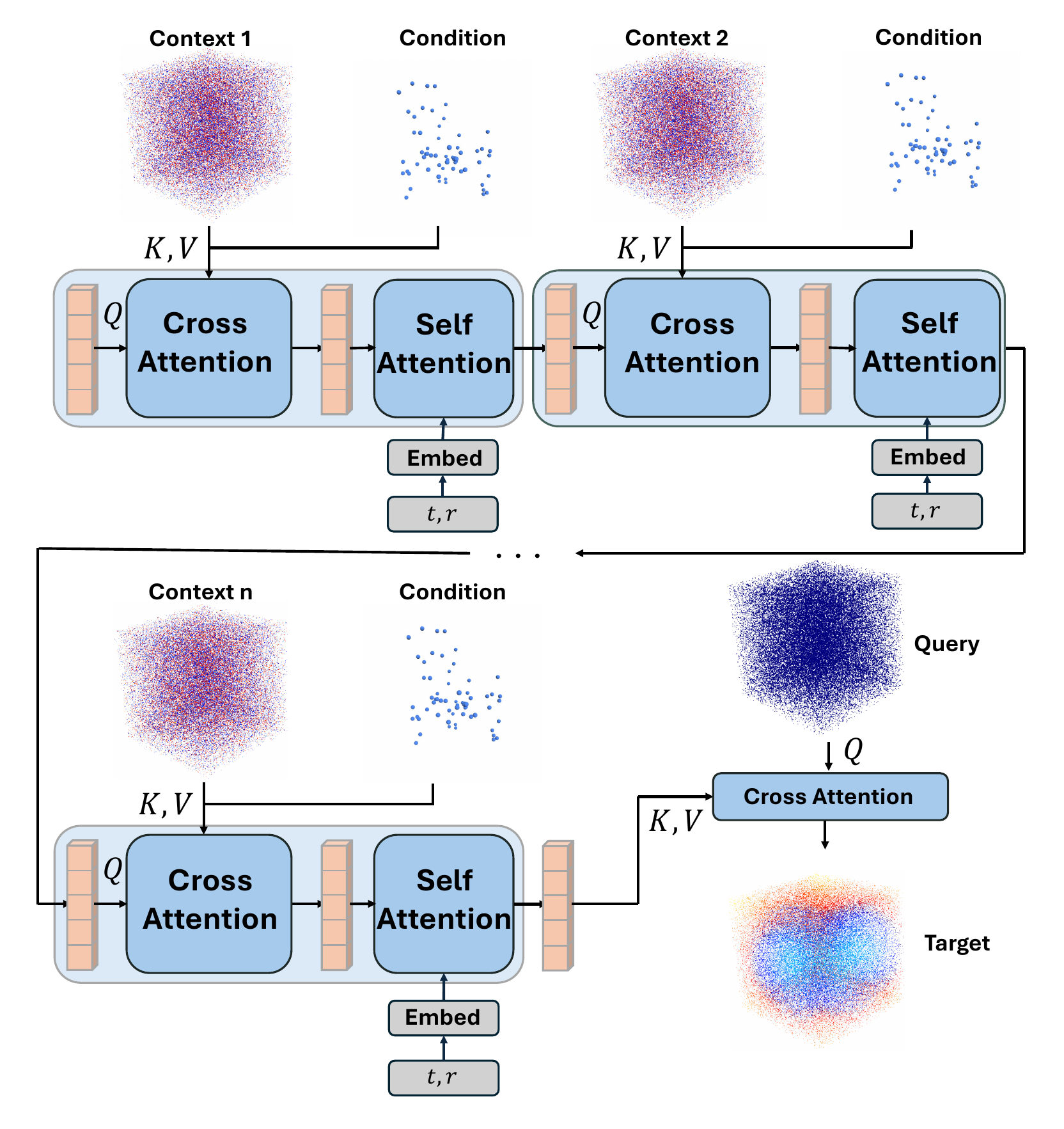}
  \caption{The network design of Functional Diffusion.}
  \label{fig:functional-diffusion}
  \vspace{-0.8em}
\end{wrapfigure}
We adopt the Functional Diffusion architecture~\cite{zhang2024functional} for signed distance field (SDF) generation. A SDF represents a shape as a continuous scalar function whose value at each spatial location equals the signed distance to the closest surface, with the sign indicating whether the point lies inside or outside the shape. Both inputs and outputs of the model are specified by randomly sampled points and their corresponding function values, rather than fixed grids. 
Concretely, the input function $f_c$ is given by a set of context points $\{x_c^i\}_{i=1}^n$ with values $v_c^i = f_c(x_c^i)$, while the output function $f_q$ is queried at locations $\{x_q^j\}_{j=1}^m$ to produce values $\{v_q^j\}_{j=1}^m$. This formulation naturally supports mismatched context and query sets, enabling flexible functional mappings. Following~\cite{zhang2024functional}, the context set is evenly divided into $d$ disjoint groups. 
As shown in \autoref{fig:functional-diffusion}, each group is processed sequentially by an attention block composed of cross-attention followed by self-attention. The cross-attention uses a latent vector to aggregate information from each context group, where the latent is initialized as a learnable variable representing the underlying function and is propagated across blocks. Context points are embedded by combining Fourier positional encodings of spatial coordinates with embeddings of function values, and further concatenated with conditional embeddings. In our experiments, conditioning is provided by 64 partially observed surface points.

\paragraph{Dataset}
We follow the surface reconstruction setting of Functional Diffusion~\cite{zhang2024functional}, where the model reconstructs a complete surface from 64 observed points sampled on a target shape. The generative process is conditioned on these surface points and predicts the full SDF starting from noise. All experiments are conducted on the ShapeNet-CoreV2 dataset~\cite{chang2015shapenet}, which contains approximately 57000 3D models spanning 55 object categories. Using the same preprocessing pipeline as prior  work~\cite{zhang2024functional,zhang20233dshape2vecset,zhang20223dilg}, each mesh is converted into an SDF defined over the domain $[0,1]^3$. For each shape, we uniformly sample $n=49152$ points to form the context set and their SDF values, and independently sample $m=2048$ points as query locations with corresponding SDF values. In addition, a separate set of 64 surface points near the zero-level set is sampled and used as conditional input.

\paragraph{Metrics}
We evaluate reconstructed SDF quality using Chamfer Distance, F-score, and Boundary Loss, following prior work~\cite{zhang2024functional,zhang20233dshape2vecset,zhang20223dilg}. Chamfer Distance (CD) and F-score are computed by uniformly sampling 50K points from each reconstructed surface. F-Score evaluates surface reconstruction quality by measuring the precision–recall trade-off between generated and ground-truth surface points under a fixed distance threshold. It quantifies how well the predicted surface aligns with the true surface by penalizing both missing regions and spurious geometry. Boundary Loss measures SDF accuracy near the surface boundary and is defined as $\text{Boundary}(f) = \frac{1}{|\mathcal{E}_\Omega|} \sum_{i \in \mathcal{E}_\Omega}| f(\mathbf{x}_i) - q(\mathbf{x}_i) |^2$,
where $\mathcal{E}_\Omega$ denotes points sampled near the zero-level set, $f$ is the predicted SDF, and $q$ is the ground-truth SDF. This metric is computed using 100K boundary samples. We use the same train/test split as~\cite{zhang2024functional} for our experiment.

\begin{table}[h]
\centering
\caption{Quantitative comparison of reconstruction quality. The model is trained on the ShapeNet dataset, where the conditional input consists of 64 points sampled from the target surface. The model is required to reconstruct the surface based on these 64 points. Step denotes the number of inference steps.}
\label{tab:reconstruction_metrics}

{\footnotesize
\begin{tabular}{lcccc}
\toprule
Method & Step & Chamfer $\downarrow$ & F-Score $\uparrow$ & Boundary $\downarrow$ \\
\midrule
3DS2VS \cite{zhang20233dshape2vecset} & 18 & 0.144 & 0.608 & 0.016 \\
FD \cite{zhang2024functional} & 64 & 0.101 & 0.707 & 0.012 \\
MF \cite{li2025functional} & 1 & 0.060 & 0.584 & 0.011 \\
EMF (Ours) & 1 & 0.046 & 0.674 & 0.011 \\
\bottomrule
\end{tabular}

\vspace{1mm}
\raggedleft
($\downarrow$ lower is better; $\uparrow$ higher is better.)
}
\end{table}

\paragraph{Result}
For SDF generation, we train our model on ShapeNet using only 64 surface points as conditioning input. This sparse-conditioning setting is challenging, particularly for single-step generation. Qualitative results for 1-step conditional generation are shown in Figures~\ref{fig:sdf_part1} and~\ref{fig:sdf_part2}. For quantitative evaluation, Table~\ref{tab:reconstruction_metrics} reports 3D reconstruction metrics; our method achieves quality comparable to multi-step approaches and consistently outperforms the original mean flow method.

\subsection{Point Cloud Generation}\label{sec:point_cloud}

\begin{figure*}[t]
    \centering
    \includegraphics[width=0.99\linewidth]{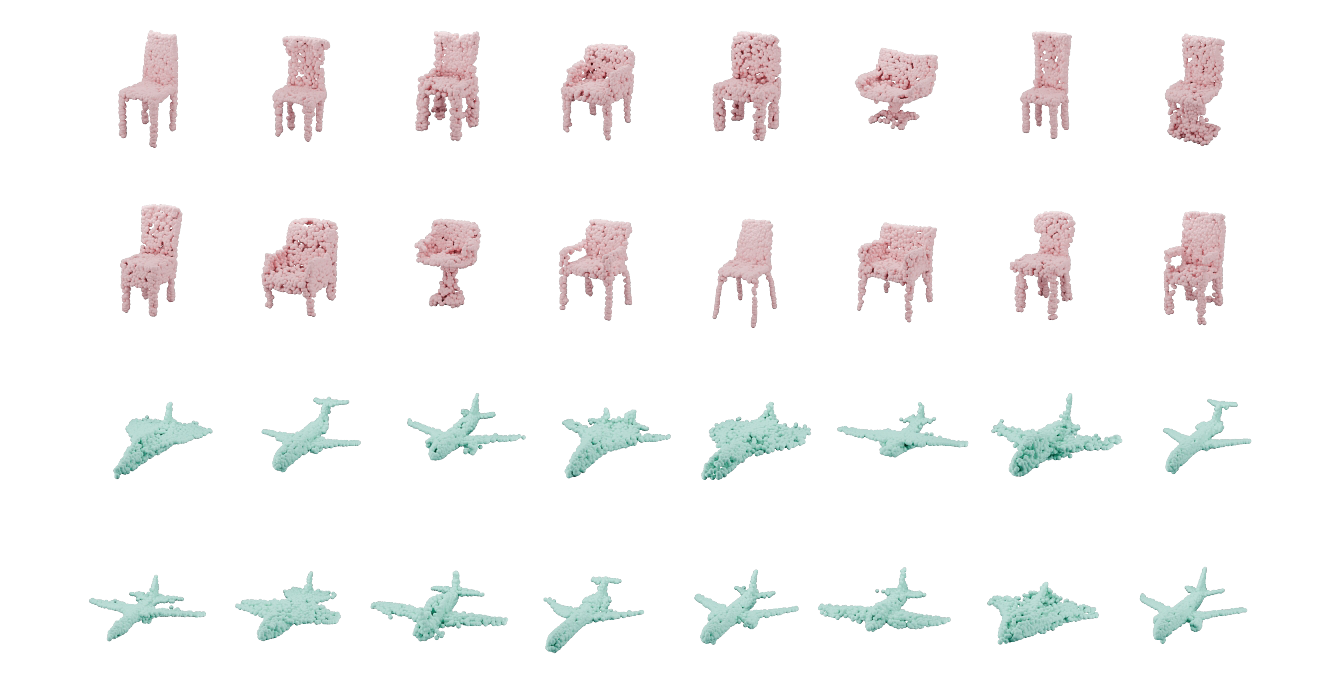}
    \caption{1-step point cloud generation result trained on ShapeNet\cite{chang2015shapenet} dataset.}
    \label{fig:pc_efm}
\end{figure*}

\paragraph{Model} We adopt the Latent Point Diffusion Model (LION) architecture~\cite{vahdat2022lion} for point cloud generation, which performs generative modeling in a structured latent space derived from point clouds. As shown \autoref{fig:lion} in The model builds upon a variational autoencoder that encodes each shape into a hierarchical latent representation consisting of a global shape latent and a point-structured latent point cloud, capturing coarse structure and fine-grained geometry, respectively. 
\begin{wrapfigure}{r}{0.62\textwidth}
  \centering
  \vspace{-0.9em}
  \includegraphics[width=\linewidth]{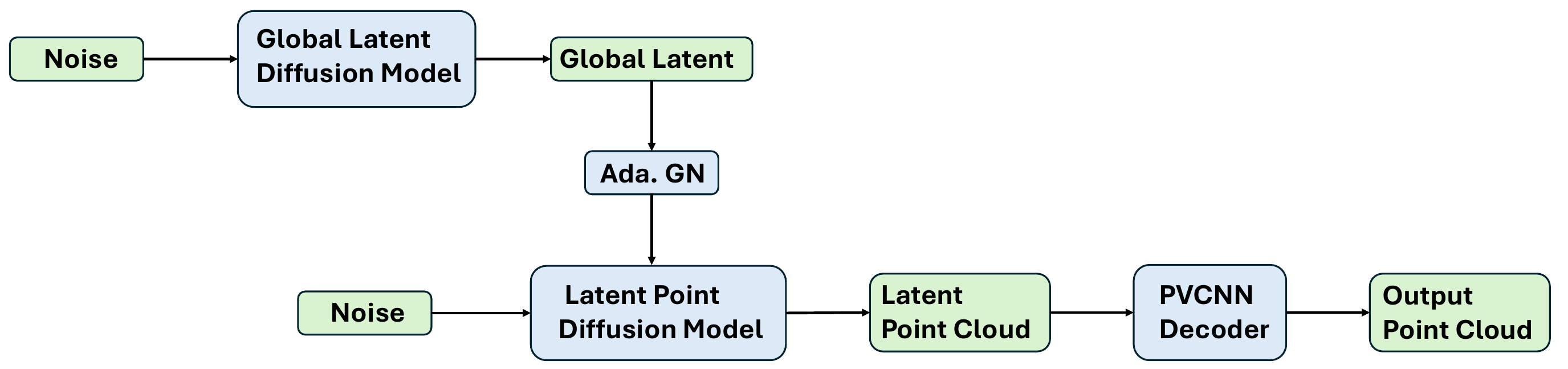}
  \caption{The network design of LION.}
  \label{fig:lion}
  \vspace{-0.9em}
\end{wrapfigure}
The encoder, decoder, and latent point diffusion modules are implemented with Point-Voxel CNNs (PVCNNs)~\cite{liu2019point}, following the design of~\cite{zhou20213d}. The global latent diffusion model is parameterized by a ResNet-style network composed of fully connected layers, implemented as $1\times1$ convolutions. Conditioning on the global latent is injected into the PVCNN layers to generate point-structured latent point cloud through adaptive Group Normalization. For modeling the point-structured latent representations, we further adopt a modified DiT-3D backbone based on~\cite{wang2025pdt}, which provides stronger modeling capacity and improved scalability. Finally, the decoder maps the generated latent representation back to the 3D space, yielding the output point cloud.
\paragraph{Dataset}
We conduct experiments on the ShapeNet dataset~\cite{chang2015shapenet} using the preprocessing and data splits provided by PointFlow~\cite{yang2019pointflow}. We focus our evaluation on two object categories: airplanes and chairs. Each shape in the processed dataset contains 15,000 points, from which 2,048 points are randomly sampled at every training iteration. The training set includes 2,832 airplane shapes and 4,612 chair shapes. For evaluation, we report sample quality metrics against the corresponding reference sets, which comprise 405 for airplanes and 662 for chairs. Following PointFlow, all shapes are normalized using a global normalization scheme, where the mean is computed per axis over the entire training set and a single standard deviation is applied across all axes.

\paragraph{Metrics} To assess the performance of point cloud generative models at the distribution level, we compare a generated set $S_g$ with a reference set $S_r$ with Coverage (COV) and 1-Nearest Neighbor Accuracy (1-NNA), both of which rely on a pairwise distance defined between point clouds.

Coverage (COV) measures the extent to which the generated samples span the variability of the reference distribution. Specifically, each reference shape is associated with its closest counterpart in the generated set, and COV is defined as the fraction of generated shapes that are selected as nearest neighbors by at least one reference shape. As a result, COV primarily reflects sample diversity and sensitivity to mode collapse, while being largely agnostic to the fidelity of individual generated point clouds.

\begin{equation}
\mathrm{COV}(S_g, S_r)
=
\frac{\left| \left\{ \arg\min_{Y \in S_r} D(X, Y) \,\middle|\, X \in S_g \right\} \right|}
{\,|S_r|\,}
\end{equation}

1-Nearest Neighbor Accuracy (1-NNA) evaluates how well the generated and reference distributions are aligned. This metric treats the union of $S_g$ and $S_r$ as a labeled dataset and computes the leave-one-out accuracy of a 1-NN classifier, where each sample is assigned the label of its nearest neighbor. 

\begin{equation}
\mathrm{1\text{-}NNA}(S_g, S_r)
=
\frac{
\sum_{X \in S_g} \mathbf{1}\!\left[ N_X \in S_g \right]
+
\sum_{Y \in S_r} \mathbf{1}\!\left[ N_Y \in S_r \right]
}
{|S_g| + |S_r|},
\end{equation}
where $N_X$ (resp., $N_Y$) denotes the nearest neighbor of $X$ (resp., $Y$) in $(S_g \cup S_r) \setminus \{X\}$.

For both COV and 1-NNA, nearest neighbors are determined using either the Chamfer Distance (CD) or the Earth Mover’s Distance (EMD). CD evaluates mutual proximity by aggregating point-to-set nearest-neighbor distances in both directions, while EMD computes the minimal transport cost between two point clouds by enforcing a one-to-one correspondence. CD and EMD are defined as:
\begin{equation}
\mathrm{CD}(X, Y)
=
\sum_{x \in X} \min_{y \in Y} \| x - y \|_2^2
+
\sum_{y \in Y} \min_{x \in X} \| x - y \|_2^2,
\end{equation}
\begin{equation}
    \mathrm{EMD}(X, Y)
=
\min_{\gamma : X \to Y}
\sum_{x \in X} \| x - \gamma(x) \|_2,
\end{equation}
where $X$ and $Y$ denote two point clouds with the same cardinality, $\|\cdot\|_2$ is the Euclidean norm, and $\gamma$ is a bijection between points in $X$ and $Y$.

\begin{table}[t]
\centering
\caption{
Unconditional generation results on the Airplane and Chair categories at a resolution of 2048 points.
We report one-nearest neighbor accuracy (1-NNA) and coverage (COV) under Chamfer Distance (CD) and Earth Mover's Distance (EMD).
For 1-NNA, lower is better ($\downarrow$), while for COV, higher is better ($\uparrow$).  Bold and underlined numbers indicate the best and second-best performance for each metric under the one-step and multi-step settings, respectively.  Global normalization is applied to both training and test sets following LION \cite{vahdat2022lion}.
}
\label{tab:shape_generation_2048}
\begin{tabular}{l c cccc cccc}
\toprule
Method & Steps
& \multicolumn{4}{c}{Airplane}
& \multicolumn{4}{c}{Chair} \\

\cmidrule(lr){3-6} \cmidrule(lr){7-10}

&
& \multicolumn{2}{c}{1-NNA$\downarrow$}
& \multicolumn{2}{c}{COV$\uparrow$}
& \multicolumn{2}{c}{1-NNA$\downarrow$}
& \multicolumn{2}{c}{COV$\uparrow$} \\

\cmidrule(lr){3-4} \cmidrule(lr){5-6}
\cmidrule(lr){7-8} \cmidrule(lr){9-10}

&
& CD & EMD & CD & EMD
& CD & EMD & CD & EMD \\

\midrule
MFM-point~\cite{molodyk2025mfm} & 1400 
& 65.36 & \textbf{57.21} & -- & -- 
& 54.92 & 53.25 & -- & -- \\

LION~\cite{vahdat2022lion} & 1000 
& 67.41 & 61.23 & 47.16 & 49.63 
& 53.70 & \underline{52.34} & 48.94 & 52.11 \\

FrePoLat~\cite{zhou2024frepolad} & 1000 
& \underline{65.25} & 62.10 & 45.16 & 47.80 
& \underline{52.35} & 53.23 & 50.28 & 50.93 \\

NSOT~\cite{hui2025not} & 1000 
& 68.64 & 61.85 & -- & -- 
& 55.51 & 57.63 & -- & -- \\

DiT-3D~\cite{mo2023dit} & 1000 
& \textbf{62.35} & \underline{58.67} &  53.16 &  54.39 
& \textbf{49.11} & \textbf{50.73} & 50.00 & 56.38 \\

PVD~\cite{zhou20213d} & 1000 
& 73.82 & 64.81 & 48.88 & 52.09  
& 56.26 & 53.32 & 49.84 & 50.60 \\

PVD-DDIM~\cite{zhou20213d} & 100 
& 76.21 & 69.84 & 44.23 & 49.75 
& 61.54 & 57.73 & 46.32 & 48.19 \\

DPM~\cite{luo2021diffusion} & 100 
& 76.42 & 86.91 & 48.64 & 33.83 
& 60.05 & 74.77 & 44.86 & 35.50 \\

ShapeGF~\cite{cai2020learning} & 10 
& 80.00 & 76.17 & 45.19 & 40.25 
& 68.96 & 65.48 & 48.34 & 44.26 \\

PSF~\cite{wu2023fast} & 1 
& \textbf{71.11} & \textbf{61.09} & 46.17 & 52.59 
& \underline{58.92} & \underline{54.45} & 46.71 & \underline{49.84} \\

r-GAN~\cite{achlioptas2018learning} & 1 
& 98.40 & 96.79 & 30.12 & 14.32 
& 83.69 & 99.70 & 24.27 & 15.13 \\

1-GAN (CD)~\cite{achlioptas2018learning} & 1 
& 87.30 & 93.95 & 38.52 & 21.23 
& 68.58 & 83.84 & 41.99 & 29.31 \\

1-GAN (EMD)~\cite{achlioptas2018learning} & 1 
& 89.49 & 76.91 & 38.27 & 38.52 
& 71.90 & 64.65 &  38.07 & 44.86 \\

PointFlow~\cite{yang2019pointflow} & 1 
& 75.68 & 70.74 & 47.90 & 46.41 
& 62.84 & 60.57 &  42.90 & 50.00 \\

DPF-Net~\cite{klokov2020discrete} & 1 
& 75.18 & 65.55 & 46.17 & 48.89 
& 62.00 & 58.53 & 44.71 &  48.79 \\

SoftFlow~\cite{kim2020softflow} & 1 
& 76.05 & 65.80 & 46.91 & 47.90 
& 59.21 & 60.05 & 41.39 & 47.43 \\

SetVAE~\cite{kim2021setvae} & 1 
& 75.31 & 77.65 & 43.70 & 48.40 
& 58.76 & 61.48 & \underline{46.83} & 44.26 \\ 

EMF (ours) & 1 
& \underline{72.84} & \underline{62.72} & \textbf{50.37} & \textbf{55.56} & \textbf{56.42} & \textbf{54.08}& \textbf{47.89}  & \textbf{52.87}    \\

\bottomrule
\end{tabular}
\end{table}

\paragraph{Result} For point cloud generation, we present 1-step unconditional samples for two ShapeNet categories in Figure~\ref{fig:pc_efm}. All models are trained on ShapeNet using the LION architecture. Quantitative results are reported in Table~\ref{tab:shape_generation_2048}, where our method achieves the best generation quality compared to prior approaches.

\section{Additional Results\&Experiments }\label{sec:additional_experiment}
\subsection{Ablation Study: Rationale for the Second Local Linear Approximation} \label{sec:rational_second_lla}
In \autoref{eq:EFM_equation}, during the derivation, we apply local linear approximation to the term $u_{t \to t+\Delta t}$ at two different places.  The first approximation appears in an independent summation term in \autoref{eq:EFM_equation}, where $u_{t \to t+\Delta t}$ is approximated by $u_{t \to t}$. The motivation of this approximation is straightforward: by reducing it to the instantaneous velocity $u_{t \to t}(x)$, we can further replace it with the conditional instantaneous velocity $u_t(x \mid x_1)$, thereby incorporating explicit supervision from the dataset.  

The second approximation is applied to the update $x_{t+\Delta t} = \Delta tu_{t \to t+\Delta t}(x_t) + x_t$,  where $u_{t \to t+\Delta t}$ is again approximated by $u_{t \to t}$. This approximation is primarily introduced for memory efficiency.  This design choice is particularly important for conditional generation. During training, conditional MeanFlow employs CFG, which replaces $u(x \mid x_1)$ with $w u(x \mid x_1) + (1 - w - k) u^\theta_{t \to t}(x_t, C_0) + k u^\theta_{t \to t}(x_t, C)$, where $C$ denotes the label of $x_1$ and $C_0$ is the null label. In this formulation, the term $u^\theta_{t \to t}(x_t, C)$ can be directly reused in the computation of $x_{t+\Delta t} \approx \Delta t u_{t \to t}(x_t) + x_t$,  which helps reduce memory consumption.  For unconditional generation, although the computation of $x_{t+\Delta t}$ requires two stop-gradient forward passes and one trainable forward pass regardless of whether $u_{t \to t+\Delta t}$ is approximated, we empirically observe that using the exact $u_{t \to t+\Delta t}$ does not improve generation quality, and moreover it prevents the use of the multi-head technique described in \autoref{fig:aux_model}, leading to increased memory usage and computational cost. Quantitative results are reported in \autoref{tab:memory_time_comparison2}.

\begin{table}[t]
\centering
\caption{Comparison of memory and computational cost between our method and MeanFlow for unconditional (CelebA-HQ) and conditional (ImageNet-1000) generation using the DiT-B/2 model. ``Peak'' denotes the maximum GPU memory usage during training, ``Fixed'' refers to the constant memory overhead, and aux-EMF indicates our method with a 4-block auxiliary head.  All experiments are conducted on a single H200 GPU with batch sizes of 64 for CelebA-HQ and 128 for ImageNet-1000, using EMA and AdamW optimization with mixed-precision (FP16) training in PyTorch.}
\label{tab:memory_time_comparison2}
\begin{tabular}{lccccc}
\hline
\textbf{Method} & \textbf{Dataset} & \textbf{Peak Memory} & \textbf{Fixed Memory} & \textbf{Speed / Iter} & \textbf{FID} \\
\hline
MeanFlow & CelebA-HQ & 32.1GB & 2.3GB  & 151.4ms & 12.4  \\
EMF (compute $u_{t\to t+\Delta t}$)                 & CelebA-HQ & 23.3GB  & 2.3GB & 91.74ms  & 11.2 \\
EMF                  & CelebA-HQ & 23.3GB & 2.3GB & 91.2 ms & \textbf{10.9} \\
aux-EMF              & CelebA-HQ & \textbf{17.6GB}  & 2.8 GB & \textbf{84.2 ms} & 11.7 \\
MeanFlow & ImageNet &101.9GB  & 2.4GB & 400.9ms & 11.1 \\
EMF (compute $u_{t\to t+\Delta t}$)              & ImageNet & 71.7GB & 2.4GB  & 232.6ms & - \\
EMF              & ImageNet & \textbf{57.9GB} & 2.4GB  &  \textbf{198.8ms} & \textbf{7.2} \\
\hline
\end{tabular}
\end{table}

\begin{figure}[t]
  \centering
  \subfigure[1 step]{\includegraphics[width=0.32\textwidth]{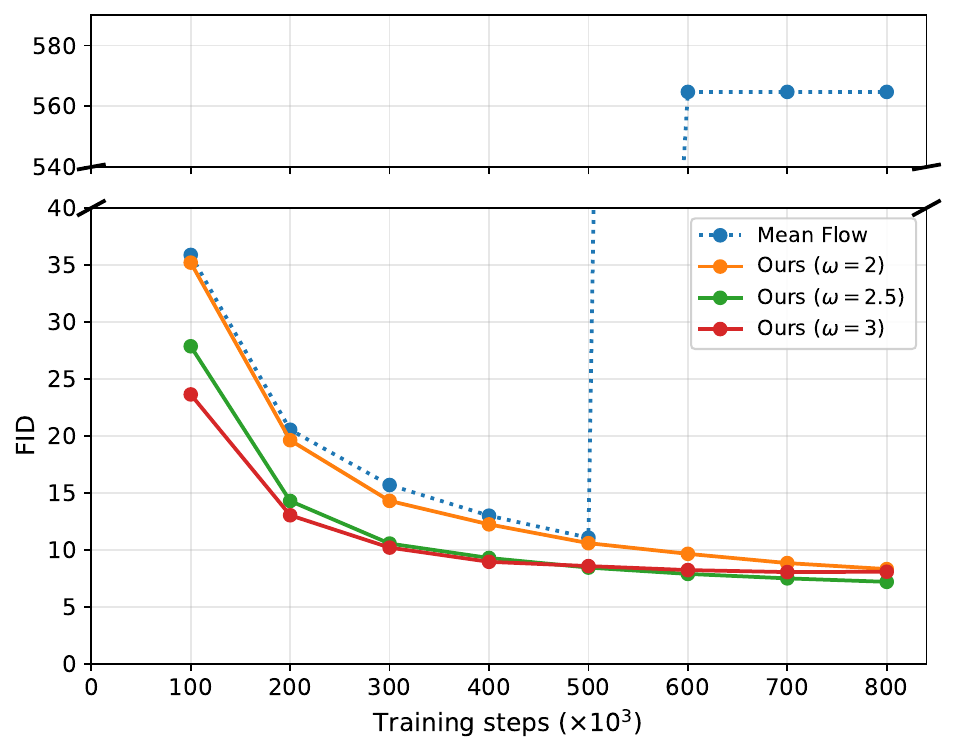}}
  \hfill
  \subfigure[2 step]{\includegraphics[width=0.32\textwidth]{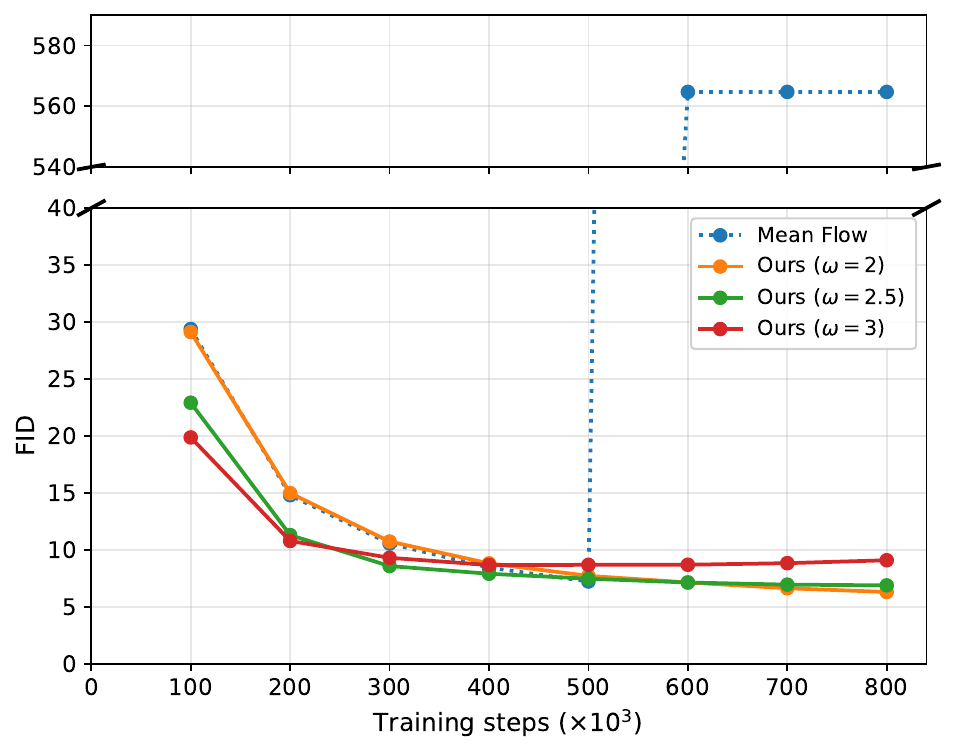}}
  \hfill
  \subfigure[4 step]{\includegraphics[width=0.32\textwidth]{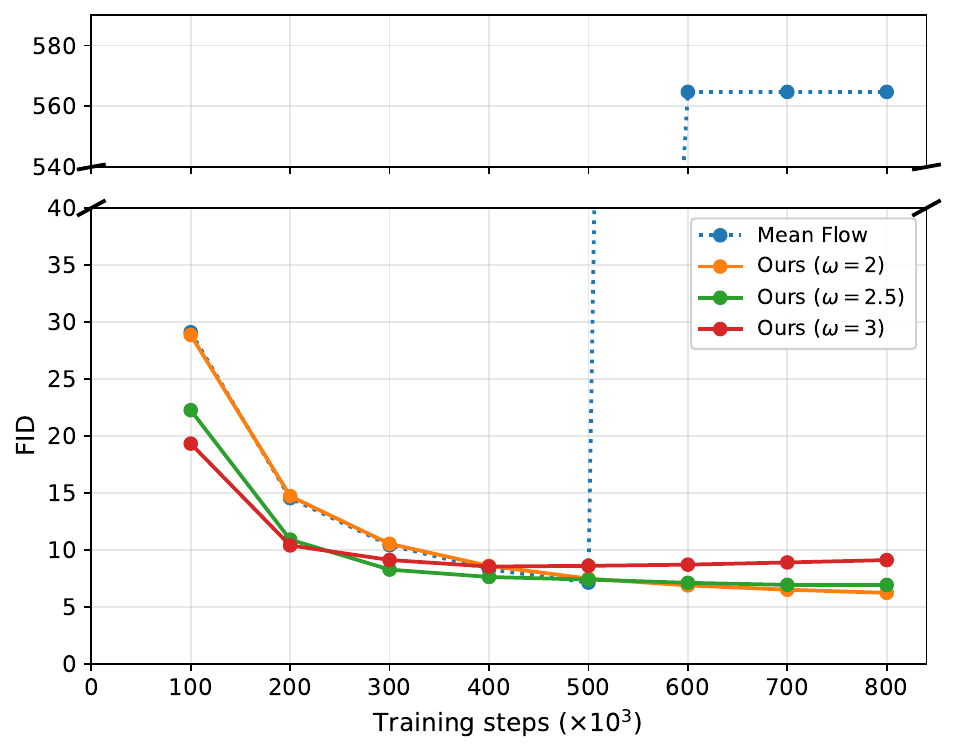}}
  \caption{Training dynamics of our method and MeanFlow on ImageNet under 1-step, 2-step, and 4-step generation. For our method, we compare different classifier-free guidance (CFG) scales, while MeanFlow uses a fixed CFG scale of 2.0, which is the best-performing setting reported in \cite{geng2025mean}.}
  \label{fig:training_dynamics_imageNet}
\end{figure}

\begin{figure}
  \centering
  \subfigure[$M_g$]{\includegraphics[width=0.32\textwidth]{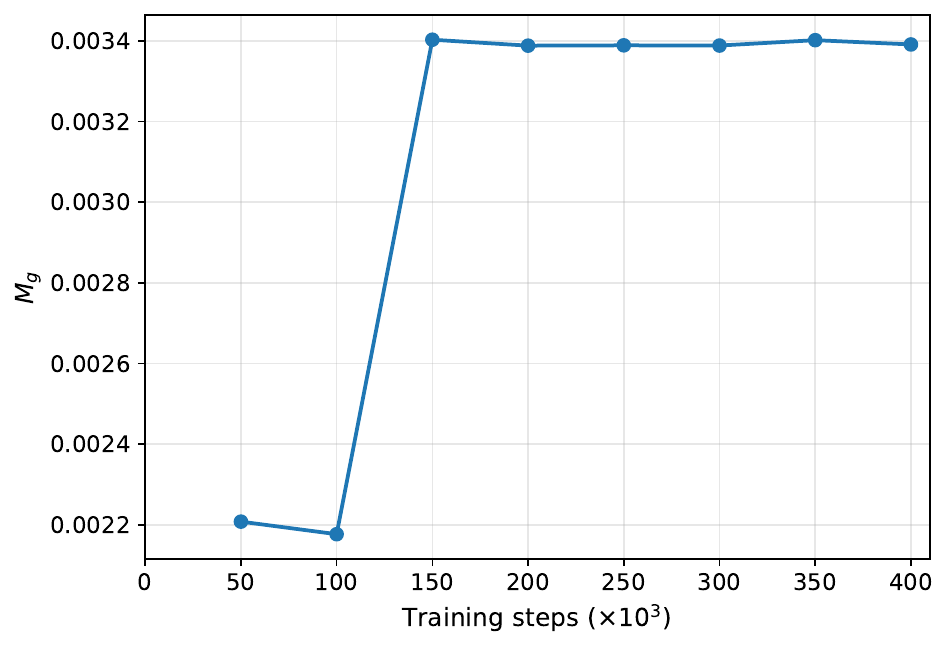}}
  \hfill
  \subfigure[$M_x$]{\includegraphics[width=0.32\textwidth]{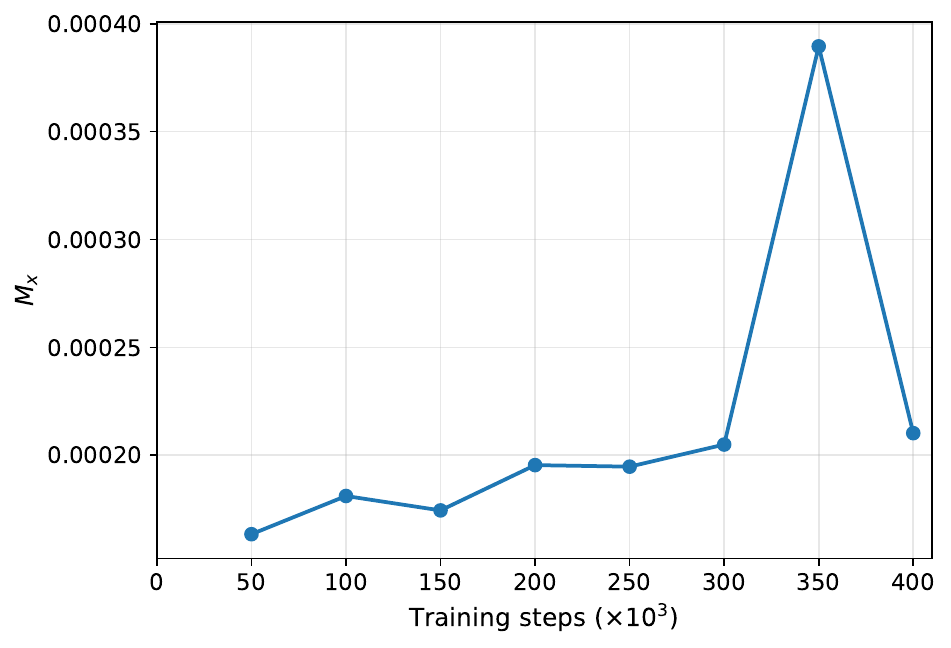}}
  \hfill
  \subfigure[$M_t$]{\includegraphics[width=0.32\textwidth]{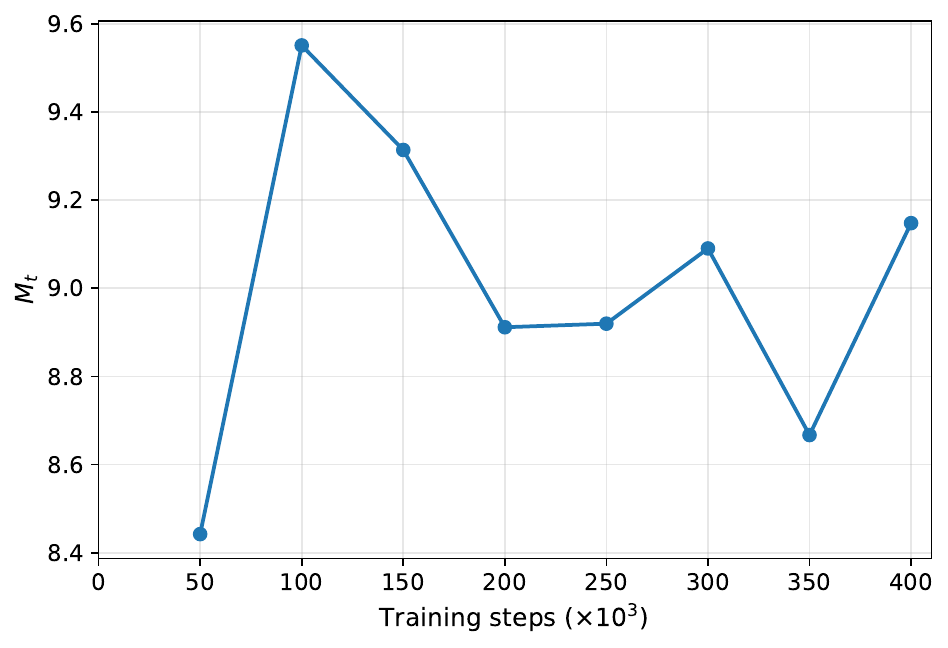}}
  \caption{Evolution of the parameters $M_g$, $M_t$, and $M_x$ in \autoref{assupmption} over the course of training.} \label{assumption_valid}
\end{figure}
\begin{figure}
  \centering
  \subfigure[$M_g'$]{\includegraphics[width=0.32\textwidth]{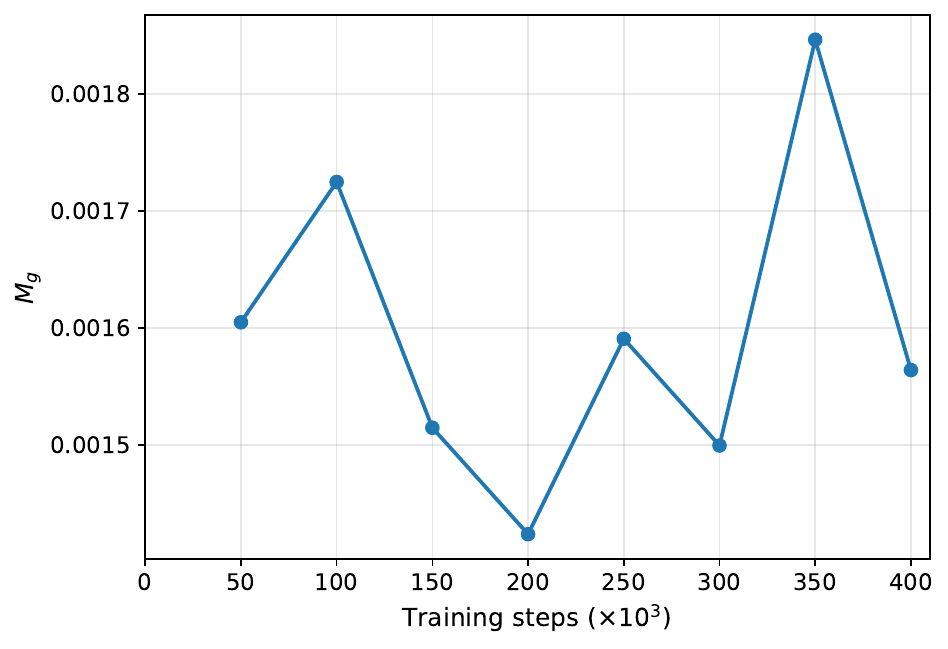}}
  \hfill
  \subfigure[$M_x'$]{\includegraphics[width=0.32\textwidth]{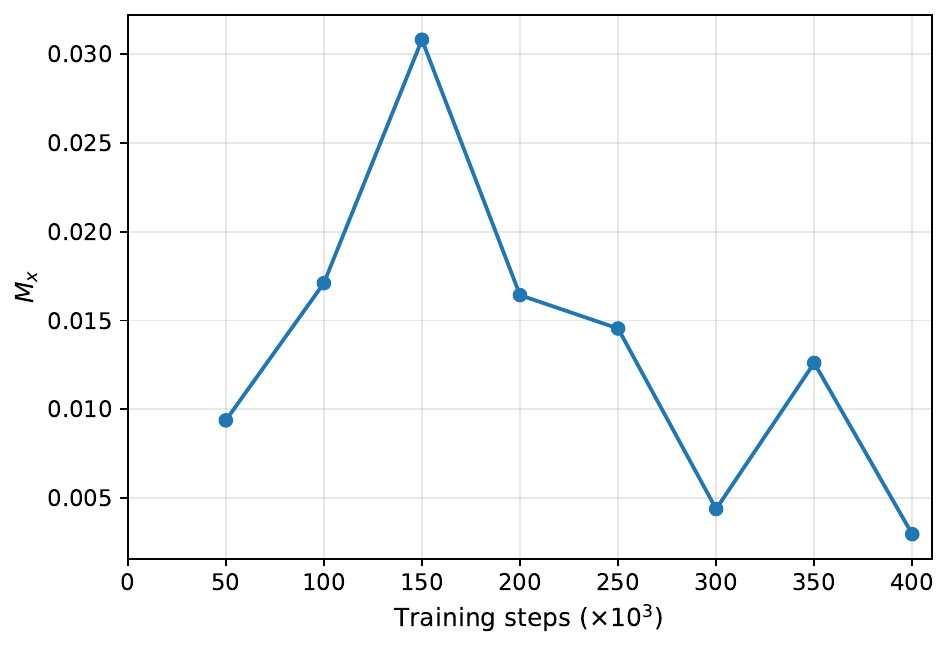}}
  \hfill
  \subfigure[$M_t'$]{\includegraphics[width=0.32\textwidth]{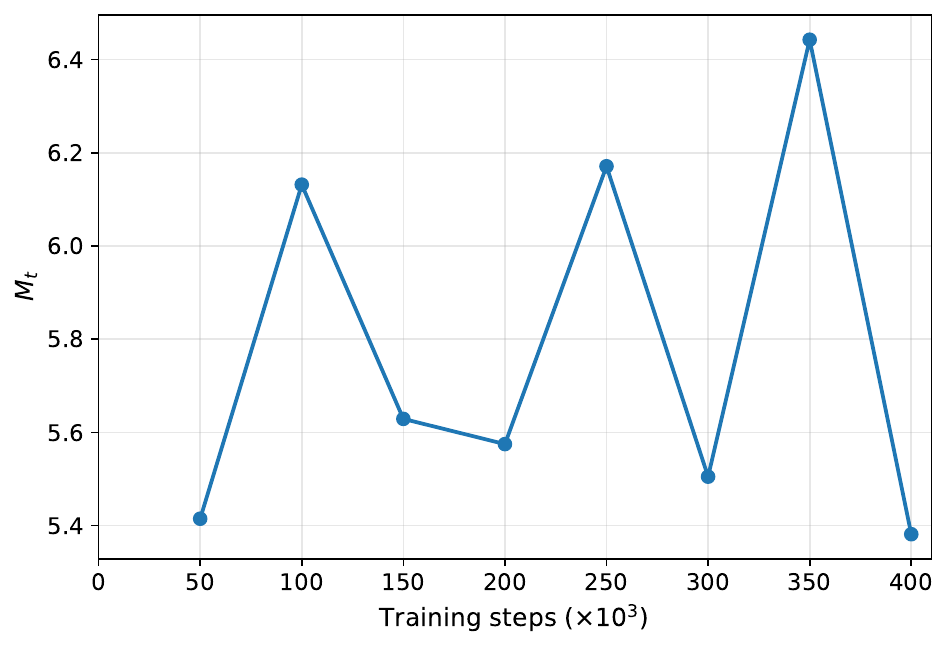}}
  \caption{Evolution of the parameters $M'_g$, $M'_t$, and $M'_x$ in \autoref{assupmption_x1} over the course of training.}\label{assumption_x1_valid}
\end{figure}

\begin{figure}
  \centering
  \subfigure[Loss of $u$-prediction]{\includegraphics[width=0.24\textwidth]{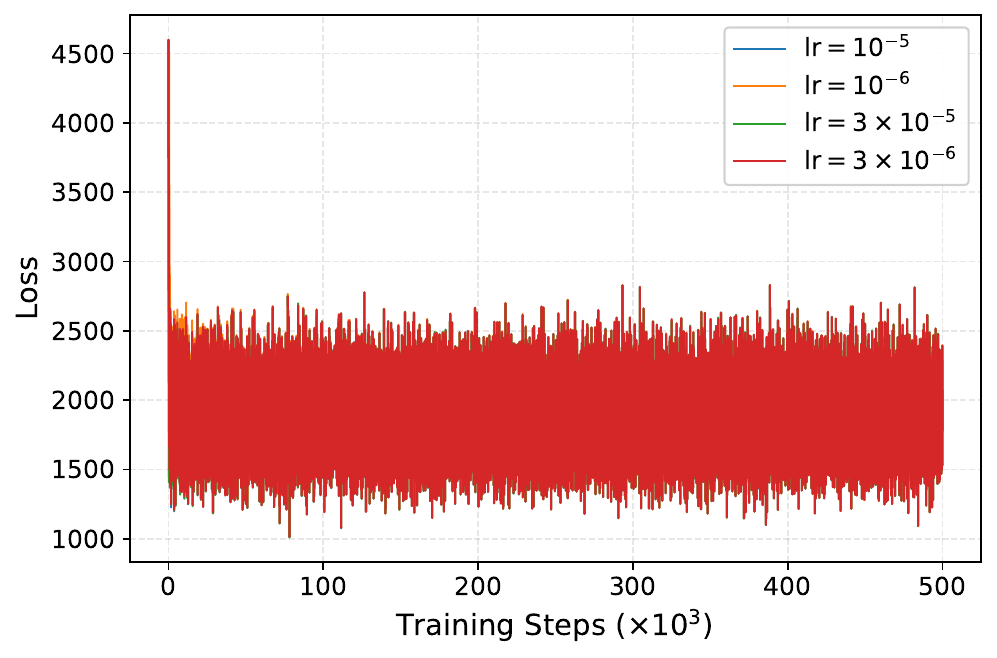}}
  \hfill
  \subfigure[Loss of $x_1$-prediction]{\includegraphics[width=0.24\textwidth]{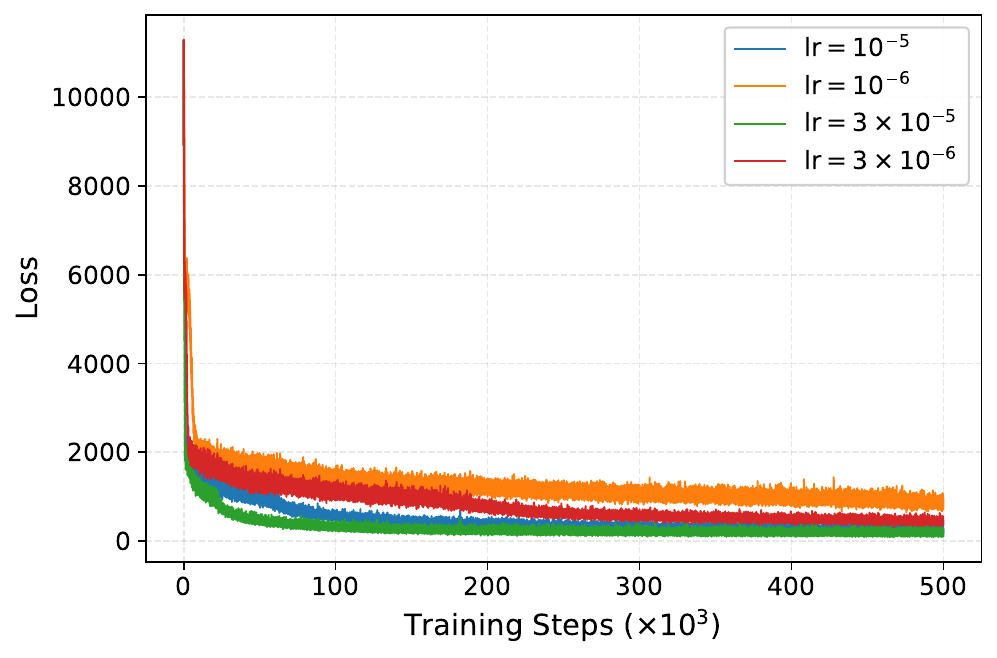}}
  \hfill
  \subfigure[Variance of $u$-prediction]{\includegraphics[width=0.24\textwidth]{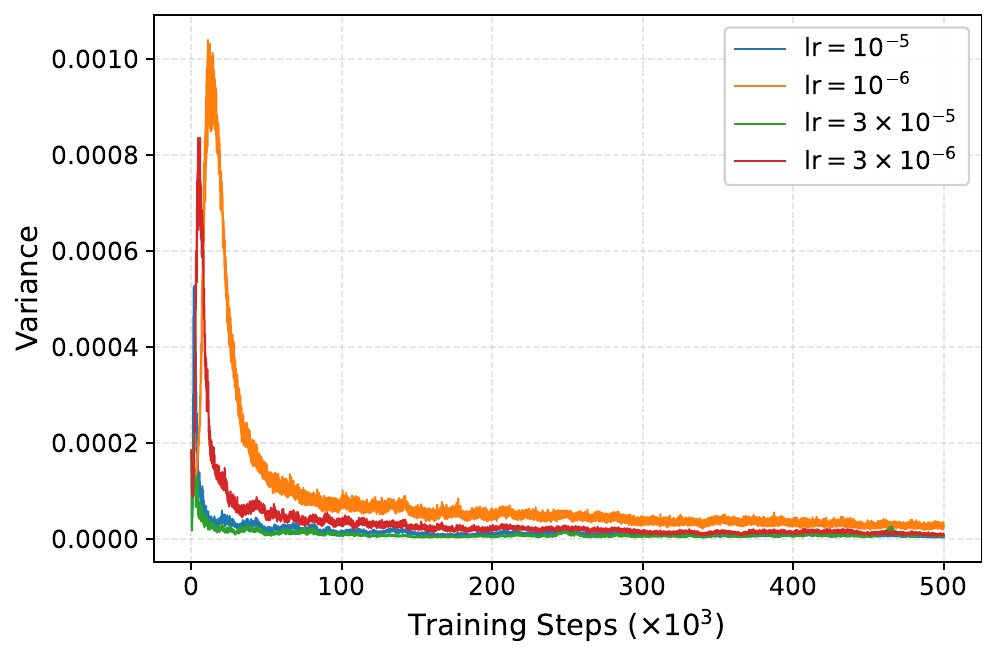}}
  \hfill
  \subfigure[Variance of $x_1$-prediction]{\includegraphics[width=0.24\textwidth]{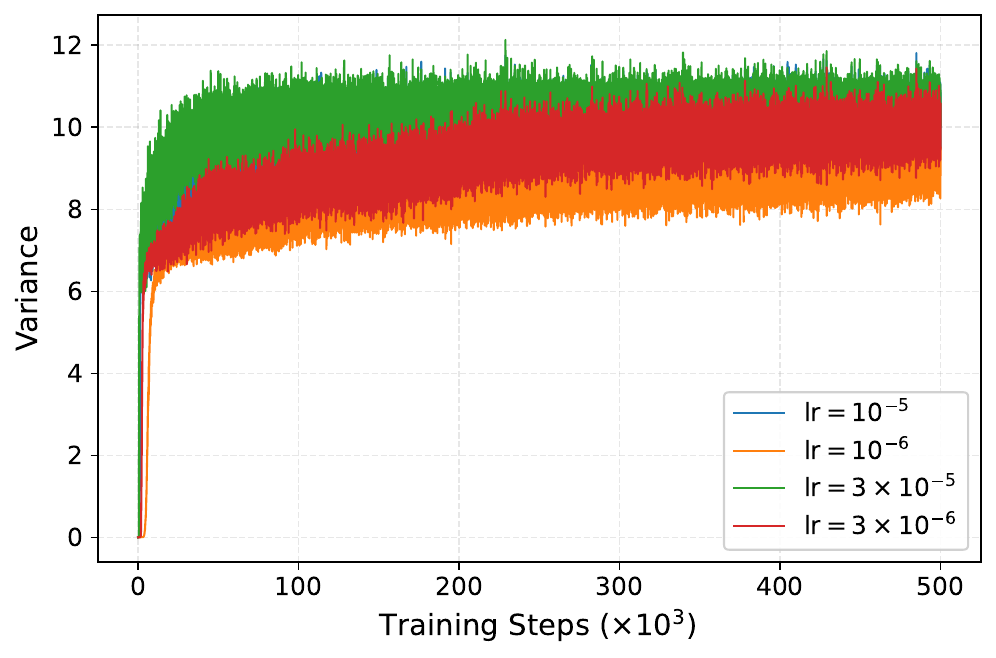}}
  \caption{Training behavior of $u$-prediction and $x_1$-prediction Flow Matching variants across different learning rates.  We report the variance of network outputs, where zero variance indicates a constant SDF field and corresponds to variance collapse.  The $u$-prediction variant suffers from spatial variance collapse and unstable optimization, whereas the $x_1$-prediction model exhibits stable variance and smooth training dynamics.}
  \vspace{0.6em}

  \subfigure{\includegraphics[width=0.98\textwidth]{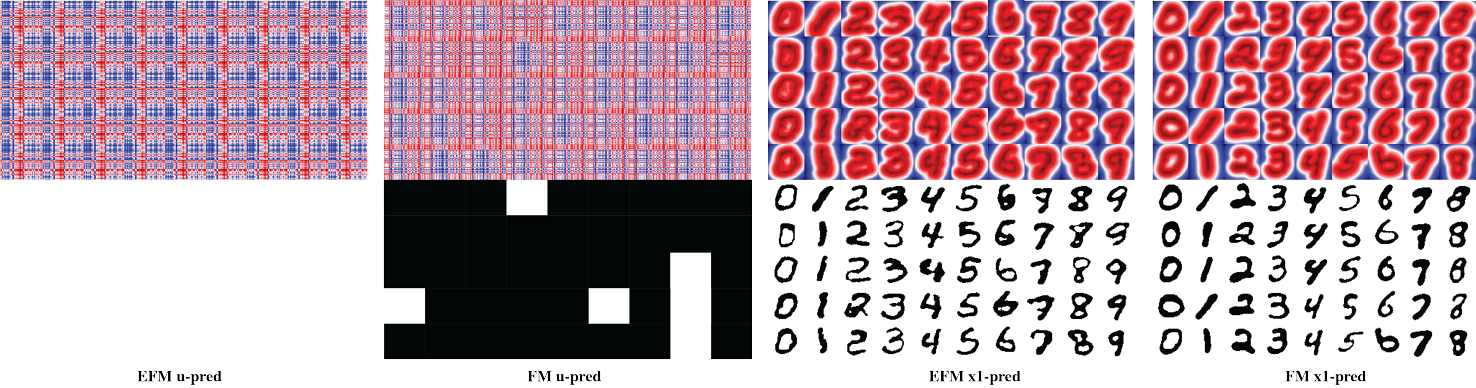}}
    \caption{Comparison of SDF generation using different prediction variants.  Due to spatial variance collapse in $u$-prediction flow matching, the model fails to generate meaningful digit SDFs.  Since $u$-prediction EMF relies on accurate estimation of instantaneous velocities, it also fails in this setting.  In contrast, $x_1$-prediction flow matching learns reliable dynamics, enabling $x_1$-prediction EMF to successfully generate high-quality SDFs.}
\label{fig:minst}
\end{figure}


\begin{figure}
  \centering
  \subfigure[FM 1 step]{\includegraphics[width=0.24\textwidth]{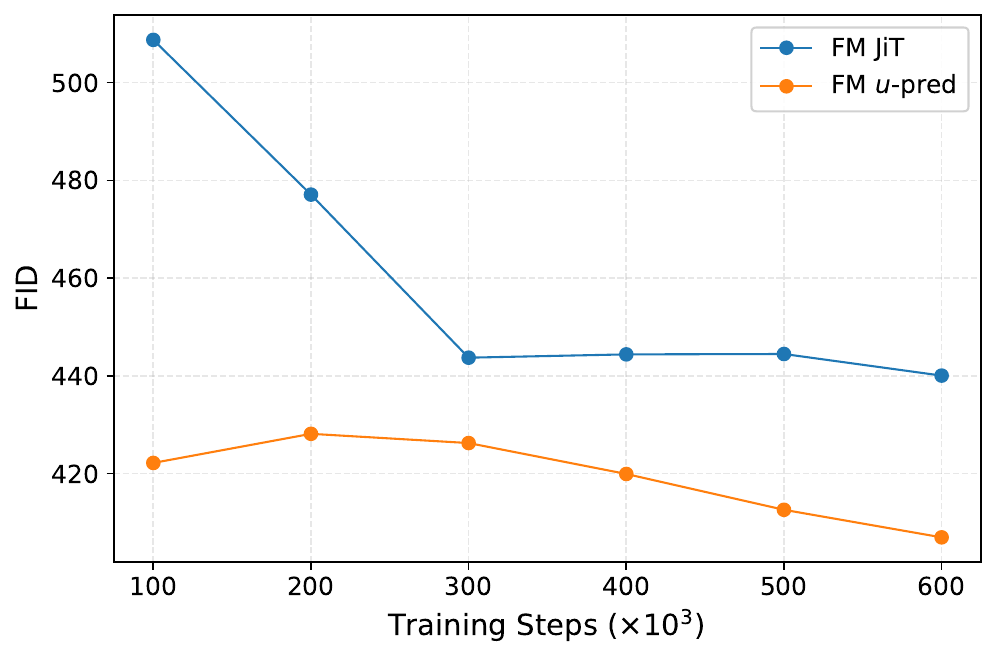}}
  \hfill
  \subfigure[FM 2 step]{\includegraphics[width=0.24\textwidth]{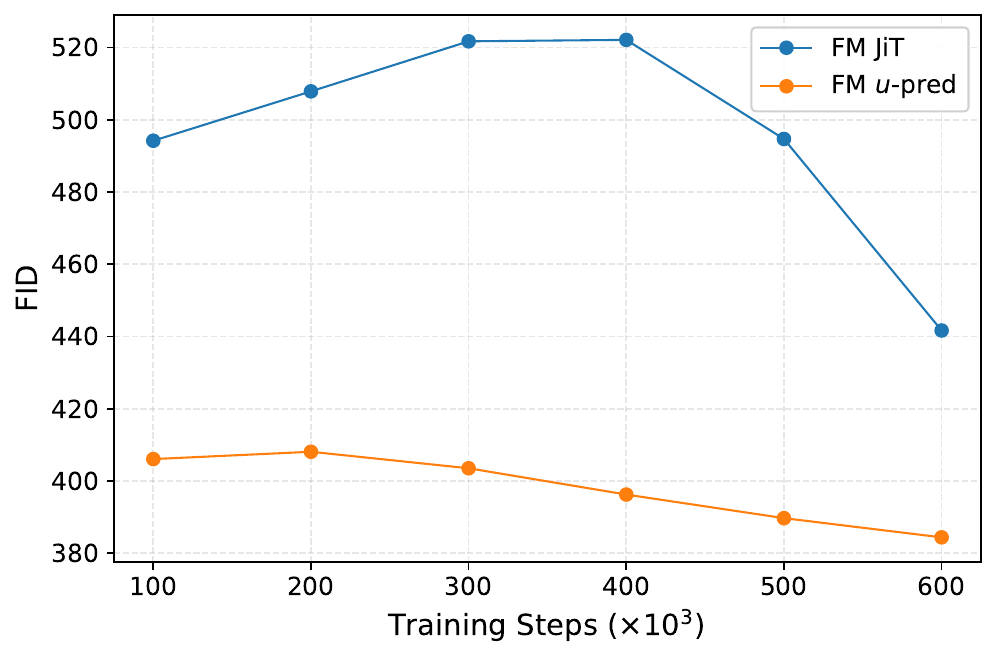}}
  \hfill
  \subfigure[MF\&EFM 1 step]{\includegraphics[width=0.24\textwidth]{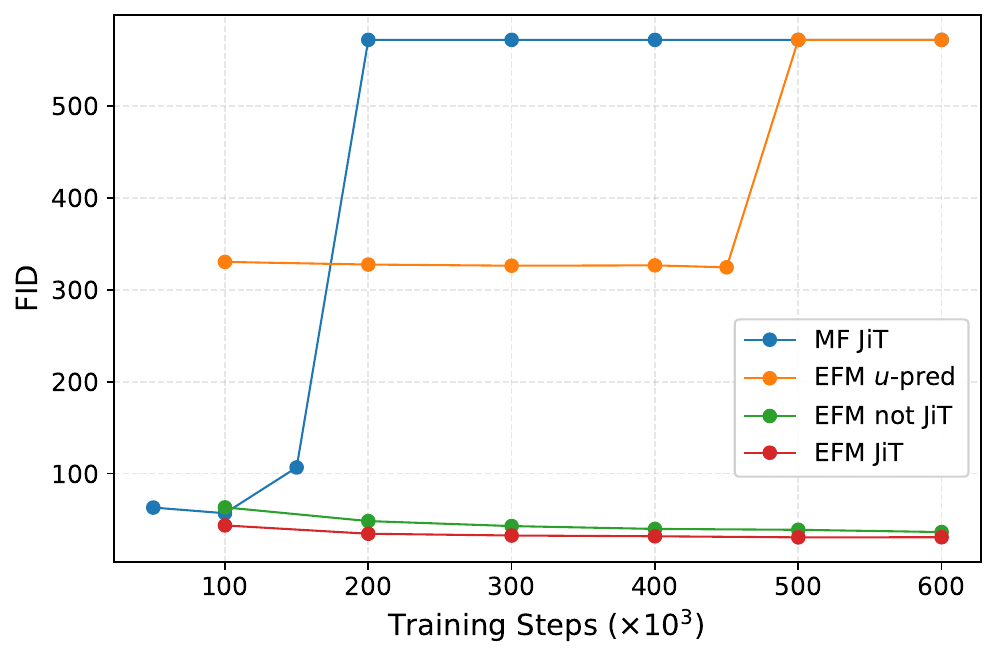}}
  \hfill
  \subfigure[MF\&EFM 2 step]{\includegraphics[width=0.24\textwidth]{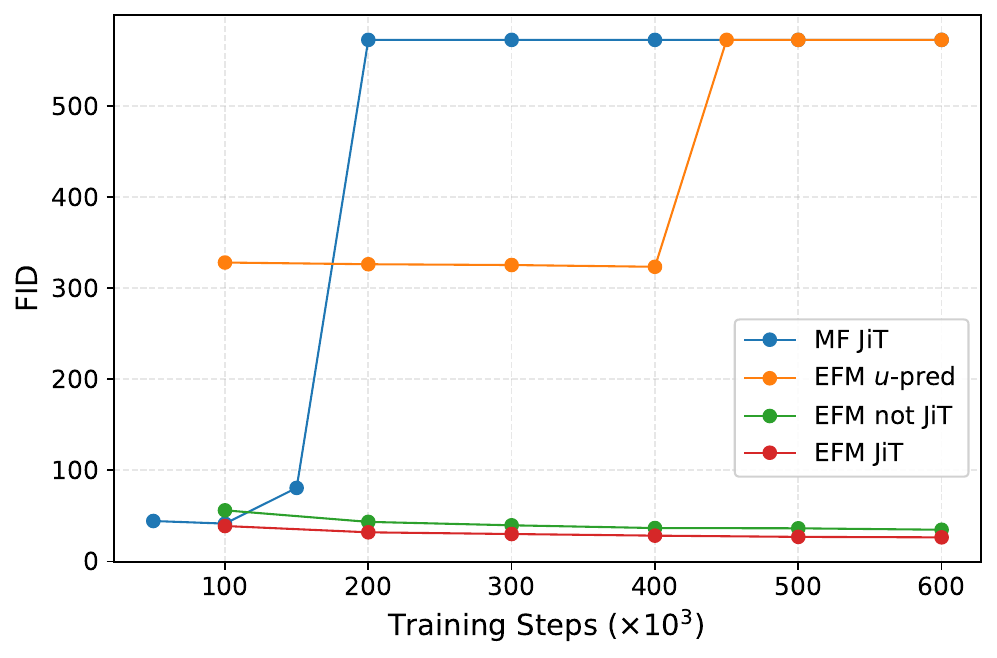}}
  \caption{Training dynamics of FID for different JiT-based generative methods.  We compare JiT-style flow matching (i.e., $x_1$-prediction FM with $u$-loss), $u$-prediction flow matching, JiT-style MeanFlow, $u$-prediction EMF, JiT-style EMF, and $x_1$-prediction EMF. Even with the JiT formulation, MeanFlow remains unstable, highlighting its inherent optimization difficulty. Similarly, $u$-prediction EMF still exhibits unstable behavior, indicating that the $x_1$-prediction variant is essential for stable and reliable training.}
   \label{fig:dynamics_jit_compare}
\end{figure}

\begin{figure*}
    \centering
    \includegraphics[width=0.99\linewidth]{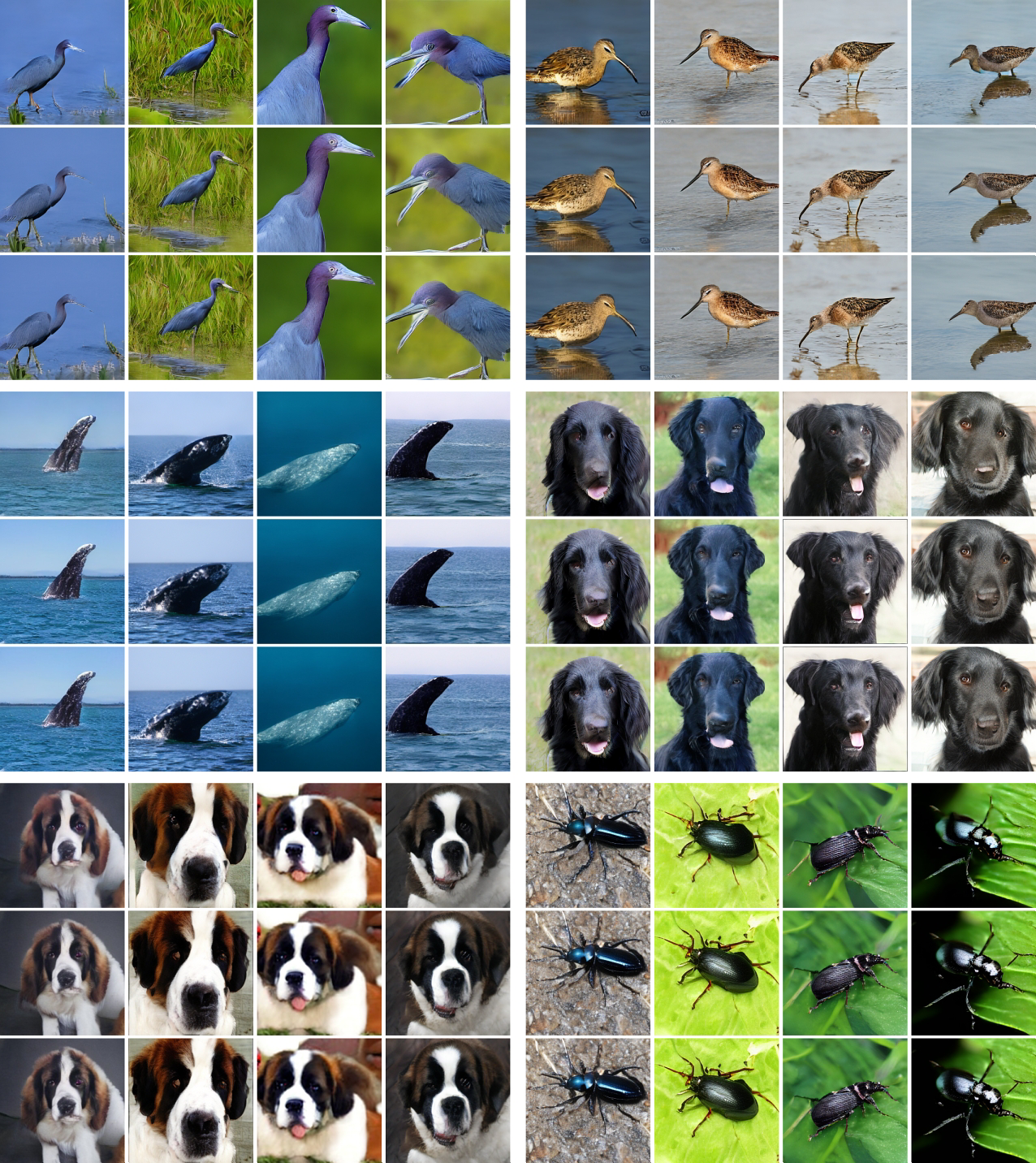}
    \caption{Part 1 of ImageNet~\cite{deng2009imagenet} generation result. First, second and third row shows 1-step, 2-steps and 4-steps generation respectively using the same condition.}
    \label{fig:imagenet_part1}
\end{figure*}

\begin{figure*}
    \centering
    \includegraphics[width=0.99\linewidth]{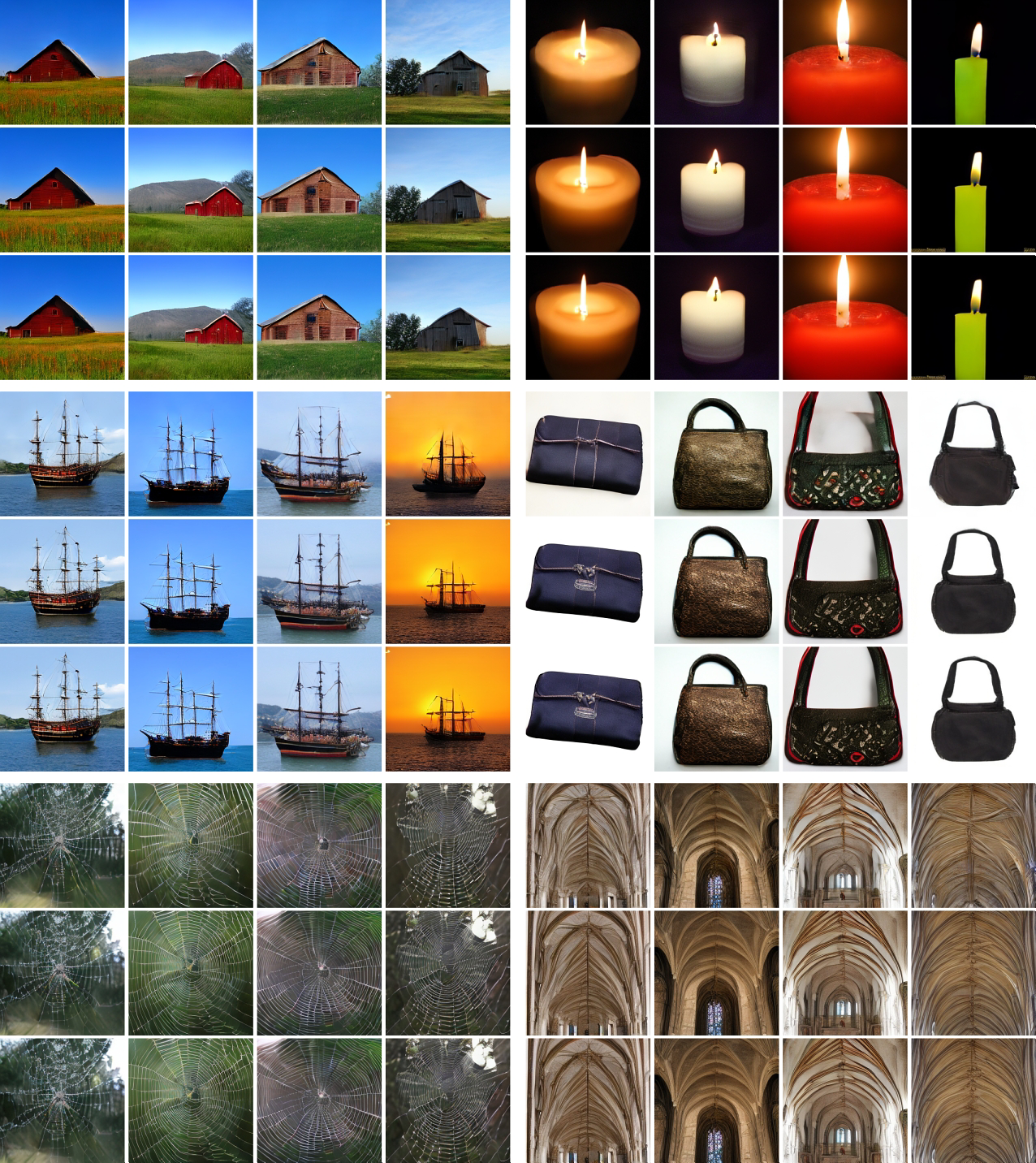}
    \caption{Part 2 of ImageNet~\cite{deng2009imagenet} generation result. First, second and third row shows 1-step, 2-steps and 4-steps generation respectively using the same condition.}
    \label{fig:imagenet_part2}
\end{figure*}

\begin{figure*}
    \centering
    \includegraphics[width=0.99\linewidth]{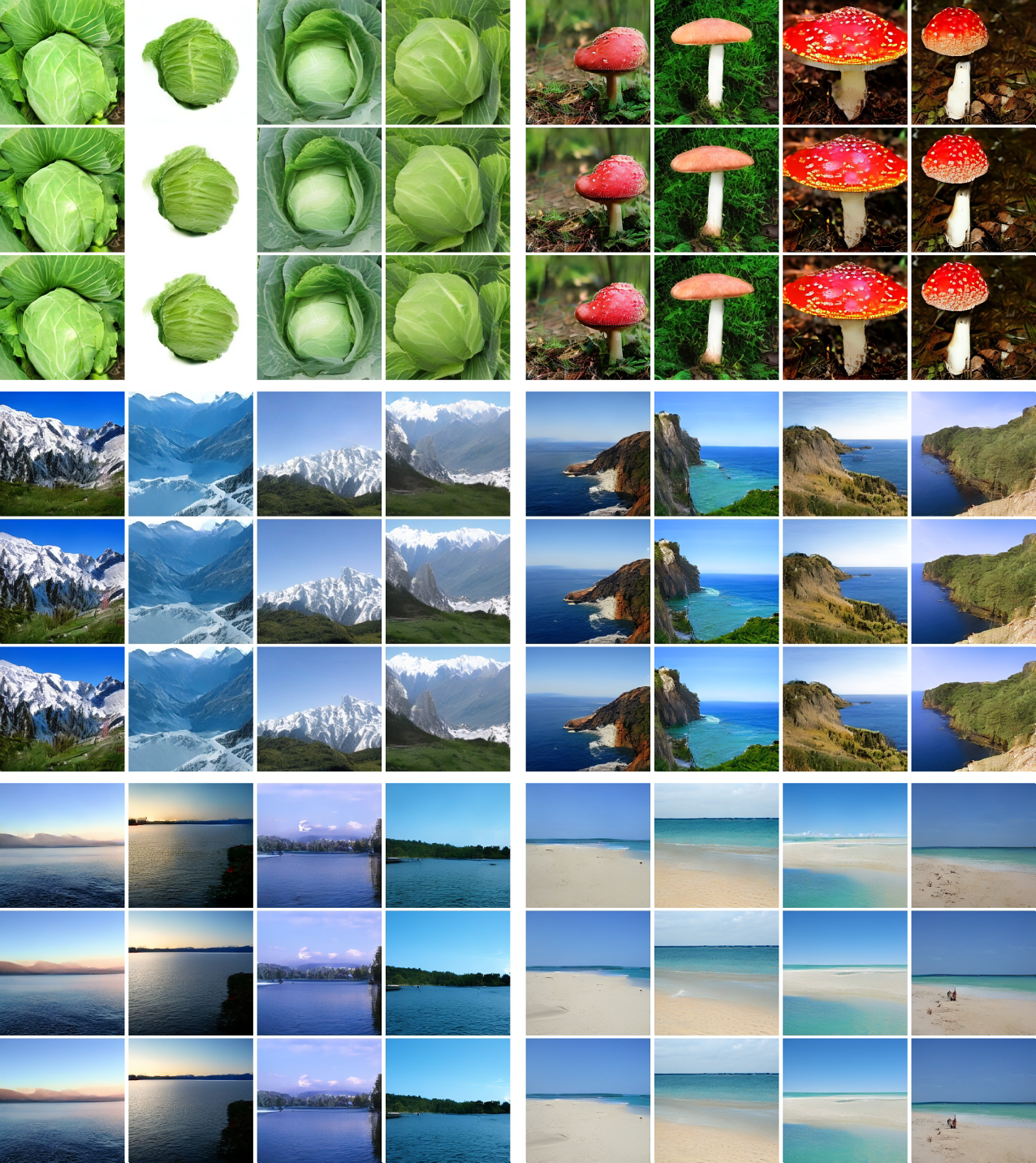}
    \caption{Part 3 of ImageNet~\cite{deng2009imagenet} generation result. First, second and third row shows 1-step, 2-steps and 4-steps generation respectively using the same condition.}
    \label{fig:imagenet_part3}
\end{figure*}

\begin{figure*}
    \centering
    \includegraphics[width=0.99\linewidth]{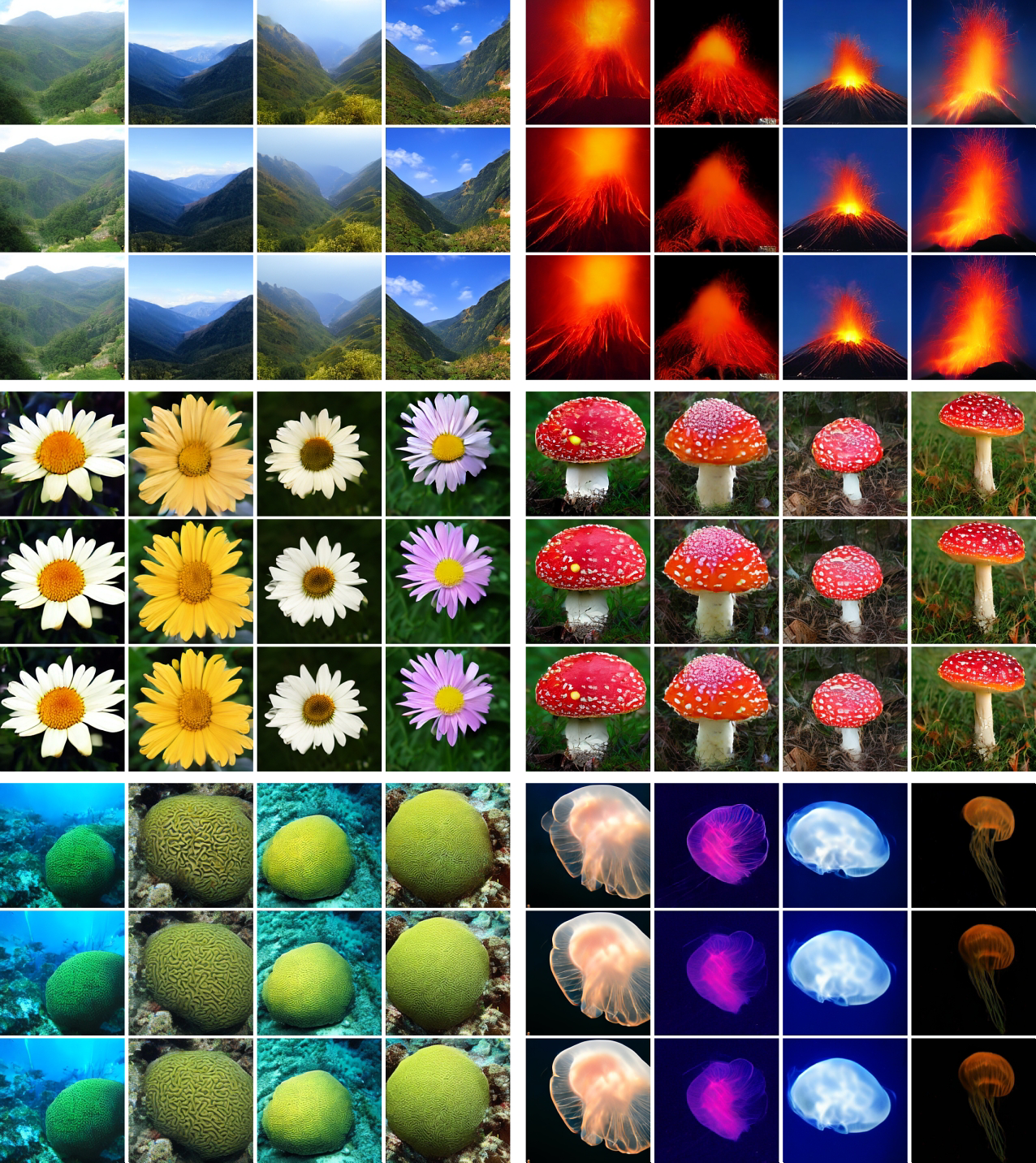}
    \caption{Part 4 of ImageNet~\cite{deng2009imagenet} generation result. First, second and third row shows 1-step, 2-steps and 4-steps generation respectively using the same condition.}
    \label{fig:imagenet_part3}
\end{figure*}

\end{document}